%% file: geomae.arxiv.tex
\documentclass[a4paper,letter]{article}

\usepackage{amsmath}
\usepackage{amsthm}
\usepackage{graphicx}
\usepackage{times}
\usepackage{subcaption}

\usepackage{natbib}
\setcitestyle{authoryear,round,citesep={;},aysep={,},yysep={;}}

\usepackage[margin=1in]{geometry}

\input{math_commands.tex}

\usepackage{hyperref}
\usepackage{url}

\def\rch{\operatorname{rch}} 
\def\vol{\operatorname{vol}} 
\def\surf{\operatorname{surf}} 
\def\diam{\operatorname{diam}}
\def\nn{\operatorname{nn}} 
\def\Vor{\operatorname{Vor}}
\def\Del{\operatorname{Del}}

\newtheorem{theorem}{Theorem} 
\newtheorem{lemma}[theorem]{Lemma}
\newtheorem{cor}[theorem]{Corollary}

\title{On the Geometry of Adversarial Examples}

\author{Marc Khoury \footnote{khoury@eecs.berkeley.edu}\\
University of California, Berkeley
\and
Dylan Hadfield-Menell \footnote{dhm@eecs.berkeley.edu}\\
University of California, Berkeley
}

\begin{document}

\maketitle

\begin{abstract}
Adversarial examples are a pervasive phenomenon of machine learning models where seemingly imperceptible perturbations to the input lead to misclassifications for otherwise statistically accurate models. We propose a geometric framework, drawing on tools from the manifold reconstruction literature, to analyze the high-dimensional geometry of adversarial examples. In particular, we highlight the importance of \emph{codimension}: for low-dimensional data manifolds embedded in high-dimensional space there are many directions off the manifold in which to construct adversarial examples. Adversarial examples are a natural consequence of learning a decision boundary that classifies the low-dimensional data manifold well, but classifies points near the manifold incorrectly. Using our geometric framework we prove (1) a tradeoff between robustness under different norms, (2) that adversarial training in balls around the data is sample inefficient, and (3) sufficient sampling conditions under which nearest neighbor classifiers and ball-based adversarial training are robust.

\textbf{keywords:} adversarial examples, high-dimensional geometry, robustness, generalization

\end{abstract}

\thispagestyle{empty}
\setcounter{page}{0}
\newpage

\section{Introduction}
\label{sec:intro}

Deep learning at scale has led to breakthroughs on important problems
in computer vision~(\cite{Krizhevsky12}), natural language processing~(\cite{Wu16}),
and robotics~(\cite{Levine15}). Shortly thereafter, the
interesting phenomena of \emph{adversarial examples} was observed. A seemingly
ubiquitous property of machine learning models where perturbations of
the input that are imperceptible to humans reliably lead to confident
incorrect classifications (\cite{Szegedy13,Goodfellow14}). 
What has ensued is a standard story from the security literature: a
game of cat and mouse where defenses are proposed only to be quickly
defeated by stronger attacks (\cite{Athalye18}). This has led
researchers to develop methods which are provably robust under
specific attack models (\cite{Madry17, Wong18a, Sinha18, Raghunathan18}). As machine
learning proliferates into society, including security-critical
settings like health care~(\cite{Esteva17}) or autonomous
vehicles~(\cite{Codevilla18}), it is crucial to develop methods that
allow us to understand the vulnerability of our models and design
appropriate counter-measures.

In this paper, we propose a geometric framework for analyzing the
phenomenon of adversarial examples. We leverage the observation that
datasets encountered in practice exhibit low-dimensional structure
despite being embedded in very high-dimensional input spaces. This
property is colloquially referred to as the ``Manifold Hypothesis'':
the idea that low-dimensional structure of `real' data leads to
tractable learning. We model data as being sampled from class-specific
low-dimensional manifolds embedded in a high-dimensional space. We consider a threat model where an adversary
may choose \emph{any} point on the data manifold to perturb by
$\epsilon$ in order to fool a classifier. In order to be robust to
such an adversary, a classifier must be correct everywhere in an
$\epsilon$-tube around the data manifold. Observe that, even though
the data manifold is a low-dimensional object, this tube has the same
dimension as the entire space the manifold is embedded in. Our
analysis argues that adversarial examples are a natural consequence of learning a decision boundary
that classifies all points on a low-dimensional data manifold correctly, but
classifies many points near the manifold incorrectly.
The high \emph{codimension}, the difference between
the dimension of the data manifold and the dimension of the embedding
space, is a key source of the pervasiveness of adversarial examples.

Our paper makes the following contributions. First, we develop a geometric framework,
inspired by the manifold reconstruction literature, that formalizes
the manifold hypothesis described above and our attack model. 
Second, we highlight the role \emph{codimension}
plays in vulnerability to adversarial examples. As the codimension increases,
there are an increasing number of directions off the data manifold in which to construct
adversarial perturbations. Prior work has attributed vulnerability
to adversarial examples to input dimension (\cite{Gilmer18}). 
This is the first work that investigates the
role of \emph{codimension} in adversarial examples. Interestingly, we find that different 
classification algorithms are less sensitive to changes in codimension.
Third, we apply this framework to prove the following results: 
(1) we show that the choice of norm to restrict
an adversary is important in that there exists a tradeoff between being
robust to different norms: we present a classification problem where
improving robustness under the $\|\cdot \|_{\infty}$ norm requires a
loss of $\Omega(1 - 1/\sqrt{d})$ in robustness to the
$\|\cdot\|_{2}$ norm; (2) we show that a common approach, training
against adversarial examples drawn from balls around the training set,
is insufficient to learn robust decision boundaries with
realistic amounts of data; and (3) we show that nearest neighbor classifiers do
not suffer from this insufficiency, due to geometric properties of their
decision boundary away from data, and thus represent a potentially
robust classification algorithm. Finally we provide experimental 
evidence on synthetic datasets and MNIST that support our theoretical results.

\section{Related Work}

This paper approaches the problem of adversarial examples using techniques and intuition from the manifold reconstruction literature. Both fields have a great deal of prior work, so we focus on only the most related papers here.

\subsection{Adversarial Examples}
Some previous work has considered the relationships between adversarial examples and high dimensional geometry. \cite{Franceschi18} explore the robustness of classifiers to random noise in terms of distance to the decision boundary, under the assumption that the decision boundary is locally flat. The work of \cite{Gilmer18} experimentally evaluated the setting of two concentric under-sampled $499$-spheres embedded in $\R^{500}$, and concluded that adversarial examples occur on the data manifold. In contrast, we present a geometric framework for proving robustness guarantees for learning algorithms, that makes no assumptions on the decision boundary. We carefully sample the data manifold in order to highlight the importance of \emph{codimension}; adversarial examples exist \emph{even} when the manifold is perfectly classified. Additionally we explore the importance of the spacing between the constituent data manifolds, sampling requirements for learning algorithms, and the relationship between model complexity and robustness. 

\cite{Wang18} explore the robustness of $k$-nearest neighbor classifiers to adversarial examples. In the setting where the Bayes optimal classifier is uncertain about the true label of each point, they show that $k$-nearest neighbors is not robust if $k$ is a small constant. They also show that if $k \in \Omega(\sqrt{dn\log{n}})$, then $k$-nearest neighbors is robust. Using our geometric framework we show a complementary result: in the setting where each point is certain of its label, $1$-nearest neighbors is robust to adversarial examples.

The decision and medial axes defined in Section~\ref{sec:geom} are maximum margin decision boundaries. Hard margin SVMs define define a linear separator with maximum margin, maximum distance from the training data (\cite{Cortes95}). Kernel methods allow for maximum margin decision boundaries that are non-linear by using additional features to project the data into a higher-dimensional feature space (\cite{Taylor04}). The decision and medial axes generalize the notion of maximum margin to account for the arbitrary curvature of the data manifolds. There have been attempts to incorporate maximum margins into deep learning (\cite{Sun16,Liu16,Liang17,Elsayed18}), often by designing loss functions that encourage large margins at either the output (\cite{Sun16}) or at any layer (\cite{Elsayed18}). In contrast, the decision axis is defined on the input space and we use it as an analysis tool for proving robustness guarantees. 

\subsection{Manifold Reconstruction}
Manifold reconstruction is the problem of discovering the structure of a $k$-dimensional manifold embedded in $\R^d$, given \emph{only} a set of points sampled from the manifold. A large vein of research in manifold reconstruction develops algorithms that are \emph{provably good}: if the points sampled from the underlying manifold are sufficiently dense, these algorithms are guaranteed to produce a geometrically accurate representation of the unknown manifold with the correct topology. The output of these algorithms is often a \emph{simplicial complex}, a set of simplices such as triangles, tetrahedra, and higher-dimensional variants, that approximate the unknown manifold. In particular these algorithms output subsets of the Delaunay triangulation, which along with their geometric dual the Voronoi diagram, have properties that aid in proving geometric and topological guarantees (\cite{Edelsbrunner97}).

The field first focused on curve reconstruction in $\R^2$ (\cite{Amenta98}) and subsequently in $\R^3$ (\cite{Dey99}). Soon after algorithms were developed for surface reconstruction in $\R^3$, both in the noise-free setting (\cite{Amenta99, Amenta02}) and in the presence of noise (\cite{Dey04}). We borrow heavily from the analysis tools of these early works, including the medial axis and the reach. However we emphasize that we have adapted these tools to the learning setting. To the best of our knowledge, our work is the first to consider the medial axis under different norms.

In higher-dimensional embedding spaces (large $d$), manifold reconstruction algorithms face the \emph{curse of dimensionality}. In particular, the Delaunay triangulation, which forms the bedrock of algorithms in low-dimensions, of $n$ vertices in $\R^d$ can have up to $\Theta(n^{\ceil{d/2}})$ simplices. To circumvent the curse of dimensionality, algorithms were proposed that compute subsets of the Delaunay triangulation restricted to the $k$-dimensional tangent spaces of the manifold at each sample point (\cite{Boissonnat14}). Unfortunately, progress on higher-dimensional manifolds has been limited due to the presence of so-called ``sliver'' simplices, poorly shaped simplices that cause in-consistences between the local triangulations constructed in each tangent space (\cite{Cheng05, Boissonnat14}). Techniques that provably remove sliver simplices have prohibitive sampling requirements (\cite{Cheng00, Boissonnat14}). Even in the special case of surfaces ($k=2$) embedded in high dimensions ($d > 3$), algorithms with practical sampling requirements have only recently been proposed (\cite{Khoury16}). Our use of tubular neighborhoods as a tool for analysis is borrowed from \cite{Dey05} and \cite{Khoury16}. 

In this paper we are interested in \emph{learning} robust decision boundaries, \emph{not} reconstructing the underlying data manifolds, and so we avoid the use of Delaunay triangulations and their difficulties entirely. In Section~\ref{sec:proverobust} we present robustness guarantees for two learning algorithms in terms of a sampling condition on the underlying manifold. These sampling requirements scale with the dimension of the underlying manifold $k$, \emph{not} with the dimension of the embedding space $d$.

\section{The Geometry of Data}
\label{sec:geom}
We model data as being sampled from a set of low-dimensional manifolds
(with or without boundary) embedded in a high-dimensional space $\R^d$.
We use $k$ to denote the dimension of a manifold $\mathcal{M} \subset \R^d$.
The special case of a $1$-manifold is called a \emph{curve}, and
a $2$-manifold is a \emph{surface}.
The \emph{codimension} of $\mathcal{M}$ is $d - k$, the difference between
the dimension of the manifold and the dimension of the embedding space.
The ``Manifold Hypothesis'' is the observation that in practice, data is often
sampled from manifolds, usually of high codimension.

In this paper we are primarily interested in the classification problem. Thus we model data as being sampled from $C$ \emph{class manifolds} $\mathcal{M}_{1}, \ldots, \mathcal{M}_{C}$, one for each class. When we wish to refer to the entire space from which a dataset is sampled, we refer to the $\emph{data manifold}$ $\mathcal{M} = \cup_{1 \leq j \leq C} \mathcal{M}_j$. We often work with a finite sample of $n$ points, $X \subset \mathcal{M}$, and we write $X = \{ X_1, X_2, \ldots, X_n \}$. Each sample point $X_i$ has an accompanying class label $y_i \in \{ 1, 2, \ldots, C \}$ indicating which manifold $\mathcal{M}_{y_i}$ the point $X_i$ is sampled from.

Consider a $\|\cdot\|_{p}$-ball $B$ centered at some point $c \in \R^d$ and imagine growing $B$ by increasing its radius starting from zero. For nearly all starting points $c$, the ball $B$ eventually intersects one, \emph{and only one}, of the $\mathcal{M}_i$'s. Thus the nearest point to $c$ on $\mathcal{M}$, in the norm $\|\cdot\|_{p}$, lies on $\mathcal{M}_{i}$. (Note that the nearest point on $\mathcal{M}_{i}$ need not be unique.) 

The \emph{decision axis} $\Lambda_{p}$ of $\mathcal{M}$ is the set of points $c$ such that the boundary of $B$ intersects two or more of the $\mathcal{M}_{i}$, but the interior of $B$ does not intersect $\mathcal{M}$ at all. In other words, the decision axis $\Lambda_p$ is the set of points that have two or more closest points, in the norm $\|\cdot\|_{p}$, \emph{on distinct class manifolds}. See Figure \ref{fig:medialaxis}. The decision axis is inspired by the medial axis, which was first proposed by \cite{Blum67} in the context of image analysis and subsequently modified for the purposes of curve and surface reconstruction by \cite{Amenta98, Amenta02}. We have modified the definition to account for multiple class manifolds and have renamed our variant in order to avoid confusion in the future.

The decision axis $\Lambda_p$ can intuitively be thought of as a decision boundary that is optimal in the following sense. First, $\Lambda_p$ separates the class manifolds when they do not intersect (Lemma \ref{lem:medialaxisseparate}). Second, each point of $\Lambda_p$ is as far away from the class manifolds as possible in the norm $\|\cdot\|_{p}$. As shown in the leftmost example in Figure \ref{fig:medialaxis}, in the case of two linearly separable circles of equal radius, the decision axis $\Lambda_{2}$ is exactly the line that separates the data with maximum margin. For arbitrary manifolds, $\Lambda_p$ generalizes the notion of maximum margin to account for the arbitrary curvature of the class manifolds. 

\begin{figure}[h!]
\begin{center}
\includegraphics[width=0.9\textwidth]{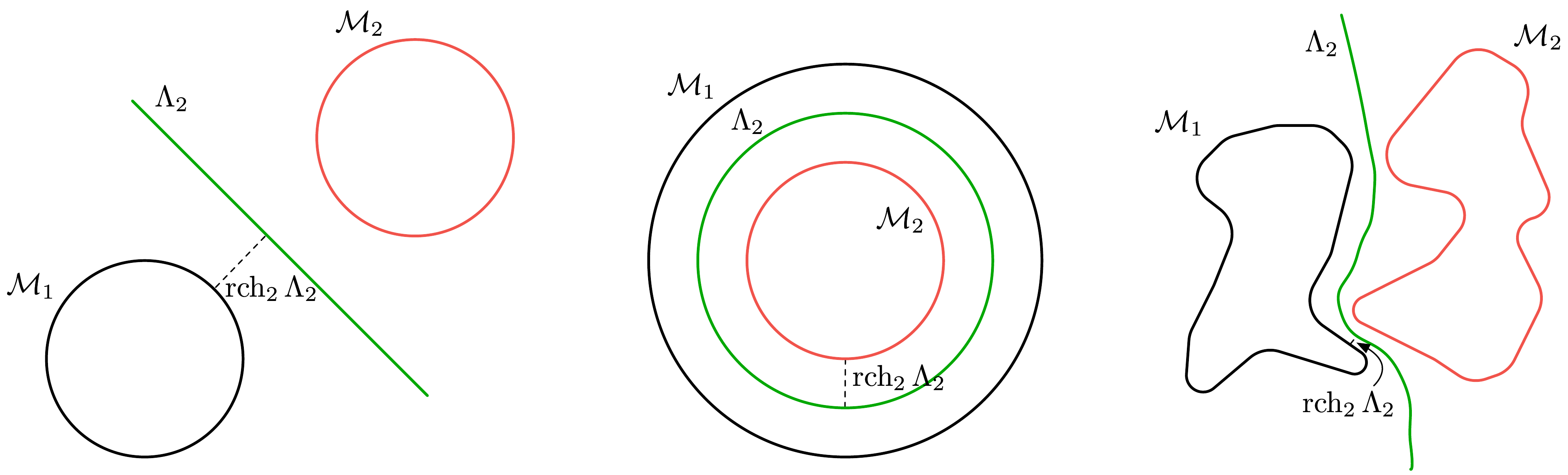}
\caption{Examples of the decision axis $\Lambda_2$, shown here in green, for different data manifolds. Intuitively, the decision axis captures an optimal decision boundary between the data manifolds. It's optimal in the sense that each point on the decision axis is as far away from each data manifold as possible. Notice that in the first example, the decision axis coincides with the maximum margin line.}
\label{fig:medialaxis}
\end{center}
\end{figure} 

Let $T \subset \R^d$ be any set. The \emph{reach} $\rch_{p}{(T; \mathcal{M})}$ of $\mathcal{M}$ is defined as $\inf_{x \in \mathcal{M}, y \in T} \|x - y\|_{p}$. When $\mathcal{M}$ is compact, the reach is achieved by the point on $\mathcal{M}$ that is closest to $T$ under the $\|\cdot\|_{p}$ norm. We will drop $\mathcal{M}$ from the notation when it is understood from context. 

Finally, an $\epsilon$-\emph{tubular neighborhood} of $\mathcal{M}$ is defined as $\mathcal{M}^{\epsilon, p} = \{x \in \R^d: \inf_{y \in \mathcal{M}}\|x - y\|_p \leq \epsilon\}$. That is, $\mathcal{M}^{\epsilon, p}$ is the set of all points whose distance to $\mathcal{M}$ under the metric induced by $\|\cdot\|_{p}$ is less than $\epsilon$. Note that while $\mathcal{M}$ is $k$-dimensional, $\mathcal{M}^{\epsilon, p}$ is always $d$-dimensional. Tubular neighborhoods are how we rigorously define adversarial examples. Consider a classifier $f: \R^d \rightarrow [C]$ for $\mathcal{M}$. An $\epsilon$-\emph{adversarial example} is a point $x \in \mathcal{M}_{i}^{\epsilon, p}$ such that $f(x) \neq i$. A classifier $f$ is robust to all $\epsilon$-adversarial examples when $f$ correctly classifies not only $\mathcal{M}$, but all of $\mathcal{M}^{\epsilon, p}$. Thus the problem of being robust to adversarial examples is rightly seen as one of \emph{generalization}. In this paper we will be primarily concerned with exploring the conditions under which we can provably learn a decision boundary that correctly classifies $\mathcal{M}^{\epsilon, p}$. When $\epsilon < \rch_p{\Lambda_p}$, the decision axis $\Lambda_p$ is one decision boundary that correctly classifies $\mathcal{M}^{\epsilon, p}$ (Corollary~\ref{cor:medialaxisdb}). Throughout the remainder of the paper we will drop the $p$ in $\mathcal{M}^{\epsilon, p}$ from the notation, instead writing $\mathcal{M}^{\epsilon}$; the norm will always be clear from context.

The geometric quantities defined above can be defined more generally for any distance metric $d(\cdot, \cdot)$. In this paper we will focus exclusively on the metrics induced by the norms $\|\cdot\|_{p}$ for $p > 0$. The decision axis under $\|\cdot\|_{2}$ is in general \emph{not} identical to the decision axis under $\|\cdot\|_{\infty}$. In Section \ref{sssec:tradeoff} we will prove that since $\Lambda_{2}$ is not identical to $\Lambda_{\infty}$ there exists a tradeoff in the robustness of any decision boundary between the two norms. 

\section{A Provable Tradeoff in Robustness Between Norms}
\label{sssec:tradeoff}
\cite{Schott18} explore the vulnerability of robust classifiers to attacks under different norms. In particular, they take the robust pretrained classifier of \cite{Madry17}, which was trained to be robust to $\|\cdot\|_{\infty}$-perturbations, and subject it to $\|\cdot\|_{0}$ and $\|\cdot\|_{2}$ attacks. They show that accuracy drops to $0\%$ under $\|\cdot\|_{0}$ attacks and to $35\%$ under $\|\cdot\|_{2}$. Here we explain why poor robustness under the norm $\|\cdot\|_2$ should be expected. 

We say a decision boundary $\mathcal{D}_{f}$ for a classifier $f$ is $\epsilon$-robust in the $\|\cdot\|_{p}$ norm if $\epsilon < \rch_{p}{\mathcal{D}_{f}}$. In words, starting from any point $x \in \mathcal{M}$, a perturbation $\eta_{x}$ must have $p$-norm greater than $\rch_{p}{\mathcal{D}_{f}}$ to cross the decision boundary. The most robust decision boundary to $\|\cdot\|_{p}$-perturbations is $\Lambda_p$. In Theorem \ref{thm:tradeoff} we construct a learning setting where $\Lambda_2$ is distinct from $\Lambda_{\infty}$. Thus, in general, \emph{no single decision boundary can be optimally robust in all norms}. 

\begin{theorem}
\label{thm:tradeoff}
Let $S_1, S_2 \subset \R^{d+1}$ be two concentric $d$-spheres with radii $r_1 < r_2$ respectively. Let $S = S_{1} \cup S_{2}$ and let $\Lambda_2, \Lambda_{\infty}$ be the $\|\cdot\|_{2}$ and $\|\cdot\|_{\infty}$ decision axes of $S$. Then $\Lambda_2 \neq \Lambda_{\infty}$. Furthermore $\rch_{2}{\Lambda_{\infty}} \in \mathcal{O}(\rch_{2}{\Lambda_{2}} / \sqrt{d})$.  
\end{theorem}

\begin{proof}
The decision axis under $\|\cdot\|_{2}$, $\Lambda_2$, is just the $d$-sphere with radius $(r_1 + r_2) / 2$. However, $\Lambda_{\infty}$ is \emph{not} identical to $\Lambda_2$ in this setting; in fact most $\Lambda_{\infty}$ of approaches $S_1$ as $d$ increases. 

The geometry of a $\|\cdot\|_{\infty}$-ball $B_{\Delta}$ centered at $m \in \R^d$ with radius $\Delta$ is that of a hypercube centered at $m$ with side length $2\Delta$. To find a point on $\Lambda_{\infty}$ we place $B_{\Delta}$ tangent to the north pole $q$ of $S_{1}$ so that the corners of $B_{\Delta}$ touch $S_{2}$. The north pole has coordinate representation $q = (0, \ldots, 0, r_1)$, the center $m = (0, \ldots, 0, r_1 + \Delta)$, and a corner of $B_{\Delta}$ can be expressed as $p = (\Delta, \ldots, \Delta, r_1 + 2\Delta)$. Additionally we have the constraint that $\|p\|_{2} = r_2$ since $p \in S_{2}$. Then we can solve for $\Delta$ as 
\begin{align*}
r_{2}^2 &= \|p\|_{2}^{2} = (d - 1) \Delta^2 + (r_1 + 2 \Delta)^2 = (d + 3) \Delta^2 + 4r_1 \Delta + r_1^2;\\
\Delta &= \frac{-2 r_1 + \sqrt{r_1^2 + 3 r_{2}^2 + d (r_{2}^2-r_{1}^2)}}{d+3}, 
\end{align*}
where the last step follows from the quadratic formula and the fact that $\Delta > 0$. For fixed $r_{1}, r_{2}$, the value $\Delta$ scales as $\mathcal{O}(1/\sqrt{d})$. It follows that $\rch_{2}{\Lambda_{\infty}} \in \mathcal{O}(\rch_{2}{\Lambda_2} / \sqrt{d})$.
\end{proof}

From Theorem \ref{thm:tradeoff} we conclude that the minimum distance from $S_1$ to $\Lambda_{\infty}$ \emph{under the $\|\cdot\|_{2}$ norm} is upper bounded as $\rch_{2}{\Lambda_{\infty}} \in \mathcal{O}(\rch_{2}{\Lambda_2} / \sqrt{d})$. If a classifier $f$ is trained to learn $\Lambda_{\infty}$, an adversary, starting on $S_1$, can construct an $\|\cdot\|_{2}$ adversarial example for a perturbation as small as $\mathcal{O}(1/\sqrt{d})$. Thus we should \emph{expect} $f$ to be less robust to $\|\cdot\|_2$-perturbations. Figure~\ref{fig:robustness-tradeoff} verifies this result experimentally.

\begin{figure}
\centering
\includegraphics[width=0.5\textwidth]{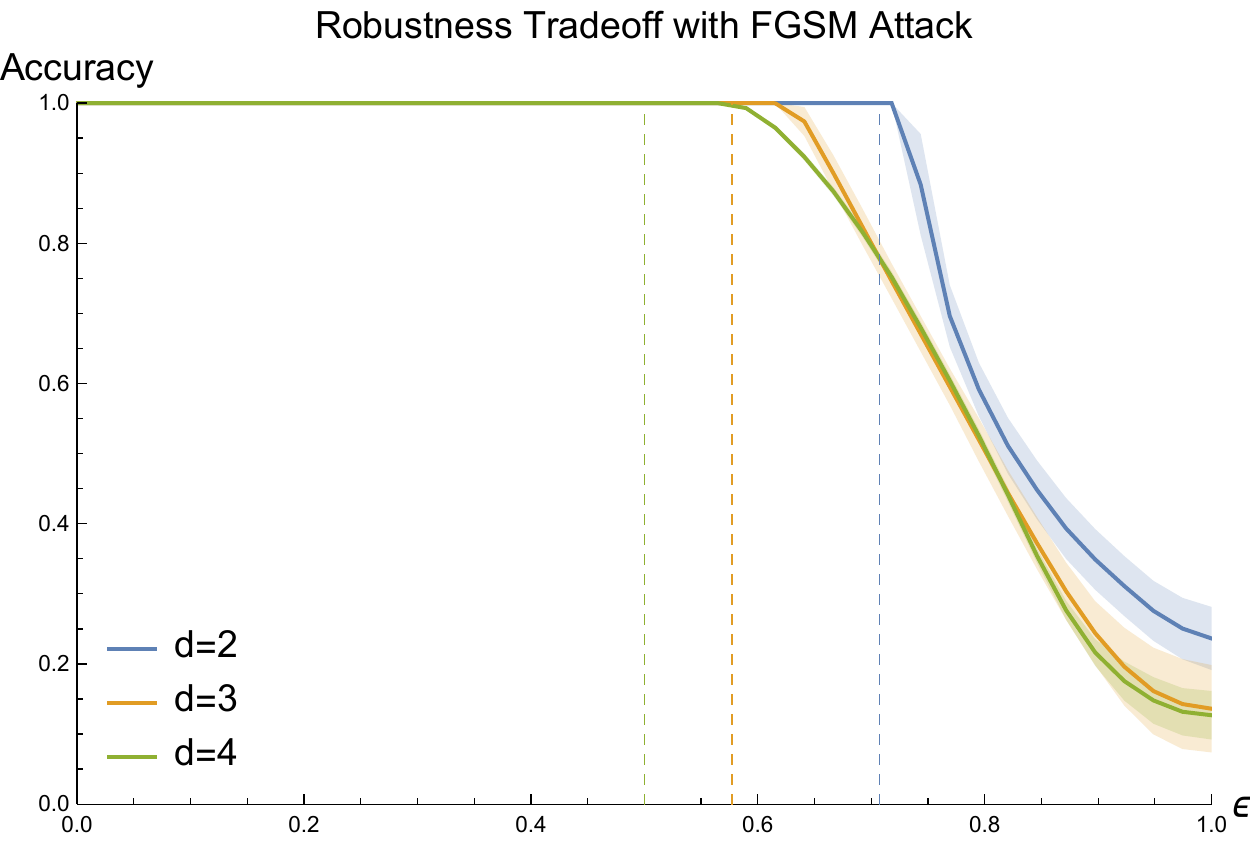}
\caption{\label{fig:robustness-tradeoff} As the dimension increases, the $\rch_{2}{(\Lambda_{\infty}; S_1 \cup S_2)}$ decreases, and so an $\|\cdot\|_{\infty}$ robust classifier is less robust to $\|\cdot\|_{2}$ attacks. The dashed lines are placed at $1 / \sqrt{d}$, where our theoretical results suggest we should start finding $\|\cdot\|_{2}$ adversarial examples. We use the robust $\|\cdot\|_{\infty}$ loss of \cite{Wong18a}}
\end{figure}

We expect that $\Lambda_{2} \neq \Lambda_{\infty}$ is the common case in practice. For example, Theorem~\ref{thm:tradeoff} extends immediately to concentric cylinders and intertwined tori by considering $2$-dimensional planar cross-sections. In general, we expect that $\Lambda_{2} \neq \Lambda_{\infty}$ in situations where a $2$-dimensional cross-section with $\mathcal{M}$ has nontrivial curvature.

Theorem \ref{thm:tradeoff} is important because, even in recent literature, researchers have attributed this phenomena to overfitting. \cite{Schott18} state that ``the widely recognized and by far most successful defense by Madry et al. (1) \emph{overfits} on the $L_{\infty}$ metric (it’s highly susceptible to $L_2$ and $L_0$ perturbations)'' (emphasis ours). We disagree; the \cite{Madry17} classifier performed exactly as intended. It learned a decision boundary that is robust under $\|\cdot\|_{\infty}$, which we have shown is quite different from the most robust decision boundary under $\|\cdot\|_{2}$. 

Interestingly, the proposed models of \cite{Schott18} also suffer from this tradeoff. Their model ABS has accuracy $80\%$ to $\|\cdot\|_{2}$ attacks but drops to $8\%$ for $\|\cdot\|_{\infty}$. Similarly their model ABS Binary has accuracy $77\%$ to $\|\cdot\|_{\infty}$ attacks but drops to $39\%$ for $\|\cdot\|_{2}$ attacks.

We reiterate, in general, no single decision boundary can be optimally robust in all norms. 
 
\section{Provably Robust Classifiers}
\label{sec:proverobust}
Adversarial training, the process of training on adversarial examples generated in a $\|\cdot\|_{p}$-ball around the training data, is a very natural approach to constructing robust models (\cite{Goodfellow14, Madry17}). In our notation this corresponds to training on samples drawn from $X^{\epsilon}$ for some $\epsilon$. While natural, we show that there are simple settings where this approach is much less sample-efficient than other classification algorithms, if the \emph{only} guarantee is correctness in $X^{\epsilon}$. 

Define a learning algorithm $\mathcal{L}$ with the property that, given a training set $X \subset \mathcal{M}$ sampled from a manifold $\mathcal{M}$, $\mathcal{L}$ outputs a model $f_{\mathcal{L}}$ such that for every $x \in X$ with label $y$, and every $\hat{x} \in B(x, \rch_{p}{\Lambda_{p}})$, $f_{\mathcal{L}}(\hat{x}) = f_{\mathcal{L}}(x) = y$. Here $B(x, r)$ denotes the ball centered at $x$ of radius $r$ in the relevant norm. That is, $\mathcal{L}$ learns a model that outputs the same label for any $\|\cdot\|_{p}$-perturbation of $x$ up to $\rch_{p}{\Lambda_{p}}$ as it outputs for $x$. $\mathcal{L}$ is our theoretical model of adversarial training (\cite{Goodfellow14, Madry17}). Theorem~\ref{thm:classexists} states that $\mathcal{L}$ is sample inefficient in high codimensions.  

\begin{theorem}
\label{thm:classexists}
There exists a classification algorithm $\mathcal{A}$ that, for a particular choice of $\mathcal{M}$, correctly classifies $\mathcal{M}^{\epsilon}$ using exponentially fewer samples than are required for $\mathcal{L}$ to correctly classify $\mathcal{M}^{\epsilon}$. 
\end{theorem}

Theorem~\ref{thm:classexists} follows from Theorems~\ref{thm:sampling} and \ref{thm:samplinggap}. In Theorems~\ref{thm:sampling} and \ref{thm:samplinggap} we will prove that a nearest neighbor classifier $f_{\nn}$ is one such classification algorithm. Nearest neighbor classifiers are naturally robust in high codimensions because the Voronoi cells of $X$ are \emph{elongated in the directions normal} to $\mathcal{M}$ when $X$ is dense (\cite{Dey07}).

Before we state Theorem \ref{thm:sampling} we must introduce a sampling condition on $\mathcal{M}$. A $\delta$-cover of a manifold $\mathcal{M}$ in the norm $\|\cdot \|_{p}$ is a finite set of points $X$ such that for every $x \in \mathcal{M}$ there exists $X_i$ such that $\|x - X_i\|_{p} \leq \delta$. Theorem~\ref{thm:sampling} gives a sufficient sampling condition for $f_{\mathcal{L}}$ to correctly classify $\mathcal{M}^{\epsilon}$ for all manifolds $\mathcal{M}$. Theorem~\ref{thm:sampling} also provides a sufficient sampling condition for a nearest neighbor classifier $f_{\nn}$ to correctly classify $\mathcal{M}^{\epsilon}$, which is substantially less dense than that of $f_{\mathcal{L}}$. Thus different classification algorithms have different sampling requirements in high codimensions.

\begin{theorem}
\label{thm:sampling}
Let $\mathcal{M} \subset \R^d$ be a $k$-dimensional manifold and let $\epsilon < \rch_p{\Lambda_p}$ for any $p > 0$. Let $f_{nn}$ be a nearest neighbor classifier and let $f_{\mathcal{L}}$ be the output of a learning algorithm $\mathcal{L}$ as described above. Let $X_{\nn}, X_{\mathcal{L}} \subset \mathcal{M}$ denote the training sets for $f_{\nn}$ and $\mathcal{L}$ respectively. We have the following sampling guarantees: 

\begin{enumerate}
\item If $X_{\nn}$ is a $\delta$-cover for $\delta \leq 2 (\rch_p{\Lambda_p} - \epsilon)$ then $f_{\nn}$ correctly classifies $\mathcal{M}^{\epsilon}$.
\item If $X_{\mathcal{L}}$ is a $\delta$-cover for $\delta \leq \rch_p{\Lambda_p} - \epsilon$ then $f_{\mathcal{L}}$ correctly classifies $\mathcal{M}^{\epsilon}$.
\end{enumerate}
\end{theorem}
\begin{proof}
Here we use $d(\cdot, \cdot)$ to denote the metric induced by the $\|\cdot\|_{p}$ norm. We begin by proving (1). Let $q \in \mathcal{M}^{\epsilon}$ be any point in $\mathcal{M}^{\epsilon}$. Suppose without loss of generality that $q \in \mathcal{M}_{i}^{\epsilon}$ for some class $i$. The distance $d(q, \mathcal{M}_{j})$ from $q$ to any other data manifold $\mathcal{M}_{j}$, and thus any sample on $\mathcal{M}_{j}$, is lower bounded by $d(q,\mathcal{M}_{j}) \geq 2\rch_p{\Lambda_p} - \epsilon$. See Figure \ref{fig:samplingproof}. It is then both necessary and sufficient that there exists a $x \in \mathcal{M}_{i}$ such that $d(q, x) < 2\rch_p{\Lambda_p} - \epsilon$ for $f_{\nn}(q) = i$. (Necessary since a properly placed sample on $\mathcal{M}_{j}$ can achieve the lower bound on $d(q, \mathcal{M}_{j})$.) The distance from $q$ to the nearest sample $x$ on $\mathcal{M}_{i}$ is $d(q, x) \leq \epsilon + \delta$ for some $\delta > 0$. The question is how large can we allow $\delta$ to be and still guarantee that $f_{\nn}$ correctly classifies $\mathcal{M}^{\epsilon}$? We need
\begin{equation*}
d(q, x) \leq \epsilon + \delta \leq 2\rch_p{\Lambda_p} - \epsilon \leq d(q, \mathcal{M}_{j})
\end{equation*}
which implies that $\delta \leq 2 (\rch_p{\Lambda_p} - \epsilon)$. It follows that a $\delta$-cover with $\delta = 2(\rch_p{\Lambda_p} - \epsilon)$ is sufficient, and in some cases necessary, to guarantee that $f_{nn}$ correctly classifies $\mathcal{M}^{\epsilon}$.

Next we prove (2). As before let $q \in \mathcal{M}_{i}^{\epsilon}$. It is both necessary and sufficient for $q \in B(x, \rch_p{\Lambda_p})$ for some sample $x \in \mathcal{M}_{i}$ to guarantee that $f_{\mathcal{L}}(q) = i$, by definition of $\mathcal{L}$. The distance to the nearest sample $x$ on $\mathcal{M}_{i}$ is $d(q, x) \leq \epsilon + \delta$ for some $\delta > 0$. Thus it suffices that $\delta \leq \rch_p{\Lambda_p}-\epsilon$. 
\end{proof}

\begin{figure}
\begin{center}
\includegraphics[width=0.5\linewidth]{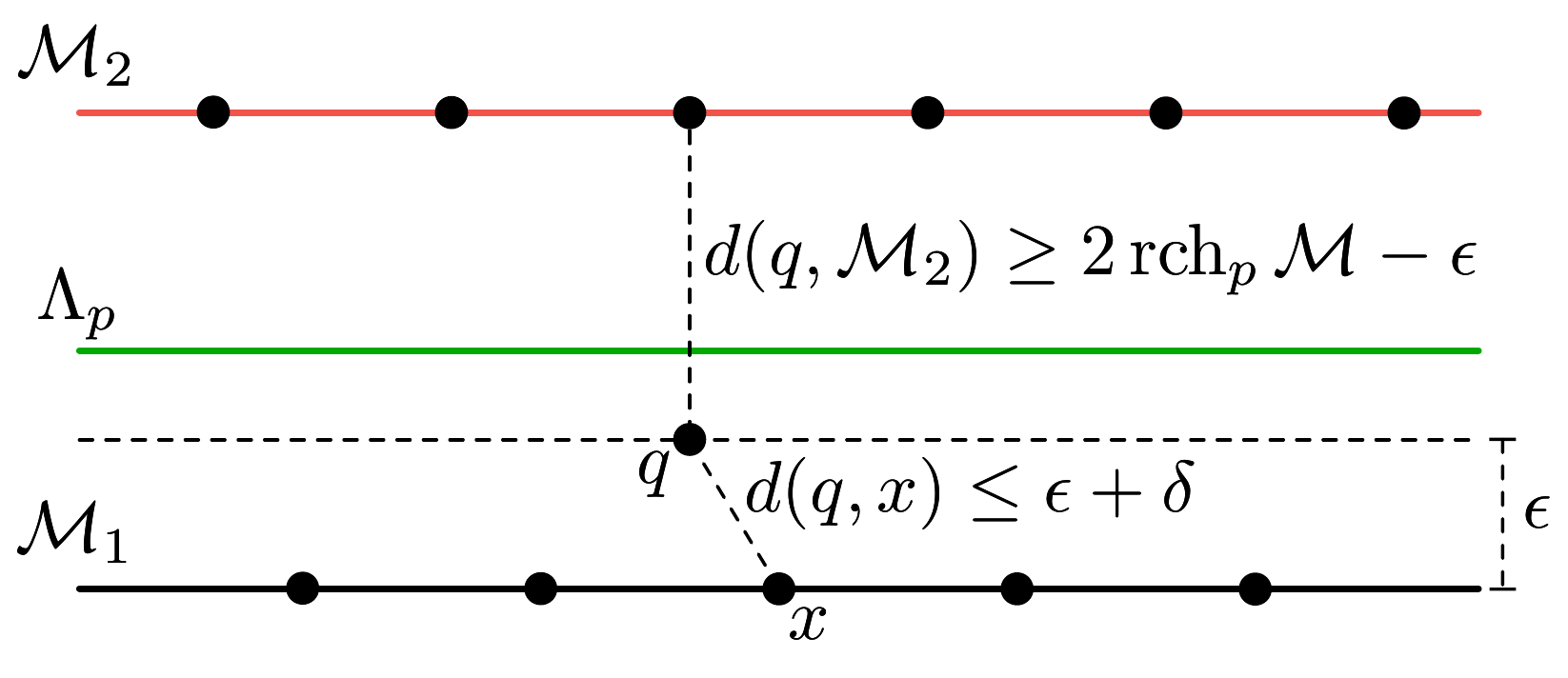}
\caption{Proof of Theorem \ref{thm:sampling}. The distance from a query point $q$ to $\mathcal{M}_{2}$, and thus the closest incorrectly labeled sample, is lower bounded by the distance necessary to reach the medial axis $\Lambda_{p}$ plus the distance from $\Lambda_{p}$ to $\mathcal{M}_{2}$.} 
\label{fig:samplingproof}
\end{center}
\end{figure}

In Appendix~\ref{sec:addtheory} we provide additional robustness results for nearest neighbors including: (1) a similar robustness guarantee as in Theorem~\ref{thm:sampling} when noise is introduced into the samples and (2) that the decision boundary $\mathcal{D}_{f_{\nn}}$ of $f_{\nn}$ approaches the decision axis as the sample density increases.

The bounds on $\delta$ in Theorem~\ref{thm:sampling} are sufficient, but they are not always necessary. There exist manifolds where the bounds in Theorem~\ref{thm:sampling} are pessimistic, and less dense samples corresponding to larger values of $\delta$ would suffice. 

Next we will show a setting where bounds on $\delta$ similar to those in Theorem~\ref{thm:sampling} are \emph{necessary}. In this setting, the difference of a factor of $2$ in $\delta$ between the sampling requirements of $f_{\nn}$ and $f_{\mathcal{L}}$ leads to an exponential gap between the sizes of $X_{\nn}$ and $X_{\mathcal{L}}$ necessary to achieve the same amount of robustness.

Define $\Pi_1 = \{x \in \R^d: \ell \leq x_1, \ldots, x_k \leq \mu \text{ and } x_{k+1} = \ldots = x_d = 0\}$; that is $\Pi_1$ is a subset of the $x_1$-$\ldots$-$x_k$-plane bounded between the coordinates $[\ell, \mu]$. Similarly define $\Pi_2 = \{x \in \R^d: \ell \leq x_1, \ldots, x_k \leq \mu \text{ and } x_{k+1} = \ldots = x_{d-1} = 0 \text{ and } x_d = 2\}$. Note that $\Pi_2$ lies in the subspace $x_d = 2$; thus $\rch_{2}{\Lambda_{2}} = 1$, where $\Lambda_2$ is the decision axis of $\Pi = \Pi_1 \cup \Pi_2$. In the $\|\cdot\|_{2}$ norm we can show that the gap in Theorem~\ref{thm:sampling} is necessary for $\Pi = \Pi_1 \cup \Pi_2$. Furthermore the bounds we derive for $\delta$-covers for $\Pi$ for both $f_{\nn}$ and $f_{\mathcal{L}}$ are tight. Combined with well-known properties of covers, we get that the ratio $|X_{\mathcal{L}}|/|X_{\nn}|$ is exponential in $k$.

\begin{theorem}
\label{thm:samplinggap}
Let $\Pi = \Pi_{1} \cup \Pi_{2}$ as described above. Let $X_{\nn}, X_{\mathcal{L}} \subset \Pi$ be minimum training sets necessary to guarantee that $f_{\nn}$ and $f_{\mathcal{L}}$ correctly classify $\mathcal{M}^{\epsilon}$. Then we have that
\begin{equation}
\frac{|X_{\mathcal{L}}|}{|X_{\nn}|} \geq 2^{k/2} 
\end{equation}
\end{theorem}
\begin{proof}
Let $q \in \Pi_1^{\epsilon}$. Since $\Pi_1$ is flat, the distance to from $q$ to the nearest sample $x \in \Pi_1$ is bounded as $\|q - x\|_{2} \leq \sqrt{\epsilon^2 + \delta^2}$. For $f_{\nn}(q) = 1$ we need that $\|q - x\|_{2} \leq 2 - \epsilon$, and so it suffices that $\delta \leq 2 \sqrt{1 - \epsilon}$. In this setting, this is also necessary; should $\delta$ be any larger a property placed sample on $\Pi_{2}$ can claim $q$ in its Voronoi cell.

Similarly for $f_{\mathcal{L}}(q) = 1$ we need that $\|q - x\|_{2} \leq 1$, and so it suffices that $\delta \leq \sqrt{1-\epsilon^2}$. In this setting, this is also necessary; should $\delta$ be any larger, $q$ lies outside of every $\|\cdot\|_{2}$-ball $B(x, 1)$ and so $\mathcal{L}$ is free to learn a decision boundary that misclassifies $q$.

Let $\mathcal{N}(\delta, \mathcal{M})$ denote the size of the minimum $\delta$-cover of $\mathcal{M}$. Since $\Pi$ is flat (has no curvature) and since the intersection of $\Pi$ with a $d$-ball centered at a point on $\Pi$ is a $k$-ball, a standard volume argument can be applied in the affine subspace $\operatorname{aff}{\Pi}$ to conclude that $\mathcal{N}(\delta,\Pi) \in \Theta\left(\vol_{k}\Pi/\delta^{k}\right)$. So we have
\begin{align*}
\frac{\mathcal{N}(\sqrt{1-\epsilon^2}, \Pi)}{\mathcal{N}(2\sqrt{1-\epsilon}, \Pi)} &= 2^{k} \left(\frac{1}{1 + \epsilon}\right)^{k/2}\\
&\geq 2^{k /2} 
\end{align*}

Since $\Pi$ is constant in both settings, the factor $\vol_{k}{\Pi}$ as well as the constant factors hidden by $\Theta(\cdot)$ cancel. (Note that we are using the fact that $\Pi_1, \Pi_{2}$ have finite $k$-dimensional volume.) The inequality follows from the fact that the expression $(1 + \epsilon)^{-k/2}$ is monotonically decreasing on the interval $[0,1]$ and takes value $2^{-k /2}$ at $\epsilon = 1$. 
\end{proof}

We have shown that both $\mathcal{L}$ and nearest neighbor classifiers learn robust decision boundaries when provided sufficiently dense samples of $\mathcal{M}$. However there are settings where nearest neighbors is exponentially more sample-efficient than $\mathcal{L}$ in achieving the same amount of robustness. We experimentally verify these theoretical results in Section~\ref{ssec:codim}.

\section{$X^{\epsilon}$ is a Poor Model of $\mathcal{M}^{\epsilon}$}
\begin{figure}
\begin{center}
\includegraphics[width=.4\linewidth]{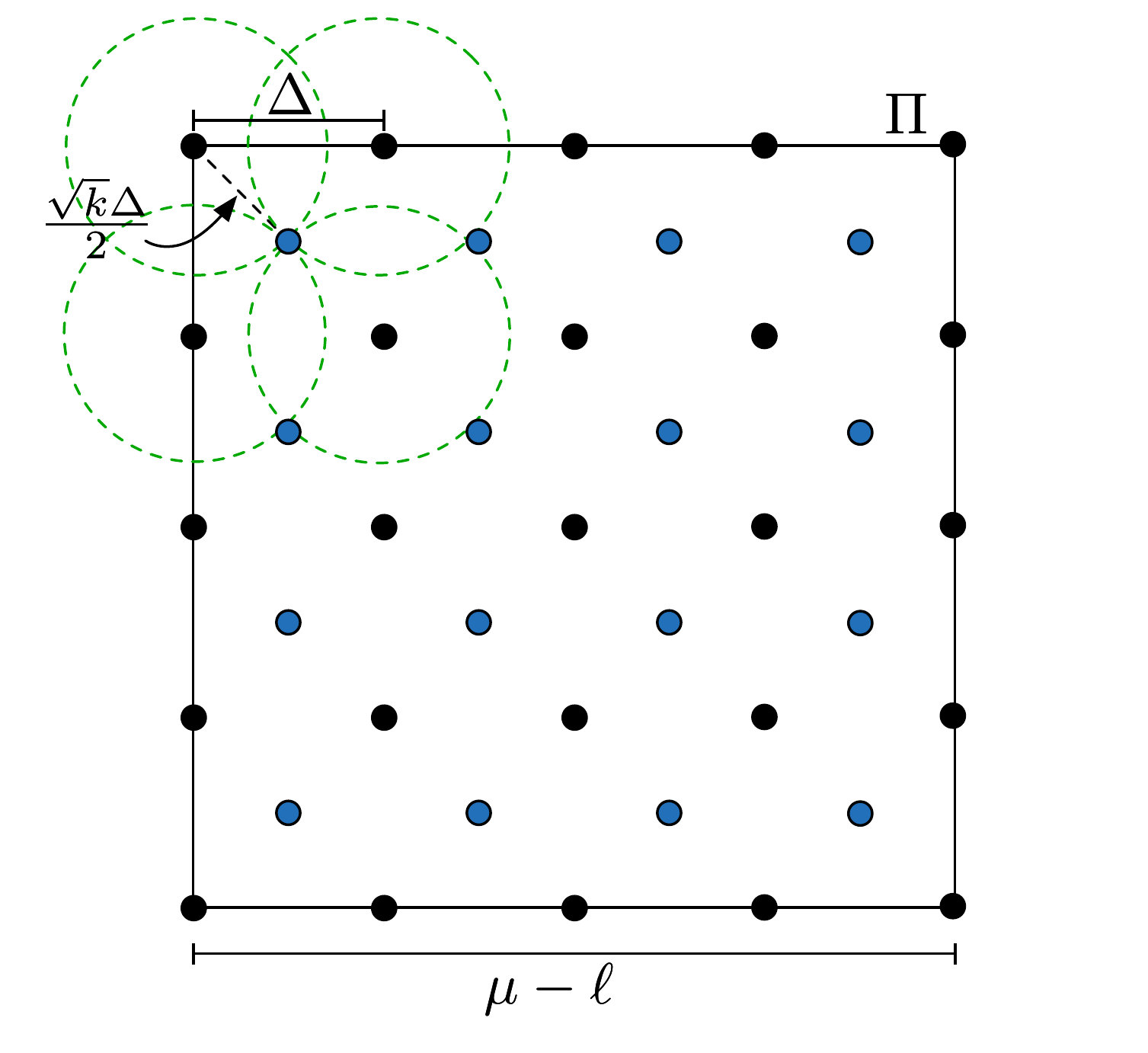}
\end{center}
\caption{To construct an $\delta$-cover we place sample points, shown here in black, along a regular grid with spacing $\Delta$. The blue points are the furthest points of $\Pi$ from the sample. To cover $\Pi$ we need $\Delta = 2\delta/\sqrt{k}$.}
\label{fig:grid}
\end{figure}

\label{sssec:volumemodels}
\cite{Madry17} suggest training a robust classifier with the help of an adversary which, at each iteration, produces $\epsilon$-perturbations around the training set that are incorrectly classified. In our notation, this corresponds to learning a decision boundary that correctly classifies $X^{\epsilon} = \{x \in \R^d: \|x - X_i\|_{2} \leq \epsilon \text{ for some training point } X_i\}$. We believe this approach is insufficiently robust in practice, as $X^{\epsilon}$ is often a poor model for $\mathcal{M}^{\epsilon}$. In this section, we show that the volume $\vol{X^{\epsilon}}$ is often a vanishingly small percentage of $\vol{\mathcal{M}^{\epsilon}}$. These results shed light on why the ball-based learning algorithm $\mathcal{L}$ defined in Section~\ref{sec:proverobust} is so much less sample-efficient than nearest neighbor classifiers. In Section \ref{ssec:codim} we experimentally verify these observations by showing that in high-dimensional space it is easy to find adversarial examples even after training against a strong adversary. For the remainder of this section we will consider the $\|\cdot\|_{2}$ norm. 

\begin{theorem}
\label{thm:manifoldvolumelowerbound}
Let $\mathcal{M} \subset \R^d$ be a $k$-dimensional manifold embedded in $\R^d$ such that $\vol_{k}{\mathcal{M}} < \infty$. Let $X \subset \mathcal{M}$ be a finite set of points sampled from $\mathcal{M}$. Suppose that $\epsilon \leq \rch_{2}{\Xi}$ where $\Xi$ is the medial axis of $\mathcal{M}$, defined as in \cite{Dey07}. Then the percentage of $\mathcal{M}^{\epsilon}$ covered by $X^{\epsilon}$ is upper bounded by
\begin{equation}
\label{equ:manifoldvolumelowerbound}
\frac{\vol{X^\epsilon}}{\vol{\mathcal{M}^\epsilon}} \leq \frac{\pi^{k/2} \Gamma(\frac{d-k}{2}+1)}{\Gamma(\frac{d}{2}+1)} \frac{\epsilon^{k}}{\vol_{k}{\mathcal{M}}} |X| \in \mathcal{O}\left(\left(\frac{2\pi}{d - k}\right)^{k/2} \frac{\epsilon^k}{\vol_{k}{\mathcal{M}}} |X| \right).
\end{equation}
As the codimension $(d - k) \rightarrow \infty$, Equation~\ref{equ:manifoldvolumelowerbound} approaches $0$, for any fixed $|X|$.
\end{theorem}
\begin{proof}
Assuming the balls centered on the samples in $X$ are disjoint we get the upper bound 
\begin{equation}
\label{equ:manifoldupperbound}
\vol{X^{\epsilon}} \leq \vol{B_{\epsilon} |X|} = \frac{\pi^{d/2}}{\Gamma(\frac{d}{2}+1)}\epsilon^d |X|.
\end{equation}
This is identical to the reasoning in Equation~\ref{equ:planeupperbound}.

The medial axis $\Xi$ of $\mathcal{M}$ is defined as the closure of the set of all points in $\R^d$ that have two or more closest points on $\mathcal{M}$ in the norm $\|\cdot\|_{2}$. The medial axis $\Xi$ is similar to the decision axis $\Lambda_2$, except that the nearest points do not need to be on distinct class manifolds. For $\epsilon \leq \rch_{2}{\Xi}$, we have the lower bound

\begin{equation}
\label{equ:manifoldtubelowerbound}
\vol{\mathcal{M}^{\epsilon}} \geq \vol_{d-k}{B^{d-k}_{\epsilon}} \vol_{k}{\mathcal{M}} = \frac{\pi^{(d-k)/2}}{\Gamma\left(\frac{d-k}{2}+1\right)}\epsilon^{d-k}  \vol_{k}{\mathcal{M}}.
\end{equation}

Combining Equations~\ref{equ:manifoldupperbound} and \ref{equ:manifoldtubelowerbound} gives the result. To get the asymptotic result we apply Stirling's approximation to get
\begin{align*}
\frac{\Gamma(\frac{d-k}{2}+1)}{\Gamma(\frac{d}{2}+1)} &\approx (2e)^{k/2} \frac{(d - k)^{(d - k + 1) / 2}}{d^{(d+1)/2}}\\
                                                      &= (2e)^{k/2} \frac{\left(\frac{d - k}{d}\right)^{(d+1)/2}}{(d - k)^{k/2}}\\
                                                      &= (2e)^{k/2} \frac{\left(1 - \frac{k}{d}\right)^{(d+1)/2}}{(d - k)^{k/2}}\\ 
                                                      &\approx \left(\frac{2}{d - k}\right)^{k/2}.
\end{align*}
The last step follows from the fact that $\lim_{d \rightarrow \infty} (1 - k/d)^{(d+1)/2} = e^{-k/2}$, where $e$ is the base of the natural logarithm.
\end{proof}

In high codimension, even moderate under-sampling of $\mathcal{M}$ leads to a significant loss of coverage of $\mathcal{M}^{\epsilon}$ because the volume of the union of balls centered at the samples shrinks faster than the volume of $\mathcal{M}^{\epsilon}$. Theorem~\ref{thm:manifoldvolumelowerbound} states that in high codimensions the fraction of $\mathcal{M}^{\epsilon}$ covered by $X^{\epsilon}$ goes to $0$. Almost nothing is covered by $X^{\epsilon}$ for training set sizes that are realistic in practice. Thus $X^{\epsilon}$ is a poor model of $\mathcal{M}^{\epsilon}$, and high classificaiton accuracy on $X^{\epsilon}$ does not imply high accuracy in $\mathcal{M}^{\epsilon}$. 

Note that an alternative way of defining the ratio $\vol{X^\epsilon} / \vol{\mathcal{M}^\epsilon}$ is as $\vol{(X^{\epsilon} \cap \mathcal{M}^{\epsilon})} / \vol{\mathcal{M}^{\epsilon}}$. This is equivalent in our setting since $X \subset \mathcal{M}$ and so $X^{\epsilon} \subset \mathcal{M}^{\epsilon}$. 

For the remainder of the section we provide intuition for Theorem~\ref{thm:manifoldvolumelowerbound} by considering the special case of $k$-dimensional planes. Define $\Pi = \{x \in \R^d: \ell \leq x_1, \ldots, x_k \leq \mu \text{ and } x_{k+1} = \ldots = x_d = 0\}$; that is $\Pi$ is a subset of the $x_1$-$\ldots$-$x_k$-plane bounded between the coordinates $[\ell, \mu]$. Recall that a $\delta$-cover of a manifold $\mathcal{M}$ in the norm $\|\cdot \|_{2}$ is a finite set of points $X$ such that for every $x \in \mathcal{M}$ there exists $X_i$ such that $\|x - X_i\|_{2} \leq \delta$. It is easy to construct an \emph{explicit} $\delta$-cover $X$ of $\Pi$: place sample points at the vertices of a regular grid, shown in Figure~\ref{fig:grid} by the black vertices. The centers of the cubes of this regular grid, shown in blue in Figure \ref{fig:grid}, are the furthest points from the samples. The distance from the vertices of the grid to the centers is $\sqrt{k}\Delta/2$ where $\Delta$ is the spacing between points along an axis of the grid. To construct a $\delta$-cover we need $\sqrt{k}\Delta/2 = \delta$ which gives a spacing of $\Delta = 2\delta/\sqrt{k}$. The size of this sample is $|X| = \left(\frac{\sqrt{k}(\mu - \ell)}{2 \delta}\right)^k$. Note that $|X|$ scales exponentially in $k$, the dimension of $\Pi$, not in $d$, the dimension of the embedding space. 

\begin{figure}[h!]
\begin{center}
\includegraphics[width=.4\linewidth]{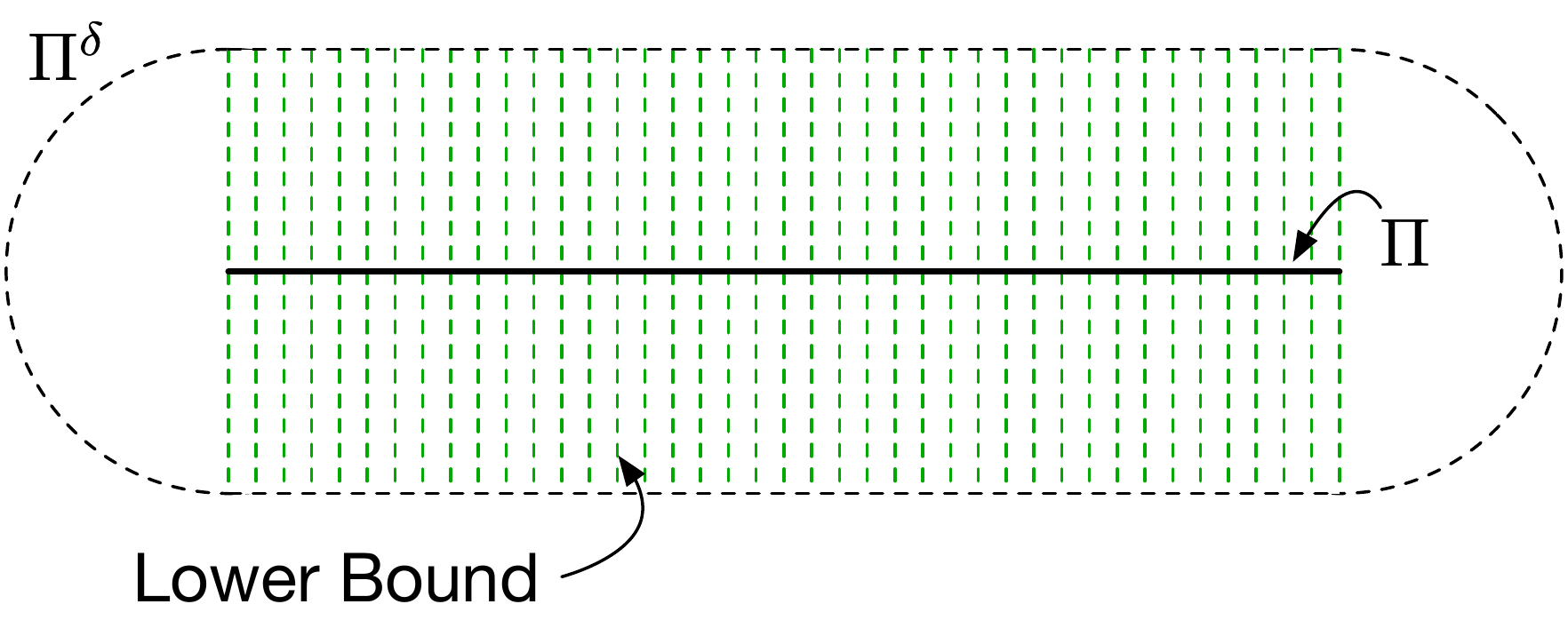}
\caption{An illustration of the lower bound technique used in Equation \ref{equ:tubelowerbound}. The volume $\vol{\Pi^{\delta}}$ shown in the black dashed lines, is bounded from below by placing a $(d-k)$-dimensional ball of radius $\delta$ at each point of $\Pi$, shown in green. In this illustration, a 1-dimensional manifold is embedded in 2 dimensions, so these balls are 1-dimensional line segments.}
\label{fig:volumelowerbound}
\end{center}
\end{figure}

Recall that $\Pi^{\delta}$ is the $\delta$-tubular neighborhood of $\Pi$. The $\delta$-balls around $X$, which comprise $X^{\delta}$, cover $\Pi$ and so any robust approach that guarantees correct classification within $X^{\delta}$ will achieve perfect accuracy on $\Pi$. However, we will show that $X^{\delta}$ covers only a vanishingly small fraction of $\Pi^{\delta}$. Let $B_{\delta}$ denote the $d$-ball of radius $\delta$ centered at the origin. An upper bound on the volume of $X^{\delta}$ is 
\begin{equation}
\label{equ:planeupperbound}
\vol{X^{\delta}} \leq \vol{B_{\delta} |X|} = \frac{\pi^{d/2}}{\Gamma(\frac{d}{2}+1)}\delta^d \left(\frac{\sqrt{k}(\mu - \ell)}{2\delta}\right)^k = \frac{\pi^{d/2}}{\Gamma(\frac{d}{2}+1)}\delta^{(d-k)} \left(\frac{\sqrt{k}(\mu - \ell)}{2}\right)^k.
\end{equation}

Next we bound the volume $\vol{\Pi^{\delta}}$ from below. Intuitively, a lower bound on the volume can be derived by placing a $(d-k)$-dimensional ball in the normal space at each point of $\Pi$ and integrating the volumes. Figure \ref{fig:grid} (Right) illustrates the lower bound argument in the case of $k = 1, d = 2$.
\begin{equation}
\label{equ:tubelowerbound}
\vol{\Pi^{\delta}} \geq \vol_{d-k}{B^{d-k}_{\delta}} \vol_{k}{\Pi} = \frac{\pi^{(d-k)/2}}{\Gamma\left(\frac{d-k}{2}+1\right)}\delta^{d-k} (\mu - \ell)^k.
\end{equation}

Combining Equations \ref{equ:planeupperbound} and \ref{equ:tubelowerbound} gives an upper bound on the percentage of $\Pi^{\delta}$ that is covered by $X^{\epsilon}$.
\begin{equation}
\label{equ:volumelowerbound}
\frac{\vol{X^{\delta}}}{\vol{\Pi^{\delta}}} \leq \frac{\pi^{k/2}\Gamma\left(\frac{d-k}{2}+1\right)}{\Gamma\left(\frac{d}{2}+1\right)} \left(\frac{\sqrt{k}}{2}\right)^k.
\end{equation}
Notice that the factors involving $\delta$ and $(\mu - \ell)$ cancel. Figure \ref{fig:planelowerbound} (Left) shows that this expression approaches $0$ as the codimension $(d-k)$ of $\Pi$ increases. 

Suppose we set $\delta = 1$ and construct a $1$-cover of $\Pi$. The number of points necessary to cover $\Pi$ with balls of radius $1$ depends \emph{only} on $k$, not the embedding dimension $d$. However the number of points necessary to cover the tubular neighborhood $\Pi^1$ with balls of radius $1$ increases depends on \emph{both} $k$ and $d$.  In Theorem \ref{thm:volumelowerbound} we derive a lower bound on the number of samples necessary to cover $\Pi^{1}$.

\begin{theorem}
\label{thm:volumelowerbound}
Let $\Pi$ be a bounded $k$-flat as described above, bounded along each axis by $\ell < \mu$. Let $n$ denote the number of samples necessary to cover the $1$-tubular neighborhood $\Pi^{1}$ of $\Pi$ with $\|\cdot\|_{2}$-balls of radius $1$. That is let $n$ be the minimum value for which there exists a finite sample $X$ of size $n$ such that $\Pi^{1} \subset \cup_{x \in X} B(x, 1) = X^{1}$. Then 
\begin{equation}
n \geq \frac{\pi^{-k/2} \Gamma\left(\frac{d}{2} + 1\right)}{\Gamma\left(\frac{d-k}{2}+1\right)}(\mu - \ell)^k \in \Omega\left(\left(\frac{d - k}{2\pi}\right)^{k/2}(\mu - \ell)^k\right).
\end{equation} 
\end{theorem}
\begin{proof}
We first construct an upper bound by generously assuming that the balls centered at the samples are disjoint. That is
\begin{equation}
\label{equ:lowerboundstep1}
\frac{\vol{X^{\delta}}}{\vol{\Pi^{\delta}}} \leq \frac{n \vol{B_{\delta}}}{\vol{\Pi^{\delta}}}.
\end{equation}
To guarantee that $\Pi^1 \subset \cup_{x \in X} B(x, 1) = X^{1}$ we set the left hand side of Equation \ref{equ:lowerboundstep1} equal to $1$ and solve for $n$.
\begin{align*}
1 = \frac{\vol{X^{\delta}}}{\vol{\Pi^{\delta}}} &\leq \frac{n \vol{B_{\delta}}}{\vol{\Pi^{\delta}}}\\
                                              n &\geq \frac{\vol{\Pi^{\delta}}}{\vol{B_{\delta}}}\\
                                                &\geq \frac{\pi^{-k/2} \Gamma\left(\frac{d}{2} + 1\right)}{\Gamma\left(\frac{d-k}{2}+1\right)}\left(\frac{\mu - \ell}{\delta}\right)^k
\end{align*}
The last inequality follows from Equation \ref{equ:tubelowerbound}. Setting $\delta = 1$ gives the result. The asymptotic result is similar to the argument in the proof of Theorem~\ref{thm:manifoldvolumelowerbound}.
\end{proof}

Theorem~\ref{thm:volumelowerbound} states that, in general, it takes many fewer samples to accurately model $\mathcal{M}$ than to model $\mathcal{M}^{\epsilon}$. Figure \ref{fig:planelowerbound} (Right) compares the number of points necessary to construct a $1$-cover of $\Pi$ with the lower bound on the number necessary to cover $\Pi^{1}$ from Theorem \ref{thm:volumelowerbound}. The number of points necessary to cover $\Pi^1$ increases as $\Omega\left((d-k)^{k/2}\right)$, scaling polynomially in $d$ and exponentially in $k$. In contrast, the number necessary to construct a $1$-cover of $\Pi$ remains constant as $d$ increases, depending only on $k$. 

Our lower bound of $\Omega\left((d-k)^{k/2}\right)$ samples is similar to the work of \cite{Schmidt18} who prove that, in the simple Gaussian setting, robustness \emph{requires} as much as $\Omega(\sqrt{d})$ more samples. Their arguments are statistical while ours are geometric.

\begin{figure}
\begin{center}
\begin{subfigure}{0.4\textwidth}
\includegraphics[width=0.98\linewidth]{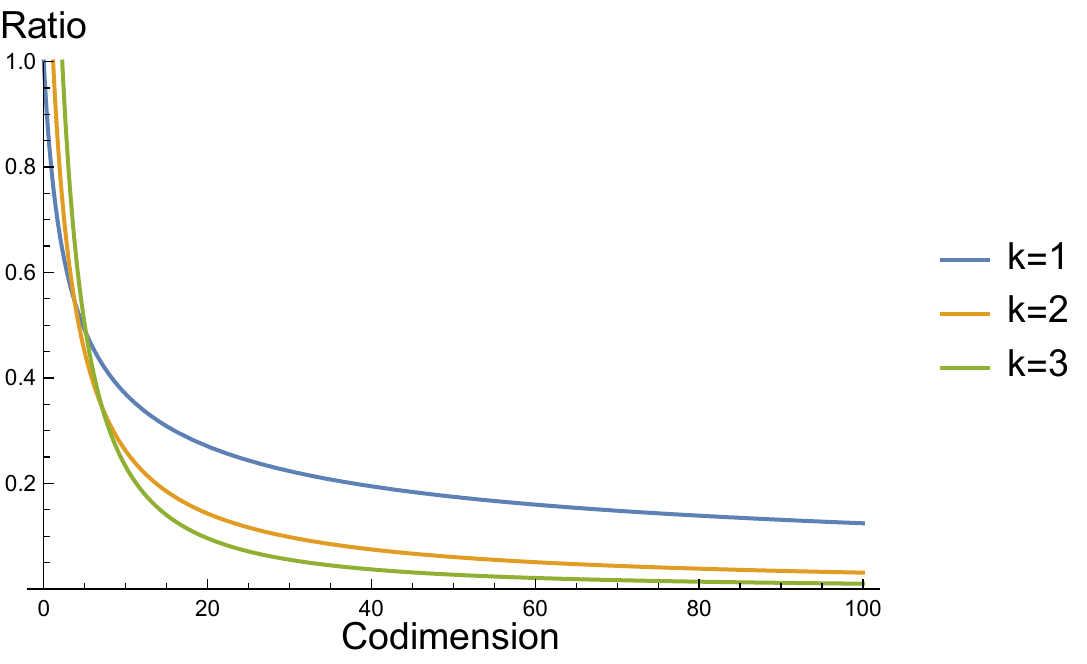}
\end{subfigure}
\begin{subfigure}{0.55\textwidth}
\includegraphics[width=0.98\linewidth]{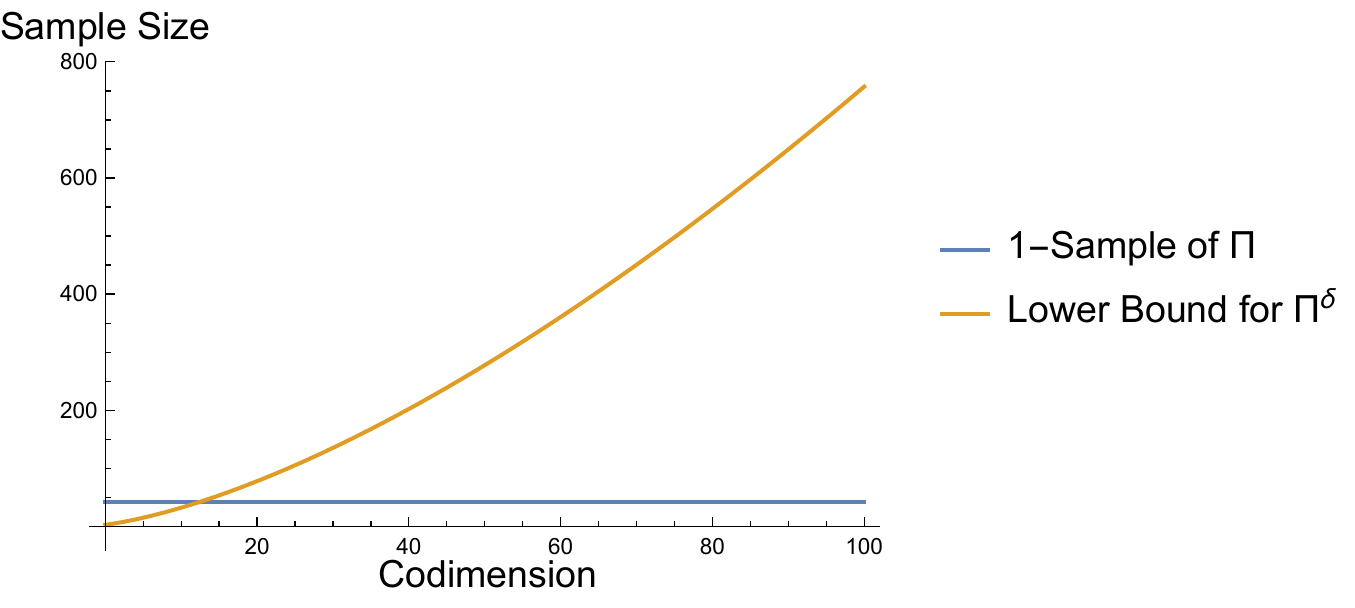}
\end{subfigure}
\caption{We plot the upper bound in Equation \ref{equ:volumelowerbound} on the left. As the codimension increases, the percentage of volume of $\Pi^{1}$ covered by $1$-balls around the $1$-sample approaches $0$. On the right we plot the number of samples necessary to cover $\Pi$, shown in blue, against the number of samples necessary to cover $\Pi^{1}$, shown in orange, as the codimension increases.}
\label{fig:planelowerbound}
\end{center}
\end{figure}

Approaches that produce robust classifiers by generating adversarial examples in the $\epsilon$-balls centered on the training set do not accurately model $\mathcal{M}^{\epsilon}$, and it will take \emph{many} more samples to do so. If the method behaves arbitrarily outside of the $\epsilon$-balls that define $X^{\epsilon}$, adversarial examples will still exist and it will likely be easy to find them. The reason deep learning has performed so well on a variety of tasks, in spite of the brittleness made apparent by adversarial examples, is because it is much easier to perform well on $\mathcal{M}$ than it is to perform well on $\mathcal{M}^{\epsilon}$.  

\section{A Lower Bound on Model Expressiveness}
\label{sec:modelsizelowerbound}
\subsection{A Simple Example}
\label{ssec:example}
Consider the case of two concentric circles $C_1, C_2$ with radii $r_1 < r_2$ respectively, as illustrated in Figure \ref{fig:simpleexample}. Each circle represents a different class of data. Suppose that we train a parametric model $f(x; \vtheta)$ with $p$ parameters so that for $x \in C_1$, $f(x; \vtheta) > 0$ and for $x \in C_2$, $f(x; \vtheta) < 0$. How does the number of parameters $p$ necessary to ensure that such a decision boundary can be expressed by $f(\cdot; \vtheta)$ increase as the gap between $C_1$ and $C_2$ decreases?

Suppose that we first lift $C_1$ and $C_2$ to a parabola in $\R^3$ via map $\phi(x_1, x_2) = (x_1, x_2, x_1^2 + x_2^2)$. That is, we construct the sets $C_1^{+} = \{\phi(x_1, x_2): (x_1, x_2) \in C_1\}$ and similarly for $C_2^{+}$. After applying $\phi$, $C_1^{+}$ and $C_{2}^{+}$ are \emph{linearly separable} for any $r_2 - r_1 > 0$. The linear decision boundary in $\R^3$ maps back to a circle in $\R^2$ that separates $C_1$ and $C_2$. This is not the case for deep networks; the number of parameters necessary to separate $C_1$ and $C_2$ will depend on the gap $r_2 - r_1$. 

In the important special case where $f$ is parameterized by a fully connected deep network with $\ell$ layers, $h$ hidden units per layer, and ReLU activations, \cite{Raghu17} prove that $f$ subdivides the input space into convex polytopes. In each convex polytope, $f$ defines a linear function that agrees on the boundary of the polytope with its neighbors. They showed that, when the inputs are in $\R^2$, the number of polytopes in the subdivision is at most $\mathcal{O}(h^{2\ell})$ (\cite{Raghu17}[Theorem 1]). 

Let $\mathcal{S}_{f}$ denote the subdivision of space into convex polytopes induced by $f$. Consider the decision boundary $\mathcal{D}_{f} = \{x \in \R^d: f(x; \vtheta) = 0\}$ of $f$. $\mathcal{D}_{f}$ can be constructed by examining each polytope $P \in \mathcal{S}_{f}$ and solving the linear equation $f_{P}(x) = 0$ where $f_{P}$ is the linear function defined on $P$ by $f$. Since $f_P$ is linear the solution is either (1) the empty set, (2) a \emph{single} line segment, or (3) all of $P$. Case (3) is a degenerate case and there are ways to perturb $f$ by an infinitesimally small amount such that case (3) never occurs and the classification accuracy is unchanged. Thus we conclude that $\mathcal{D}_{f}$ is a piecewise-linear curve comprised of line segments. (In higher dimensions $\mathcal{D}_{f}$ is composed of subsets of hyperplanes.) See Figure \ref{fig:simpleexample}.   

Suppose that $\mathcal{D}_{f}$ separates $C_1$ from $C_2$ and let $s \in \mathcal{D}_{f}$ be a line segment of the decision boundary. Since $s$ lies in the space between $C_1$ and $C_2$, the length $|s| \leq 2\sqrt{r_2^{2} - r_1^{2}}$, which is tight when $s$ is tangent to $C_1$ and touches $C_2$ at both of its endpoints. For $\mathcal{D}_{f}$ to separate $C_1$ from $C_2$, $\mathcal{D}_{f}$ must make a full rotation of $2\pi$ around the origin. The portion of this rotation that $s$ can contribute is upper bounded by $2 \arccos{\frac{r_1}{r_2}}$. Thus the number of line segments that comprise $\mathcal{D}_{f}$ is lower bounded by $\frac{\pi}{\arccos{\frac{r_1}{r_2}}}$. 

As $r_2 \rightarrow r_1$, the minimum number of segment necessary to separate $C_1$ from $C_2$ $\frac{\pi}{\arccos{\frac{r_1}{r_2}}} \rightarrow \infty$. Since each polytope $P \in \mathcal{S}_{f}$ can contribute at most one line segment to $\mathcal{D}_{f}$, the size of the model necessary to represent a decision boundary that separates $C_1$ from $C_2$ also increases as the circles get closer together. 

Now consider $C_{1}^{\epsilon}$ and $C_{2}^{\epsilon}$ under the $\|\cdot\|_{2}$ norm, defined as $C_{i}^{\epsilon} = \{x \in \R^2: \|x - C_i\|_{2} \leq \epsilon\}$. Suppose that a fully connected network $f$ described as above has sufficiently many parameters to represent a decision boundary that separates $C_1$ from $C_2$. Is $f$ also capable of learning a \emph{robust} decision boundary that separates $C_{1}^{\epsilon}$ from $C_{2}^{\epsilon}$?

\begin{figure}[h!]
\begin{center}
\includegraphics[width=0.9\textwidth]{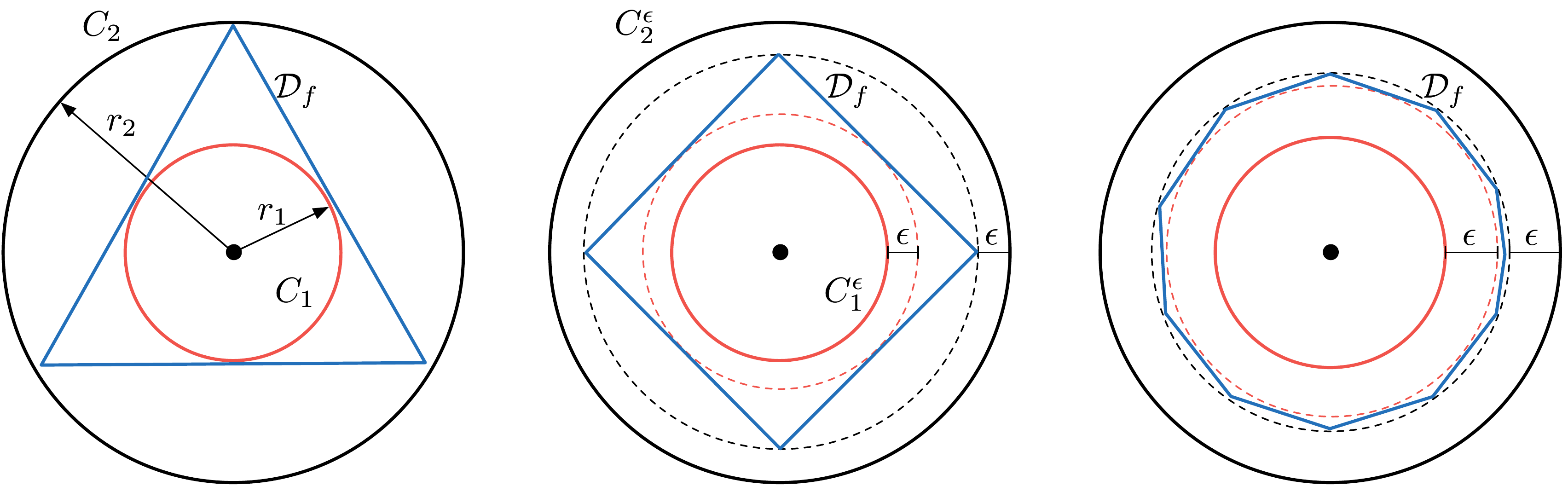}
\caption{Separating two classes of data sampled from $C_1$ and $C_2$ may require a decision boundary $\mathcal{D}_{f}$ with only a few linear segments. However a decision boundary $\mathcal{D}_{f}$ that is robust to $\epsilon$-perturbations must lie in gap between $C_{1}^{\epsilon}$ and $C_{2}^{\epsilon}$. Learning a robust decision boundary may require more linear segments and thus a more expressive model. As we increase $\epsilon$, demanding a more robust decision boundary, the gap between $C_{1}^{\epsilon}$ and $C_{2}^{\epsilon}$ decreases, and so the number of linear segments increases towards $\infty$. }
\label{fig:simpleexample}
\end{center}
\end{figure} 

For $\mathcal{D}_{f}$ to separate $C_{1}^{\epsilon}$ from $C_{2}^{\epsilon}$ it must lie in the region between $C_1^{\epsilon}$ and $C_{2}^{\epsilon}$. In this setting each segment can contribute at most $2 \arccos{\frac{r_1 + \epsilon}{r_2 - \epsilon}}$ to the full $2\pi$ rotation around the origin. The minimum number of line segments that comprise a robust decision boundary $\mathcal{D}_{f}$ is lower bounded by $\frac{\pi}{\arccos{\frac{r_1 +\epsilon}{r_2 - \epsilon}}}$. As $\epsilon \rightarrow \frac{r_2 - r_{1}}{2}$ this quantity approaches $\infty$. Even if $f$ is capable of separating $C_1$ from $C_2$ we can choose $\epsilon$ such that $\frac{\pi}{\arccos{\frac{r_1 +\epsilon}{r_2 + \epsilon}}} \in \omega(h^{2\ell})$. 

This simple example shows that learning decision boundaries that are robust to $\epsilon$-adversarial examples may require substantially more powerful models than what is required to learn the original distributions. Furthermore the amount of additional resources necessary is dependent upon the amount of robustness required. 

\subsection{An Exponential Lower Bound}
\label{ssec:proofmodelsizelowerbound}
We present an exponential lower bound on the number of linear regions necessary to represent a decision boundary that is robust to $\|\cdot\|_{2}$-perturbations of at most $\epsilon \leq \rch_{2}{\Lambda_2} - \tau$, in the simple case of two concentric $(d-1)$-spheres. 

\begin{theorem}
\label{thm:modelsizelowerbound}
Let $S_1, S_2 \subset \R^d$ be two concentric $(d-1)$-spheres with radii $r_1 < r_2$ respectively and let $S = S_1 \cup S_2$. Let $f: \R^d \rightarrow \R$ be a fully connected neural network with ReLU activations. Suppose that $f$ correctly classifies $S^{\rch_2{\Lambda_2} - \tau}$ for some $\tau \in [0, \rch_2{\Lambda_2}]$. Said differently, the decision boundary of $f$ lies in a $\tau$-tubular neighborhood of the decision axis, $\mathcal{D}_{f} \subset \Lambda_{2}^{\tau}$. Then the number of linear regions $N$ into which $f$ subdivides $\R^d$ is lower bounded as 
\begin{equation}
N \geq 2\sqrt{\pi} \frac{\Gamma(\frac{d+1}{2})}{\Gamma(\frac{d}{2})} \left(\frac{r_1 + \rch_2{\Lambda_2}}{4\tau}\right)^{\frac{d-1}{2}}.
\end{equation}
Written asymptotically, $N \in \Omega\left(\frac{\sqrt{d}}{2^d} \left(\frac{r_1 + \rch_2{\Lambda_2}}{\tau}\right)^{\frac{d-1}{2}}\right)$
\end{theorem}
\begin{proof}
For $f$ to be robust to $\epsilon$-adversarial examples for $\epsilon \leq \rch_{2}{\Lambda_2} - \tau$ the decision boundary $\mathcal{D}_{f} \subset \Lambda^{\tau}$. The boundary of $\Lambda^{\tau}$ is comprised of two disjoint $(d-1)$-spheres, which we will denote as $\partial\Lambda_{1}^{\tau}$ and $\partial\Lambda_{2}^{\tau}$ with radii $r_1 + \rch_2{\Lambda_2} - \tau$ and $r_1 + \rch_2{\Lambda_2} + \tau$ respectively. (It is standard in topology to use the $\partial$ symbol to denote the boundary of a topological space.)

The isoperimetric inequality states that a $(d-1)$-sphere minimizes the $(d-1)$-dimensional volume (thought of as ``surface area'') across all sets with fixed $d$-dimensional volume (thought of as ``volume''). Since $\mathcal{D}_{f} \subset \Lambda^{\tau}$, the $d$-dimensional volume enclosed by $\mathcal{D}_{f}$ is at least as large as that of $\partial \Lambda^{\tau}_{1}$ and so we have that $\surf{\partial \Lambda_{1}^{\tau}} \leq \surf{\mathcal{D}_{f}}$. 

Now consider any $(d-1)$-dimensional linear facet $\Pi$ of the decision boundary $\mathcal{D}_{f}$. The normal space of $\Pi$ is $1$-dimensional; let $\vn$ denote a unit vector orthogonal to $\Pi$. (There are two possible choices $\vn$ and $-\vn$.) Due to the spherical symmetry of $\Lambda^{\tau}$ and the fact that $\Pi \subset \Lambda^{\tau}$, the diameter of $\Pi$ is maximized when $\Pi$ is tangent to $\partial \Lambda_{1}^{\tau}$ at $(r_1 + \rch_2{\Lambda_2} - \tau) \vn$ (or $-(r_1 + \rch_2{\Lambda_2} - \tau) \vn$) and intersects $\partial \Lambda_{2}^{\tau}$. In pursuit of an upper bound, we will assume without loss of generality that $\Pi$ has these properties. Let $o$ denote the origin, $x = (r_1 + \rch_2{\Lambda_2} - \tau) \vn$, and $y \in \Pi \cap \partial \Lambda_{2}^{\tau}$. We consider the right triangle $\triangle oxy$ with right angle at $x$. By basic properties of right triangles, $\frac{\diam{\Pi}}{2} \leq \|x - y\|_{2} = \sqrt{(r_1+\rch_2{\Lambda_2}+\tau)^2 - (r_1 + \rch_2{\Lambda_2} - \tau)^2} = \sqrt{4\tau (r_1 + \rch_2{\Lambda_2})}$. It follows that $\Pi$ is contained in a $(d-1)$-dimensional ball of radius $\sqrt{4\tau (r_1 + \rch_2{\Lambda_2})}$. In particular the $(d-1)$-dimensional volume of $\Pi$ is bounded as $\vol_{d-1}(\Pi) \leq \vol_{d-1}{B(0, \sqrt{4\tau (r_1 + \rch_2{\Lambda_2})})}$. The $(d-1)$-dimensional volume of $\mathcal{D}_{f}$ (again thought of as ``surface area''), is equal to the sum of the $(d-1)$-dimensional volumes of the linear facets that comprise $\mathcal{D}_{f}$. Combining these inequalities gives the result.

\begin{align*}
\frac{2 \pi^{\frac{d}{2}}}{\Gamma(\frac{d}{2})} (r_1 + \rch_2{\Lambda_2})^{d-1} = \surf{\partial \Lambda_{1}^{\tau}} &\leq \surf{\mathcal{D}_{f}}\\
                                   &\leq N \vol_{d-1}{B(0, \sqrt{4\tau (r_1 + \rch_2{\Lambda_2})})}\\
                                   &\leq N \frac{\pi^{\frac{d-1}{2}}}{\Gamma(\frac{d+1}{2})} \left(4\tau (r_1 + \rch_2{\Lambda_2})\right)^{\frac{d-1}{2}}\\ 
2\sqrt{\pi} \frac{\Gamma(\frac{d+1}{2})}{\Gamma(\frac{d}{2})} \left(\frac{r_1 + \rch_2{\Lambda_2}}{4\tau}\right)^{\frac{d-1}{2}} &\leq N
\end{align*}  
\end{proof}

Prior work has experimentally verified that increasing the size of deep networks improves robustness (\cite{Madry17}). Theorem~\ref{thm:modelsizelowerbound} proves that there are settings in which robustness \emph{requires} larger models. 

\section{Experiments}
\label{sec:exp}
\label{sec:experiments}

\subsection{High Codimension Reduces Robustness}
\label{ssec:codim}
Section \ref{sssec:volumemodels} suggests that as the codimension increases it should become easier to find adversarial examples. To verify this, we introduce two synthetic datasets, {\sc Circles} and {\sc Planes}, which allow us to carefully vary the codimension while maintaining dense samples. The {\sc Circles} dataset consists of two concentric circles in the $x_1$-$x_2$-plane, the first with radius $r_1 = 1$ and the second with radius $r_2 = 3$, so that $\rch_2{\Lambda_2} = 1$. We densely sample $1000$ random points on each circle for both the training and the test sets. The {\sc Planes} dataset consists of two $2$-dimensional planes, the first in the $x_d=0$ and the second in $x_d = 2$, so that $\rch_2{\Lambda_2} = 1$. The first two axis of both planes are bounded as $-10 \leq x_1, x_2 \leq 10$, while $x_3 = \ldots = x_{d-1} = 0$. We sample the training set at the vertices of the grid described in Section~\ref{sssec:volumemodels}, and the test set at the centers of the grid cubes, the blue points in Figure~\ref{fig:grid}. Both planes are sampled so that the $1$-tubular neighborhood $X^{1}$ covers the underlying planes, where $X$ is the training set. See Figure \ref{fig:datasetsvis} for a visualization of {\sc Circles} and {\sc Planes}.

\begin{figure}[h!]
\begin{center}
\begin{subfigure}{0.4\textwidth}
\includegraphics[width=0.97\linewidth]{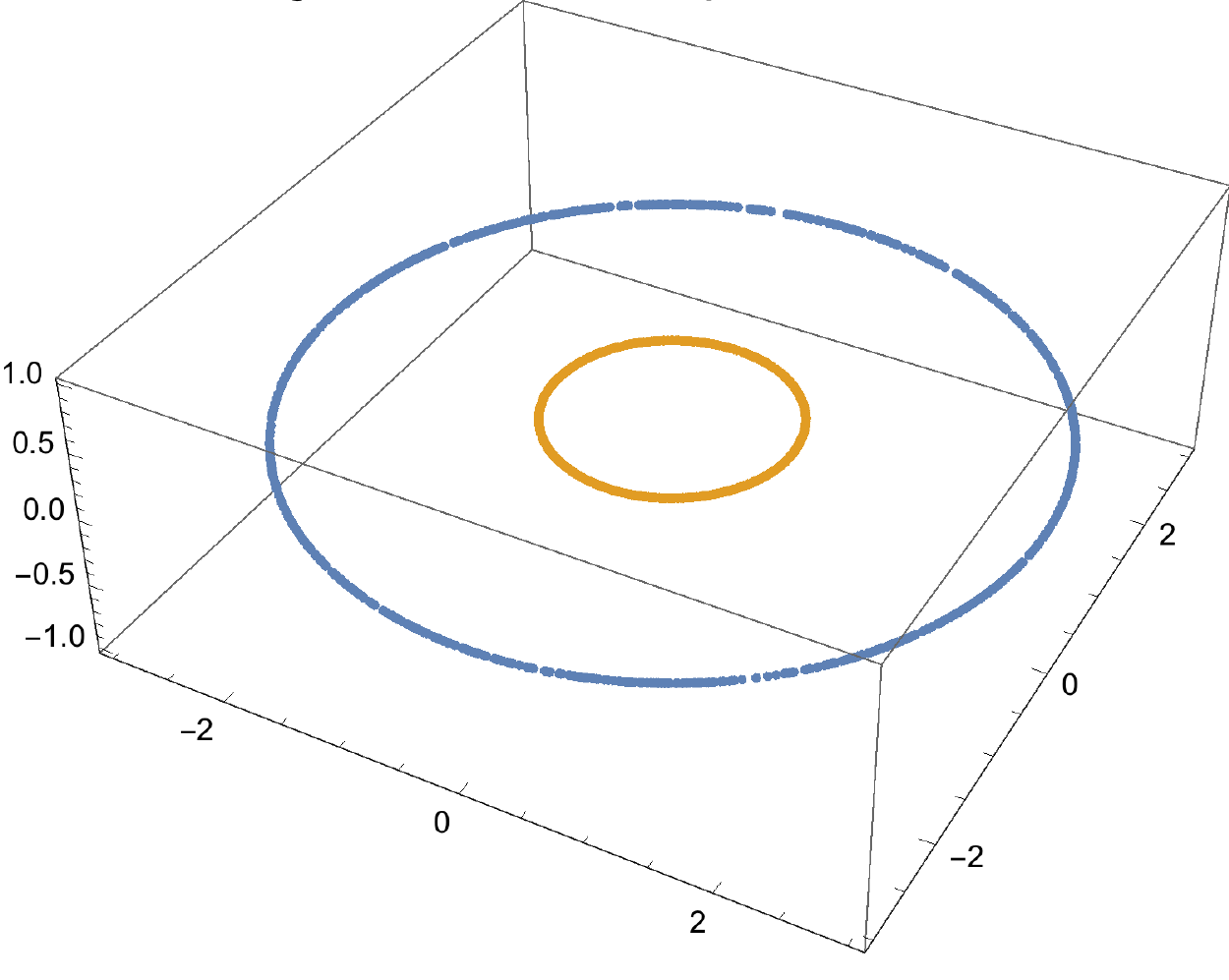}
\end{subfigure}
\begin{subfigure}{0.4\textwidth}
\includegraphics[width=0.97\linewidth]{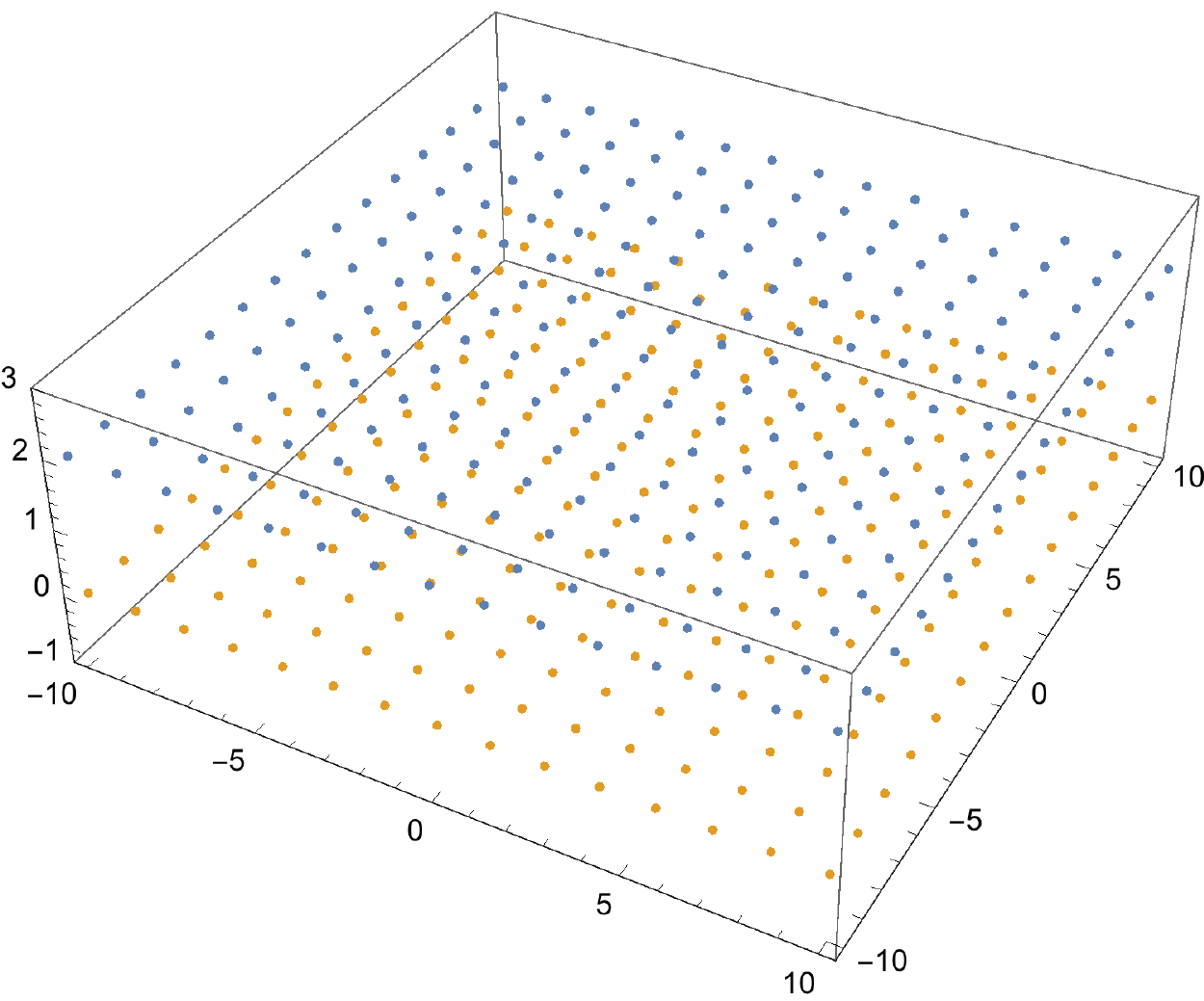}
\end{subfigure}
\caption{We create two synthetic datasets which allow us to perform controlled experiments on the affect of codimension on adversarial examples. {\bf Left}: {\sc Circles}, {\bf Right}: {\sc Planes}}
\label{fig:datasetsvis}
\end{center}
\end{figure}

We consider two attacks, the fast gradient sign method (FGSM) (\cite{Goodfellow14}) and the basic iterative method (BIM) (\cite{Kurakin16}) under $\|\cdot\|_{2}$. We use the implementations provided in the cleverhans library (\cite{Papernot18}). Further implementation details are provided in Appendix~\ref{sec:impdets}. Our experimental results are averaged over 20 retrainings of our model architecture, using Adam (\cite{Kingma15}). Further implementation details are provided in Appendix~\ref{sec:impdets}. 

Figure \ref{fig:codimexp} (Top Left, Bottom Left) shows the robustness of naturally trained networks to FGSM and BIM attacks on the {\sc Circles} dataset as we increase the codimension. For both attacks we see a steady decrease in robustness as we increase the codimension, on average. The result is reproducible with other optimization procedures; Figure \ref{fig:codimexpsgd} in Appendix~\ref{ssec:sgdexps} shows the results for SGD.

In Figure~\ref{fig:codimexp} (Top Right, Bottom Right), we use a nearest neighbor (NN) classifier to classify the adversarial examples generated by FGSM and BIM for our naturally trained networks on {\sc Circles}. Nearest neighbors is robust even when the codimension is high, as long as the low-dimensional data manifold is well sampled. This is a consequence of the fact that the Voronoi cells of the samples are elongated in the directions normal to the data manifold when the sample is dense (\cite{Dey07}). 

\begin{figure}[h!]
\begin{center}
\begin{subfigure}{0.45\textwidth}
\includegraphics[width=0.99\linewidth]{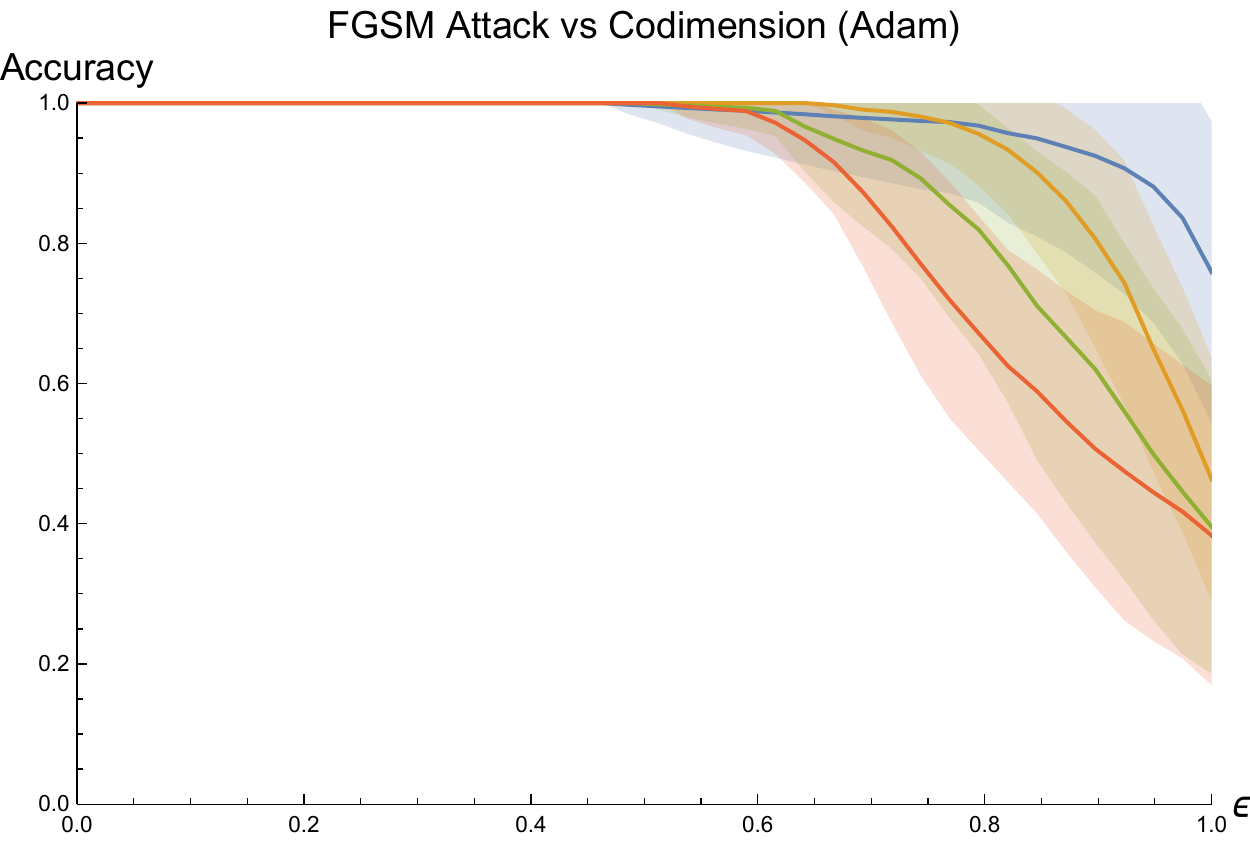}
\end{subfigure}
\begin{subfigure}{0.45\textwidth}
\includegraphics[width=0.98\linewidth]{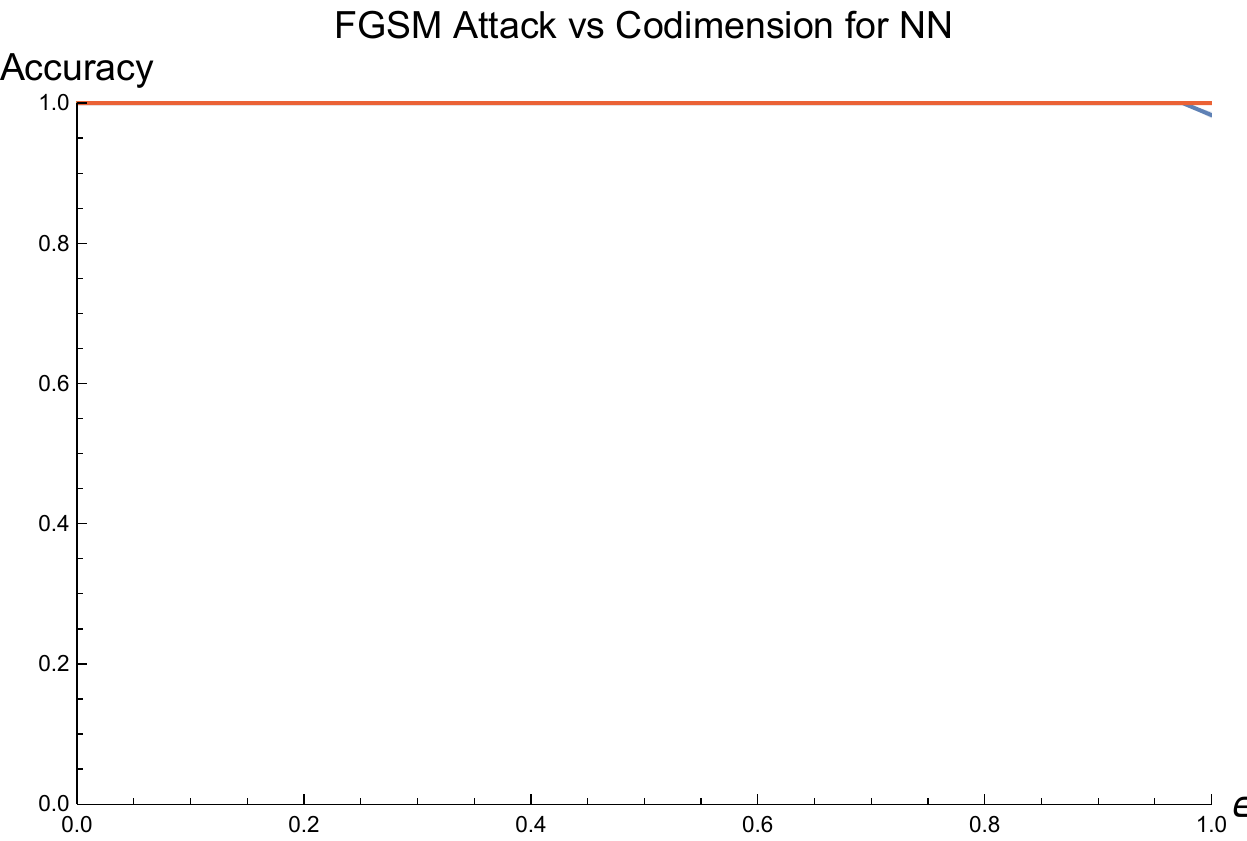}
\end{subfigure}
\begin{subfigure}{0.08\textwidth}
\includegraphics[width=0.99\linewidth]{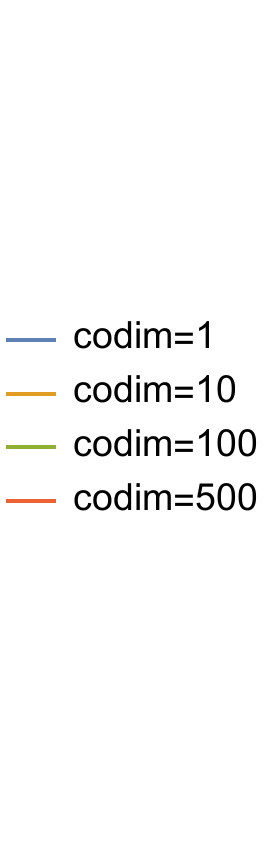}
\end{subfigure}
\begin{subfigure}{0.45\textwidth}
\includegraphics[width=0.99\linewidth]{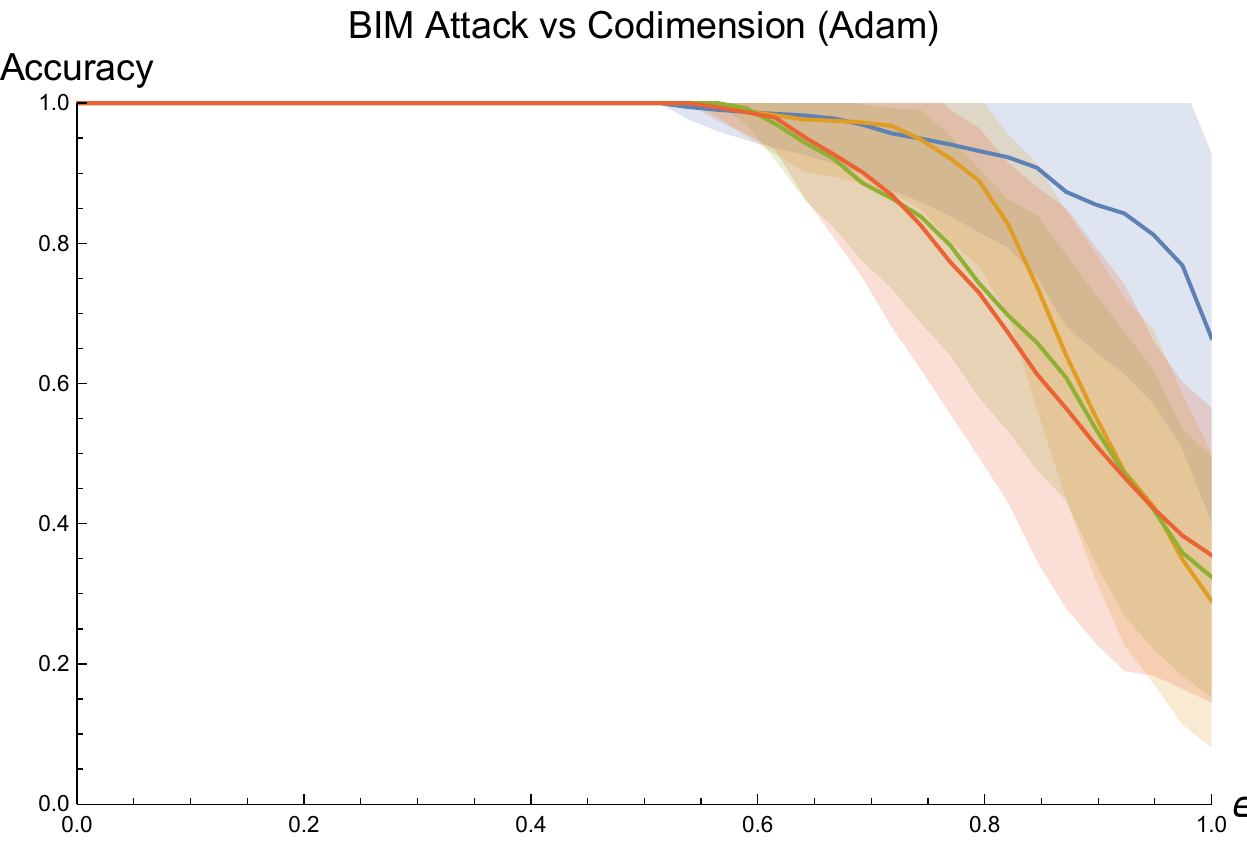}
\end{subfigure}
\begin{subfigure}{0.45\textwidth}
\includegraphics[width=0.98\linewidth]{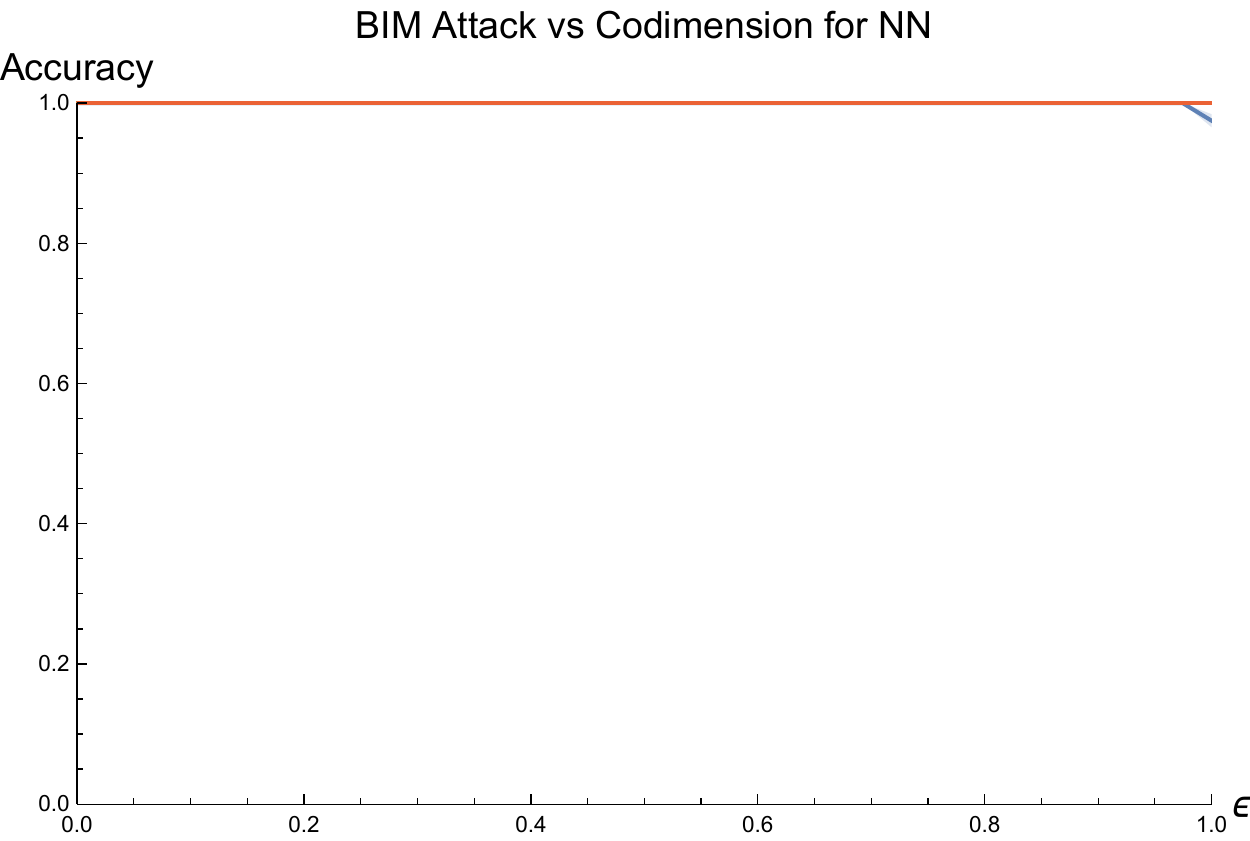}
\end{subfigure}
\begin{subfigure}{0.08\textwidth}
\includegraphics[width=0.99\linewidth]{figures/codim-legend}
\end{subfigure}
\caption{As the codimension increases the robustness of decision boundaries learned by Adam on naturally trained networks decreases steadily. \textbf{Top Left}: Effectiveness of FGSM attacks as codimension increases. \textbf{Bottom Left}: BIM attacks. \textbf{Top and Bottom Right}: A nearest neighbor classifier exhibits essentially perfect accuracy to the adversarial examples generated for our naturally trained networks by FGSM and BIM for all $\epsilon$ and codimension.}
\label{fig:codimexp}
\end{center}
\end{figure} 

\cite{Madry17} propose training against a projected gradient descent (PGD) adversary to improve robustness. Section \ref{sssec:volumemodels} suggests that this should be insufficient to guarantee robustness, as $X^{\epsilon}$ is often a poor model for $\mathcal{M}^{\epsilon}$. We follow the adversarial training procedure of \cite{Madry17} by against a PGD adversary with $\epsilon = 1$ under $\|\cdot\|_{2}$-perturbations on the {\sc Planes} dataset. Figure~\ref{fig:advcodimexp} (Left) shows that it is still easy to find adversarial examples for $\epsilon < 1$ and that as the codimension increases we can find adversarial examples for decreasing values of $\epsilon$. In contrast, a nearest neighbor classifier (Right) achieves perfect robustness for all $\epsilon$ on this data.

\begin{figure}[h!]
\begin{center}
\begin{subfigure}{0.45\textwidth}
\includegraphics[width=0.99\linewidth]{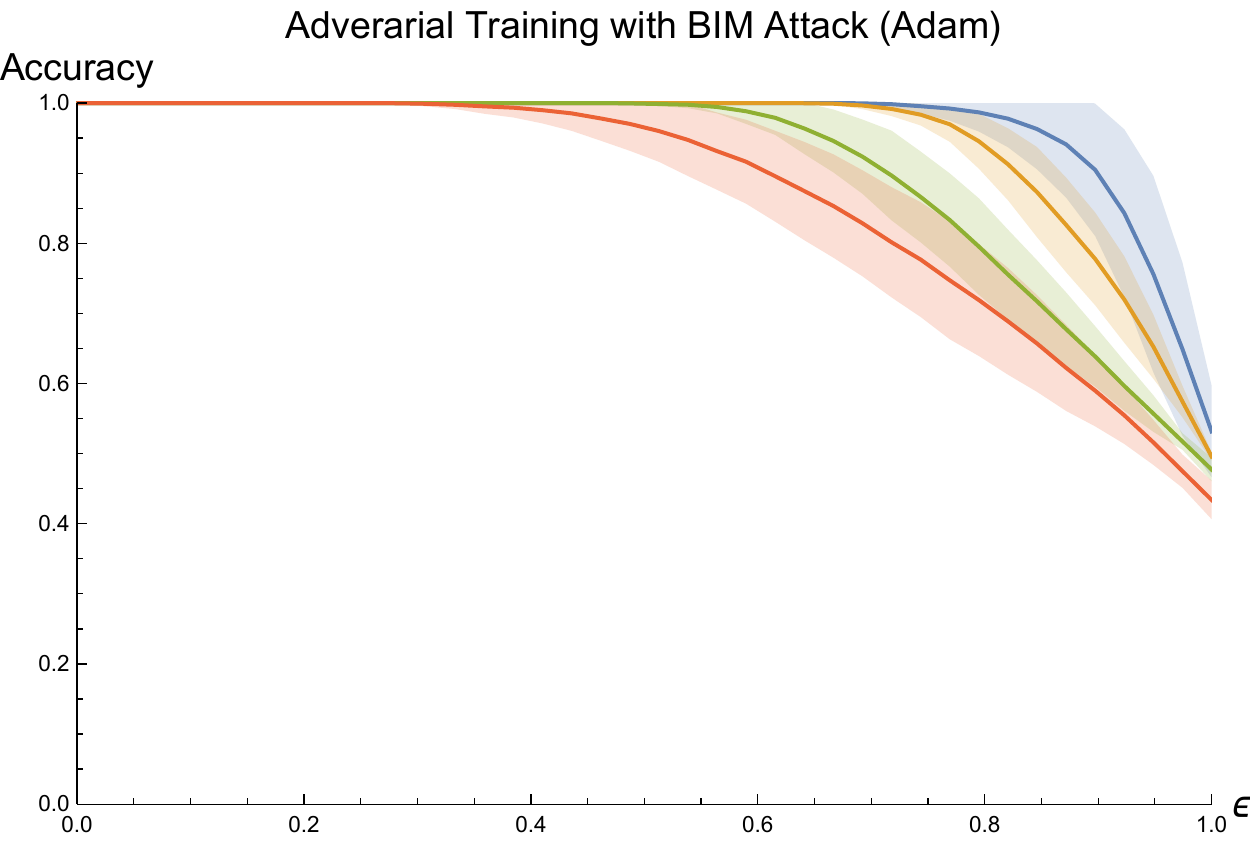}
\end{subfigure}
\begin{subfigure}{0.45\textwidth}
\includegraphics[width=0.99\linewidth]{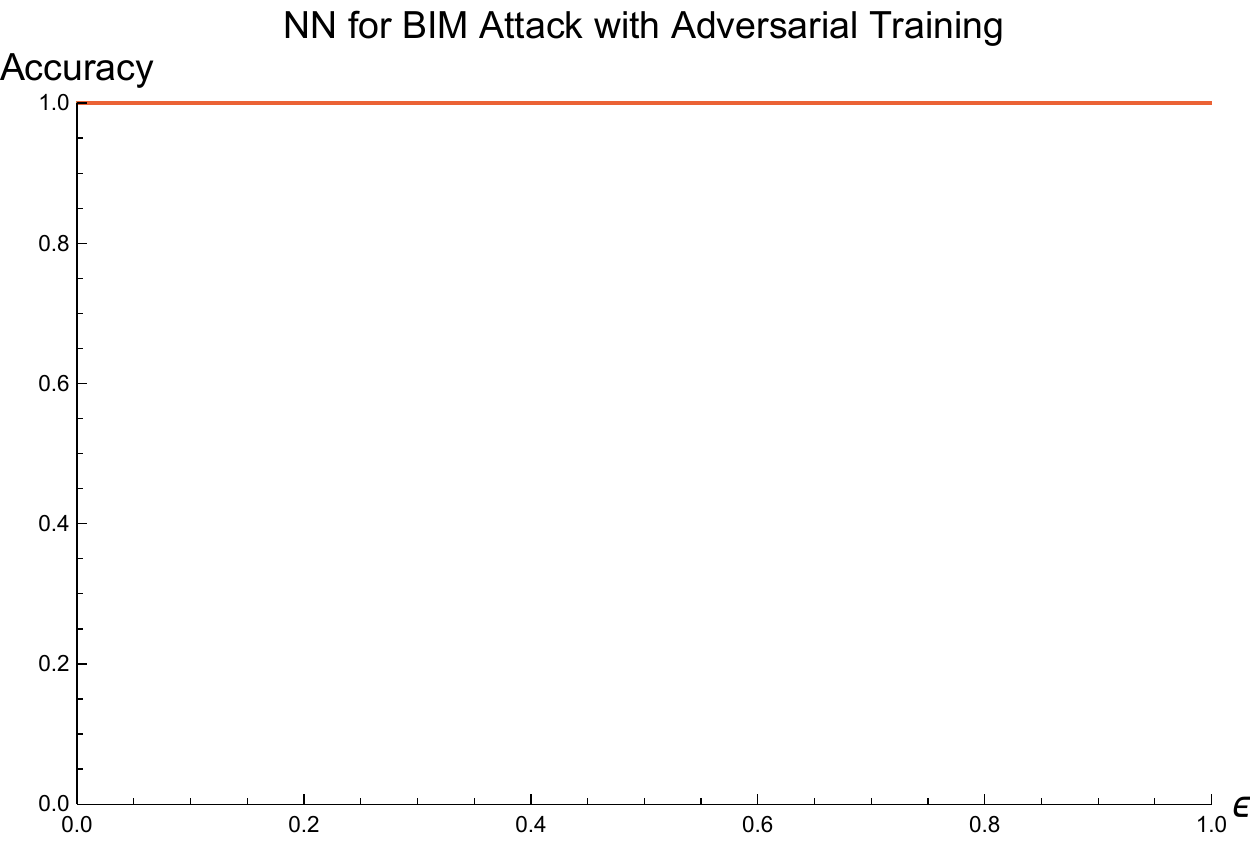}
\end{subfigure}
\begin{subfigure}{0.08\textwidth}
\includegraphics[width=0.99\linewidth]{figures/codim-legend}
\end{subfigure}
\caption{\textbf{Left}: Training using the adversarial training procedure of \cite{Madry17} is no guarantee of robustness; as the codimension increases it becomes easier to find adversarial examples using BIM attacks. \textbf{Right}: The performance of a nearest neighbor classifier on this data is perfect for all $\epsilon$ and codimension.}
\label{fig:advcodimexp}
\end{center}
\end{figure} 

The {\sc Planes} dataset is sampled so that the trianing set is a $1$-cover of the underlying planes, which requires 450 sample points. Figure~\ref{fig:samplingdensityexp} shows the results of increasing the sampling density to a $0.5$-cover (1682 samples) and a $0.25$-cover (6498 samples). Increasing the sampling density improves the robustness of adversarial training at the same codimension and particularly in low-codimension. However adversarial training with a substantially larger training set does not produce a classifier as robust as a nearest neighbor classifier on a much smaller training set. Nearest neighbors is much more sample efficient than adversarial training, as predicted by Theorem \ref{thm:sampling}.

\begin{figure}
\begin{center}
\begin{subfigure}{0.3\textwidth}
\includegraphics[width=0.99\linewidth]{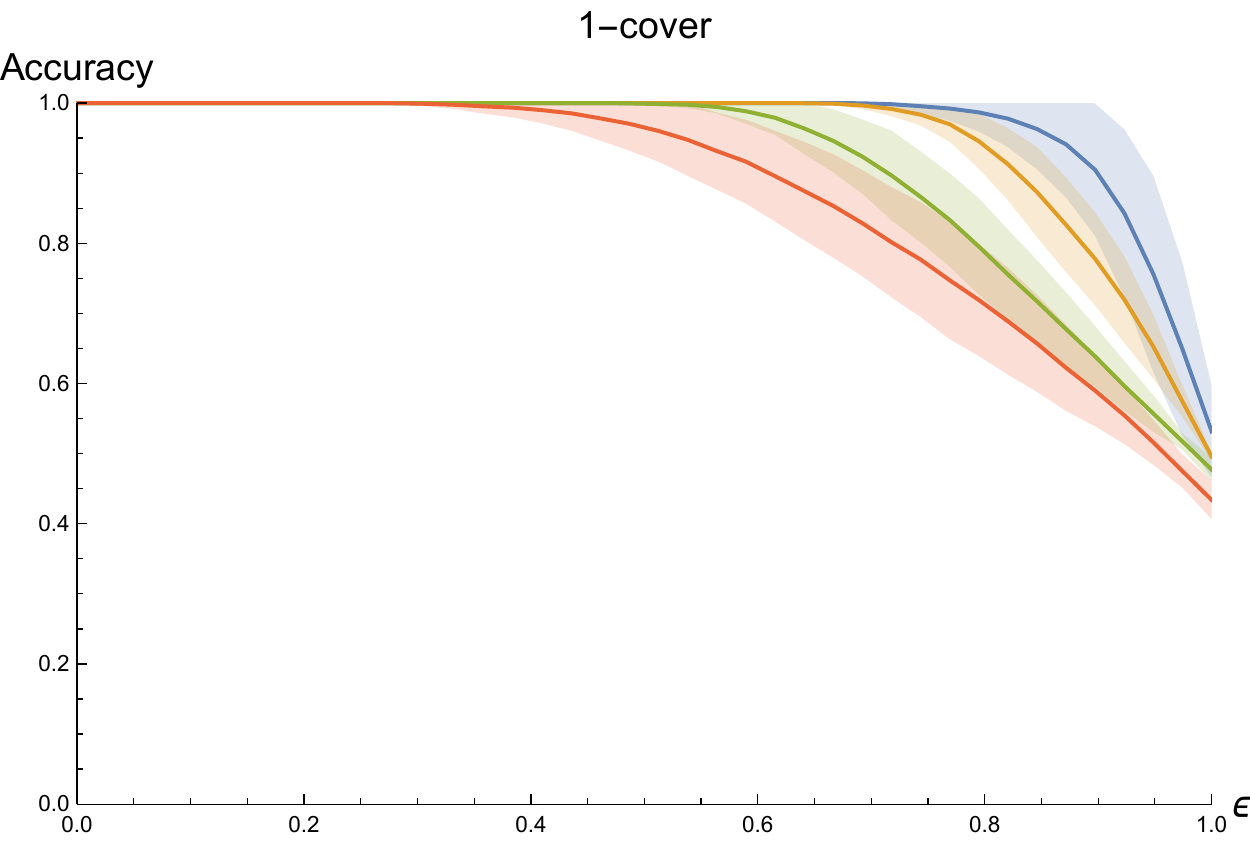}
\end{subfigure}
\begin{subfigure}{0.3\textwidth}
\includegraphics[width=0.99\linewidth]{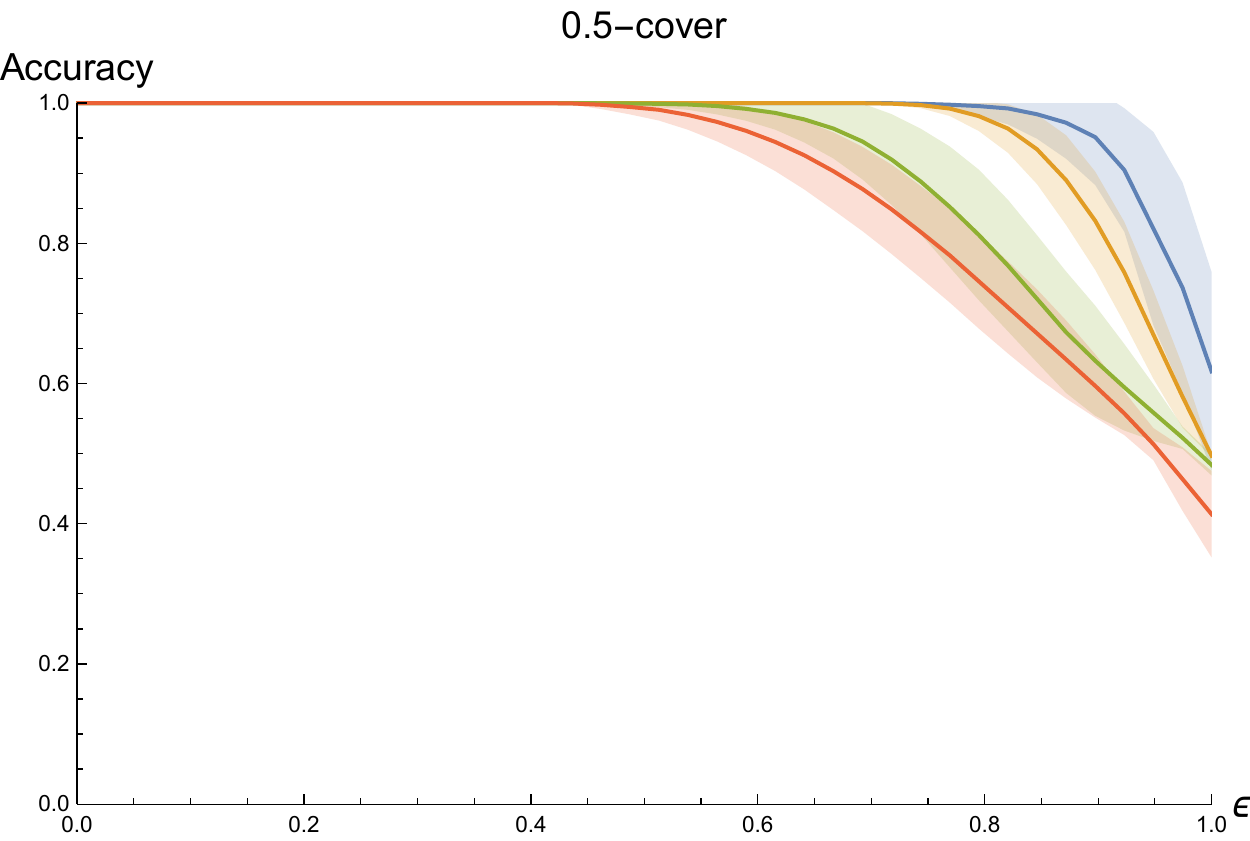}
\end{subfigure}
\begin{subfigure}{0.3\textwidth}
\includegraphics[width=0.99\linewidth]{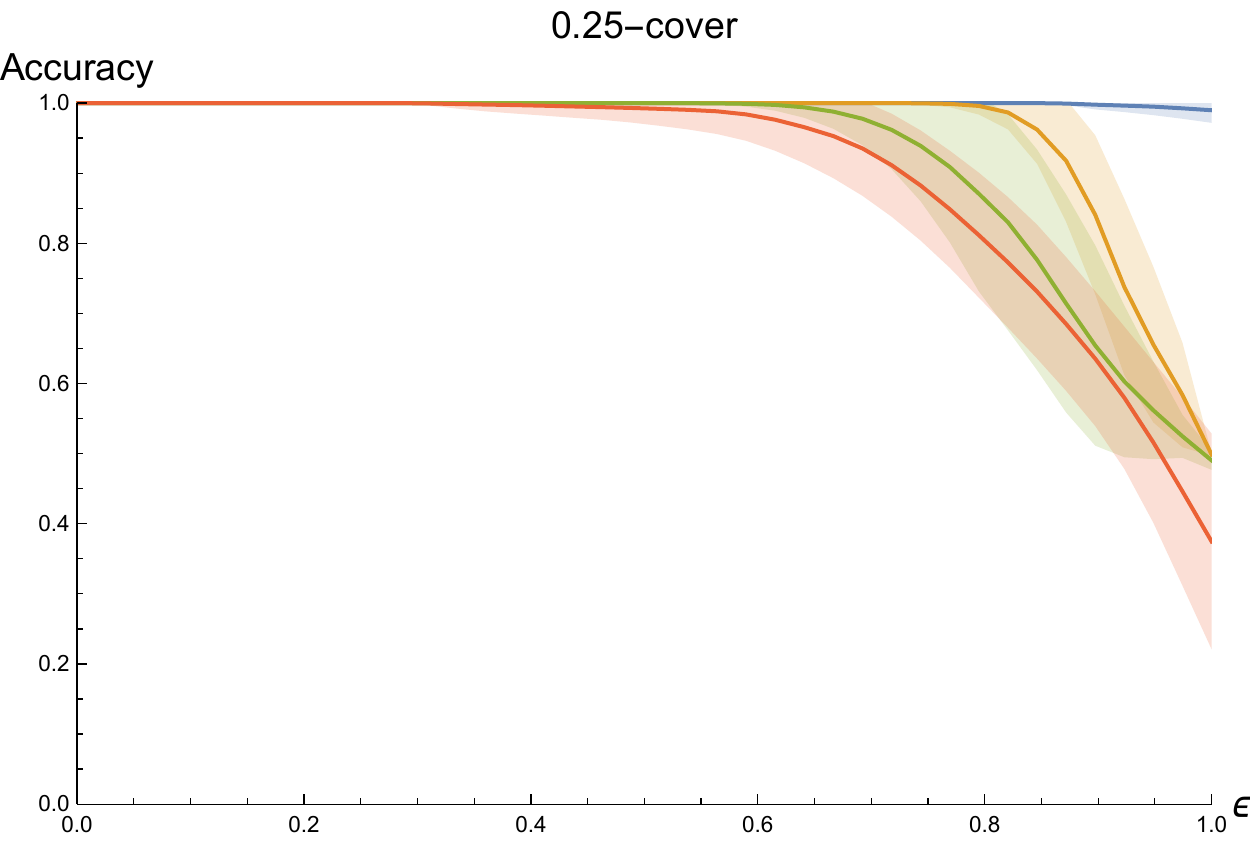}
\end{subfigure}
\begin{subfigure}{0.08\textwidth}
\includegraphics[width=0.99\linewidth]{figures/codim-legend}
\end{subfigure}
\caption{Adversarial training of \cite{Madry17} on the {\sc Planes} dataset with a $1$-cover (left), consisting of $450$ samples, a $0.5$-cover (center), $1682$ samples, and a $0.25$-cover (right), $6498$ samples. Increasing the sampling density improves robustness at the same codimension. However even training on a significantly denser training set does not produce a classifier as robust as a nearest neighbor classifier on a much sparser training set, Figure \ref{fig:advcodimexp} (Right).}
\label{fig:samplingdensityexp}
\end{center}
\end{figure}

\subsection{Adversarial Perturbations are in the Directions Normal to the Data Manifold}
\label{ssec:normexps}
Let $\eta_{x}$ be an adversarial perturbation generated by FGSM with $\epsilon = 1$ at $x \in \mathcal{M}$. Note that the adversarial example is constructed as $\hat{x} = x + \eta_{x}$. In Figure~\ref{fig:angleexp} we plot a histogram of the angles $\angle(\eta_{x}, N_{x}\mathcal{M})$ between $\eta_{x}$ and the normal space $N_{x}\mathcal{M}$ for the {\sc Circles} dataset in codimensions $1, 10, 100$, and $500$. In codimension $1$, $88\%$ of adversarial perturbations make an angle of less than $10^{\circ}$ with the normal space. Similarly in codimension $10$, $97\%$, in codimension $100$, $96\%$, and in codimension $500$, $93\%$. As Figure~\ref{fig:angleexp} shows, nearly all adversarial perturbations make an angle less than $20^{\circ}$ with the normal space. Our results are averaged over 20 retrainings of the model using SGD.

Throughout this paper we've argued that high codimension is a key source of the pervasiveness of adversarial examples. Figure~\ref{fig:angleexp} shows that, when the underlying data manifold is well sampled, adversarial perturbations are well aligned with the normal space. When the codimension is high, there are many directions normal to the manifold and thus many directions in which to construct adversarial perturbations.

\begin{figure}[h!]
\begin{center}
\begin{subfigure}{0.24\textwidth}
\includegraphics[width=0.99\linewidth]{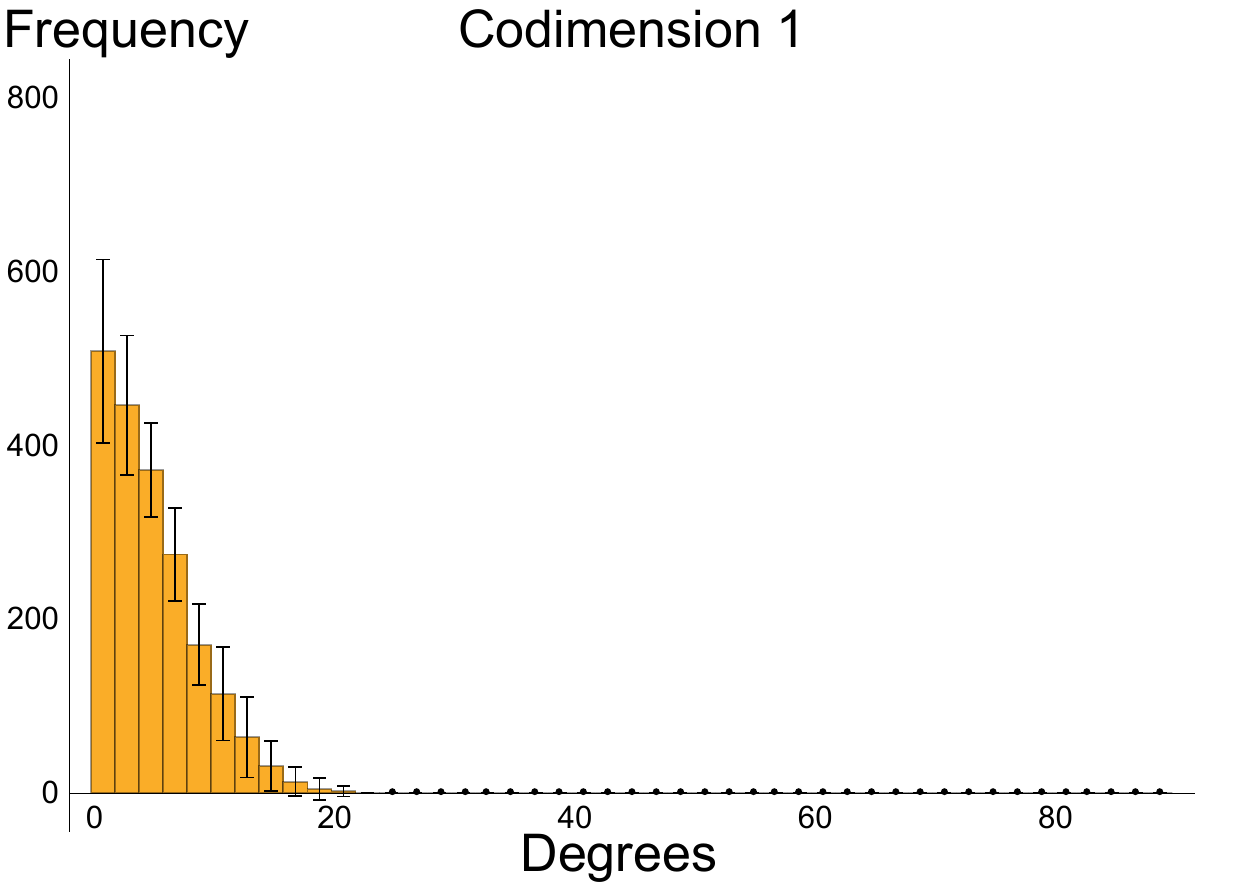}
\end{subfigure}
\begin{subfigure}{0.24\textwidth}
\includegraphics[width=0.99\linewidth]{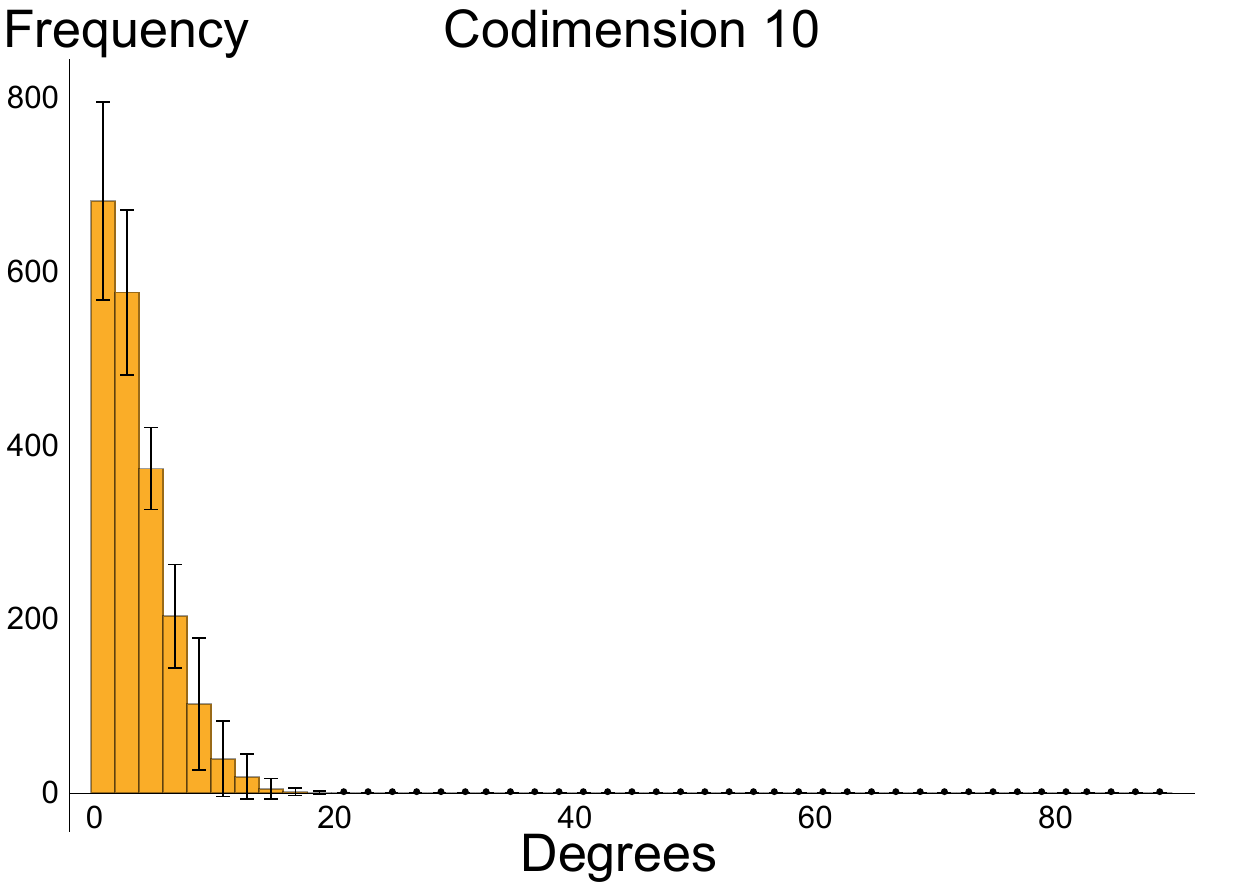}
\end{subfigure}
\begin{subfigure}{0.24\textwidth}
\includegraphics[width=0.99\linewidth]{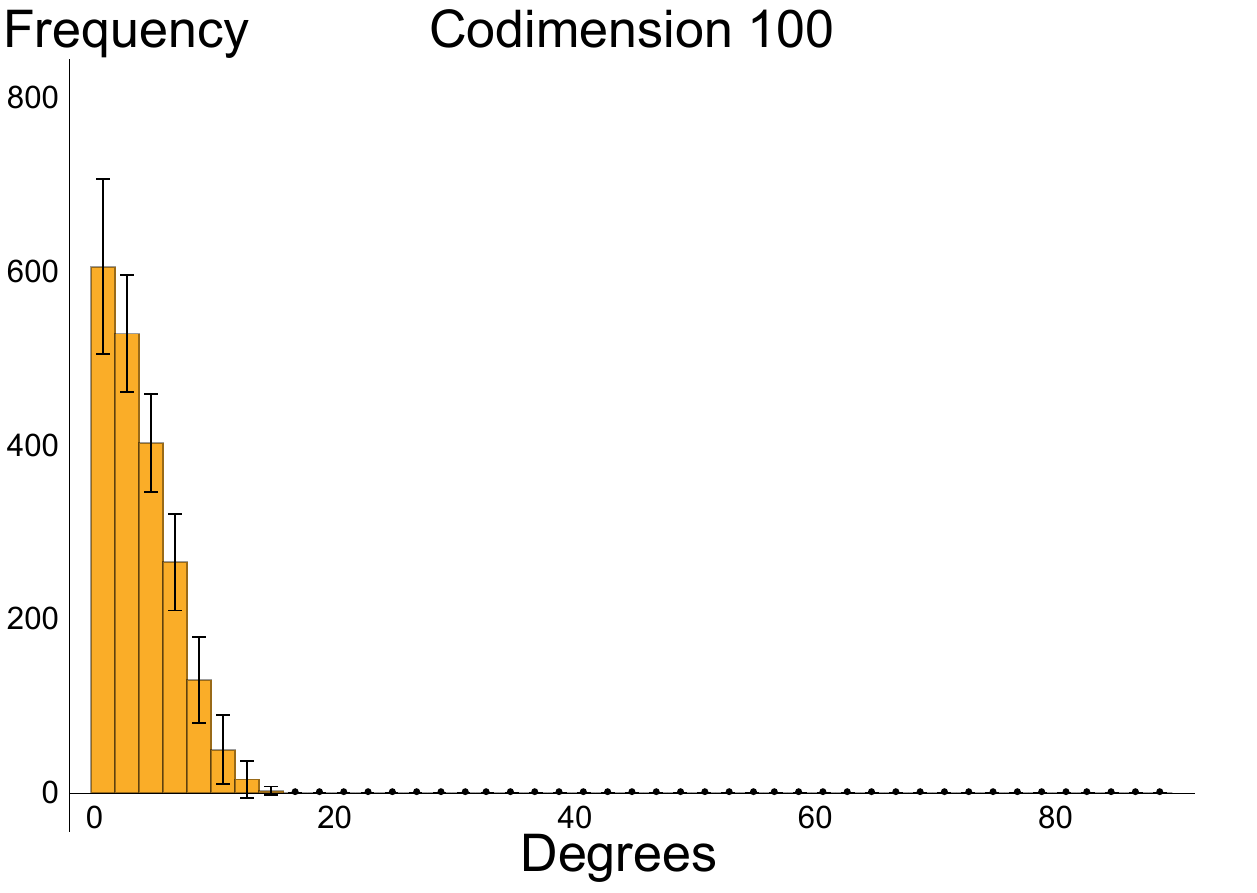}
\end{subfigure}
\begin{subfigure}{0.24\textwidth}
\includegraphics[width=0.99\linewidth]{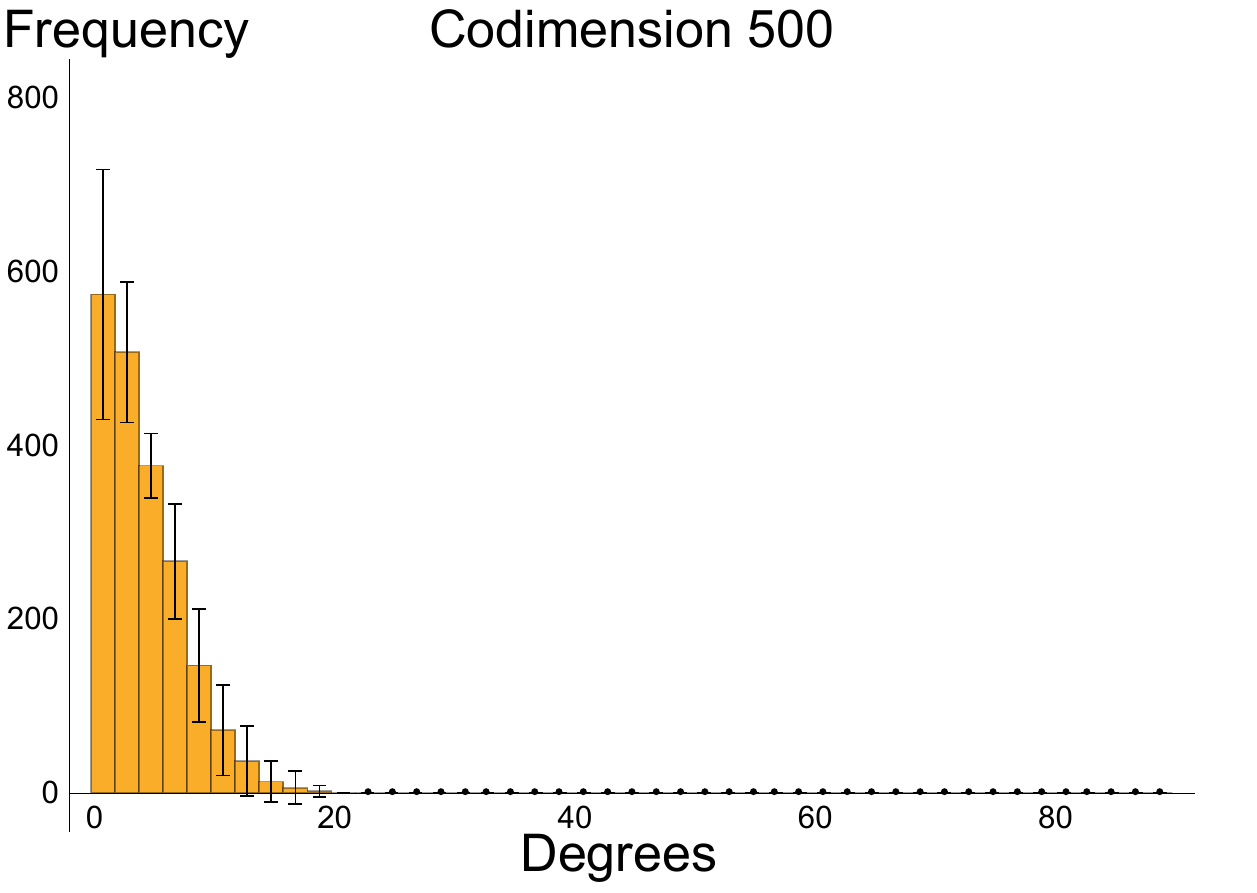}
\end{subfigure}
\caption{Histograms of the angle deviations of FGSM perturbations from the normal space for the {\sc Circles} dataset in, from right to left, codimensions 1, 10, 100, and 500. Nearly all perturbations make an angle of less than $20^{\circ}$ with the normal space.}
\label{fig:angleexp}
\end{center}
\end{figure} 

\subsection{A Gradient-Free Geometric Attack}
\label{ssec:geomattack}
Most current attacks rely on the gradient of the loss function at a test sample to find a direction towards the decision boundary. Partial resistance against such attacks can be achieved by obfuscating the gradients, but \cite{Athalye18} showed how to circumvent such defenses. \cite{Brendel18} propose a gradient-free attack for $\|\cdot\|_{2}$, that starts from a misclassified point and walks toward the original point. 

In this section we propose a gradient-free attack that only requires oracle access to a model, meaning we only query the model for a prediction. Consider a point $x \in X_{\operatorname{test}}$ and the $\|\cdot\|_{p}$-ball $B(x, r)$ centered at $x$ of radius $r$. To construct an adversarial perturbation $\eta_{x} \in B(0, r)$, giving an adversarial example $\hat{x} = x + \eta_{x}$, we project every point in $X_{\operatorname{test}}$ onto $B(x, r)$ and query the oracle for a prediction for each point. If there exists $y \in X_{\operatorname{test}}$ that is projected to a point $y' \in B(x, r)$ and that is classified differently from $x$, we take $\eta_{x} = y' - x$, otherwise $\eta_x = 0$. This incredibly simple attack reduces the accuracy of the pretrained robust model of \cite{Madry17} for $\|\cdot\|_{\infty}$ and $\epsilon = 0.3$ to $90.6\%$, less than two percent shy of the current SOTA for whitebox attacks, $88.79\%$ (\cite{Zheng18}). 

Simple datasets, such as {\sc Circles} and {\sc Planes}, allow us to diagnose issues in learning algorithms in settings where we understand how the algorithm should behave. For example \cite{Athalye18} state that the work of \cite{Madry17} does not suffer from obfuscated gradients. In Appendix~\ref{sec:madryobf} we provide evidence that \cite{Madry17} \emph{does} suffer from the obfuscated gradients problem, failing one of \cite{Athalye18}'s criteria for detecting obfuscated gradients.

\subsection{MNIST}

To explore performance on a more realistic dataset, we compared nearest neighbors with robust and natural models on MNIST. We considered two attacks: BIM under $l_\infty$ norm against the natural and robust models as well as a custom attack against nearest neighbors. Each of these attacks are generated from the MNIST test set. Architecture details can be found in Appendix~\ref{sec:impdets}. 
Figure \ref{fig:nnexpbim} (Left) shows that nearest neighbors is substantially more robust to BIM attacks than the naturally trained model. Figure \ref{fig:nnexpbim} (Center) shows that nearest neighbors is comparable to the robust model up to $\epsilon = 0.3$, which is the value for which the robust model was trained. After $\epsilon = 0.3$, nearest neighbors is substantially more robust to BIM attacks than the robust model. At $\epsilon = 0.5$, nearest neighbors maintains accuracy of $78\%$ to adversarial perturbations that cause the accuracy of the robust model to drop to $0\%$. In Appendix~\ref{ssec:fgsmexp} we provide a similar result for FGSM attacks.

Figure \ref{fig:nnexpbim} (Right) shows the performance of nearest neighbors and the robust model on adversarial examples generated for nearest neighbors. The nearest neighbor attacks are generated as follows: iteratively find the $k$ nearest neighbors and compute an attack direction by walking away from the neighbors in the true class and toward the neighbors in other classes. We find that nearest neighbors is able to be tricked by this approach, but the robust model is not. This indicates that the errors of these models are distinct and suggests that ensemble methods may be effective. 

\begin{figure}
\begin{center}
\begin{subfigure}{0.31\textwidth}
\includegraphics[width=0.99\linewidth]{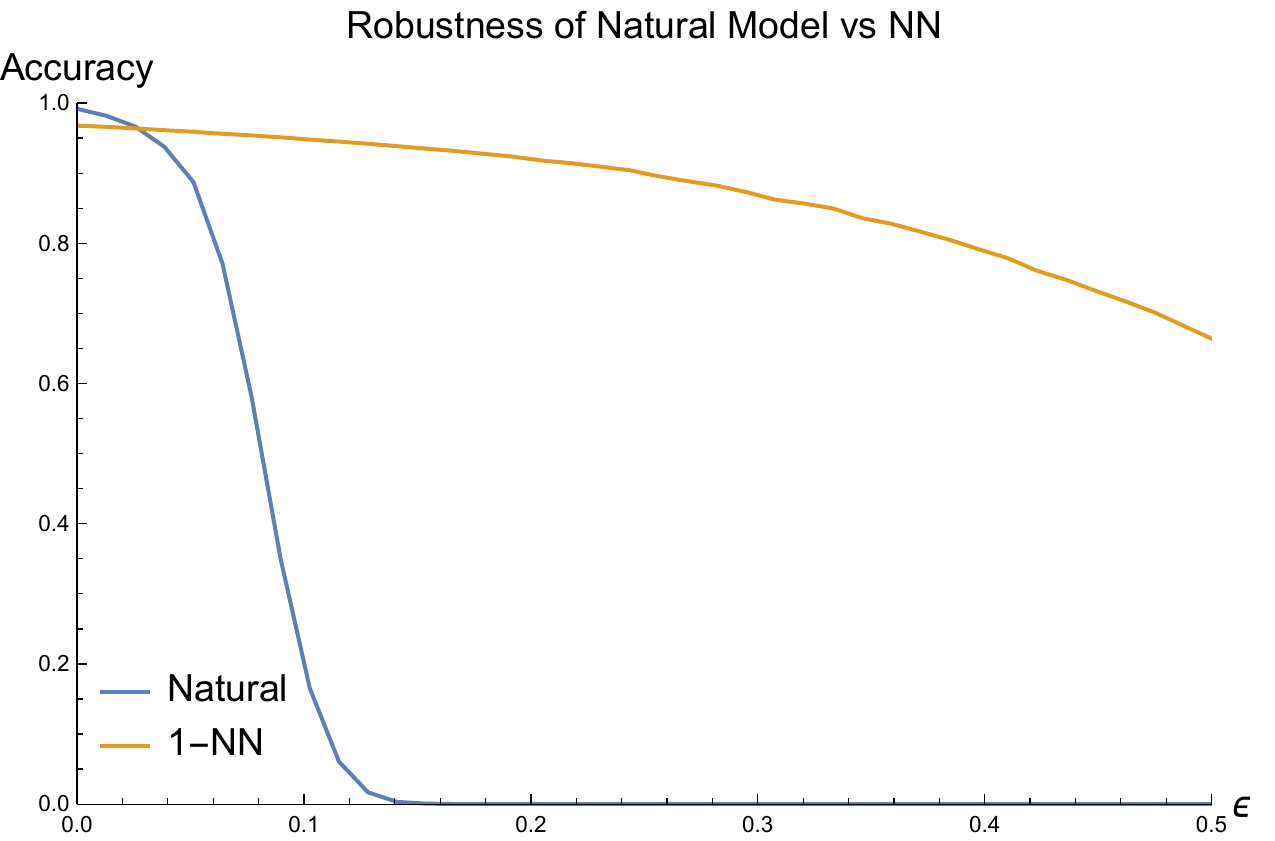}
\end{subfigure}
\begin{subfigure}{0.31\textwidth}
\includegraphics[width=0.99\linewidth]{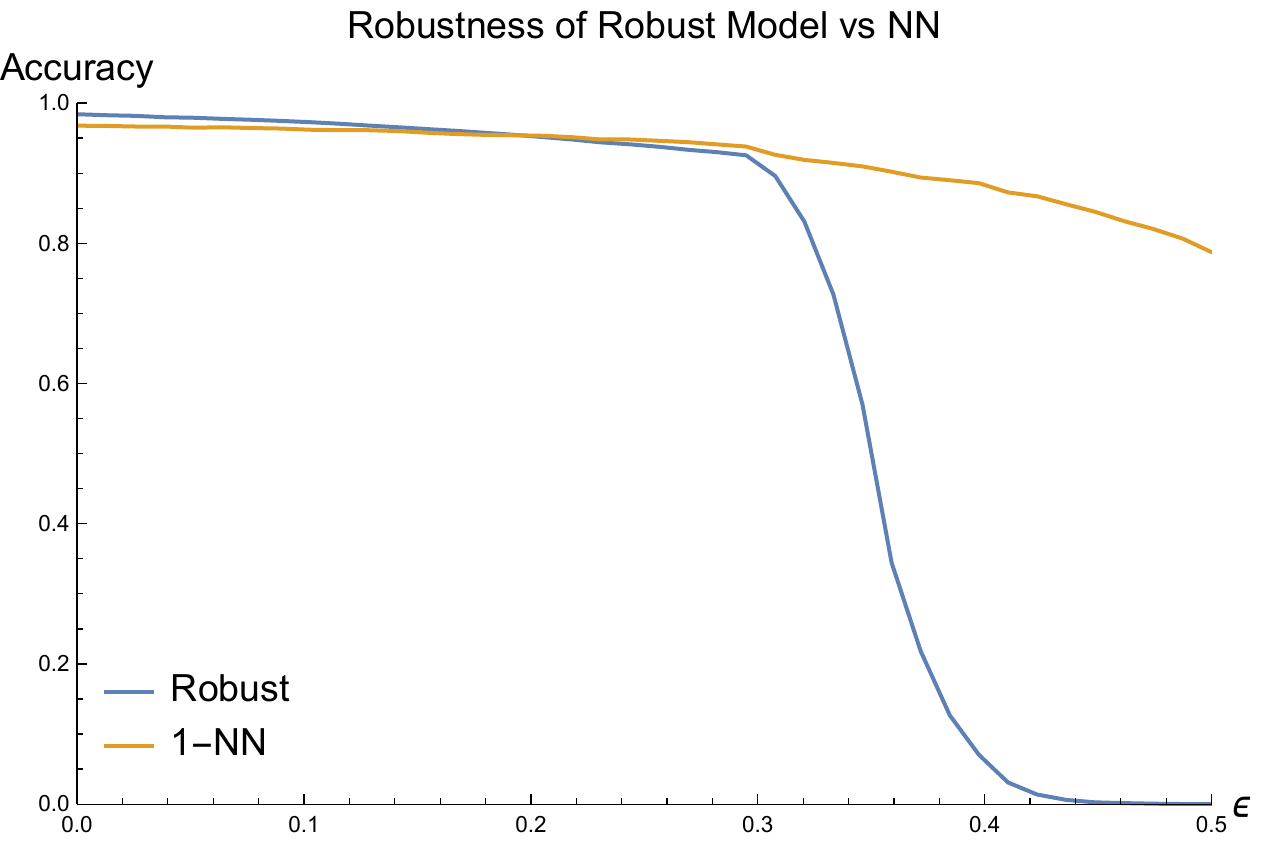}
\end{subfigure}
\begin{subfigure}{0.31\textwidth}
\includegraphics[width=0.99\linewidth]{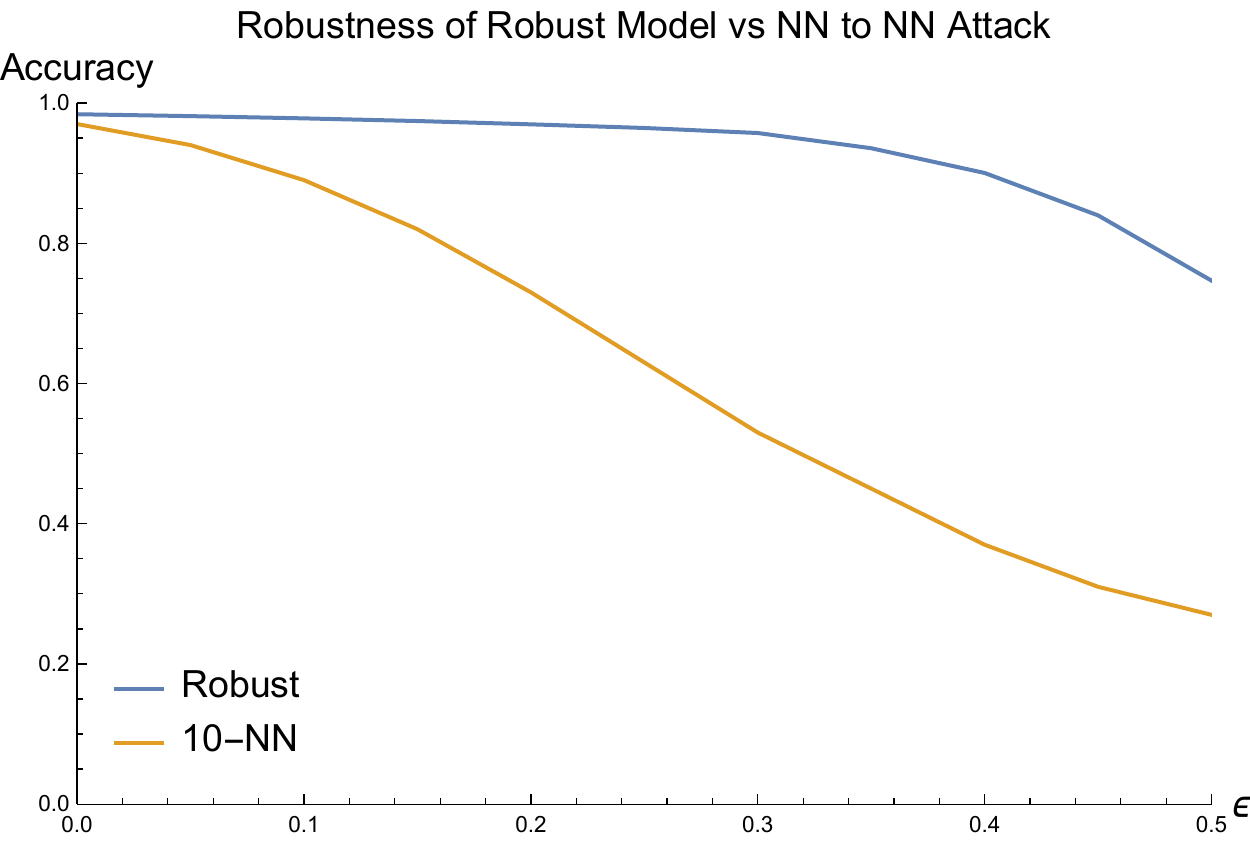}
\end{subfigure}
\caption{Robustness of nearest neighbors on MNIST. \textbf{Left}: Performance on $\l_{\infty}$ BIM attack against a naturally trained model. \textbf{Center}: The same for the adversarially trained convolutional models of \cite{Madry17}. \textbf{Right}: Performance of the robust model and nearest neighbors on examples generated by a custom attack on nearest neighbors.}
\label{fig:nnexpbim}
\end{center}
\end{figure} 

A closer investigation shows strong qualitative differences between the BIM adversarial examples and the examples generated for nearest neighbors. The top row of Figure \ref{fig:mnistexps} shows four samples from the MNIST test set. The second and third rows show adversarial examples generated from those four samples for nearest neighbors and the robust model respectively. We observe an immediate qualitative difference between rows two and three: the adversarial examples for the nearest neighbors classifier begin to look like numbers from the target class! It can reasonably be argued that the fact that the classifications of the robust model do not change is as much of an error as being fooled by a standard adversarial example. For example the rightmost image of row two in Figure \ref{fig:mnistexps} would be classified as an 8 by most people, while the robust model is confident this image is a 0 with confidence 0.91. The confidence value of the robust model should decrease significantly for this image. This provides evidence that nearest neighbors is doing a better job of the learning the \emph{human} decision boundary between numbers.

\begin{figure}[h!]
\centering
\includegraphics[width=.9\textwidth]{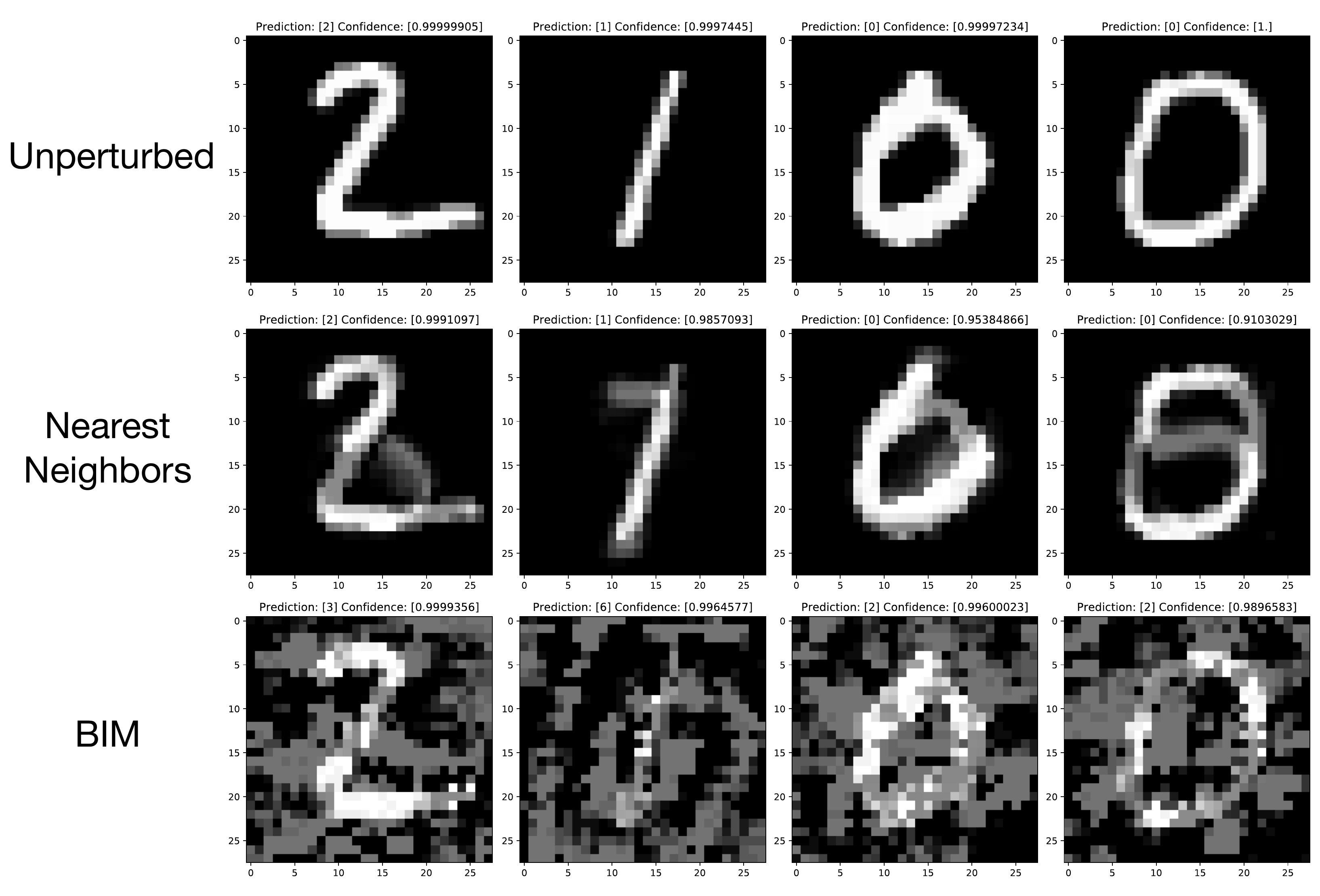}
\caption{Comparison of adversarial examples for nearest neighbor with adversarial examples for \cite{Madry17}. The top row is the original data that the examples were generated from. Each figure is labelled with the predictions from robust neural network. An immediate qualitative difference between adversarial examples for nearest neighbors and those for the robust model is apparent; the nearest neighbor examples are starting to look like numbers from a target class.} 
\label{fig:mnistexps}
\end{figure}

\section{Conclusion}
We have presented a geometric framework for proving robustness guarantees for learning algorithms. Our framework is general and can be used to describe the robustness of any classifier. We have shown that no single model can be simultaneously robust to attacks under all norms, that nearest neighbor classifiers are theoretically more sample efficient than adversarial training, and that robustness requires larger deep ReLU networks. Most importantly, we have highlighted the role of codimension in contributing to adversarial examples and verified our theoretical contributions with experimental results.

We believe that a geometric understanding of the decision boundaries learned by deep networks will lead to both new geometrically inspired attacks and defenses. In Section~\ref{ssec:geomattack} we provided a novel gradient-free geometric attack in support of this claim. Finally we believe future work into the geometric properties of decision boundaries learned by various optimization procedures will provide new techniques for black-box attacks.

\subparagraph*{Acknowledgments}
We thank Horia Mania, Ozan Sener, Sohil Shah, Jonathan Shewchuk, and Tess Smidt for providing valuable comments on an earlier draft of this work. Additionally we thank Tess Smidt for providing the compute resources for this work.

\bibliography{geomae}

\newpage
\appendix

\section{Auxiliary Results}
\label{sec:proofs}

\begin{lemma}
\label{lem:medialaxisseparate}
Let $\mathcal{M}_1, \mathcal{M}_2 \subset \R^d$ be $k$-dimensional manifolds such that $\mathcal{M} \cap \mathcal{M}_2 = \emptyset$. Let $\Lambda_p$ be their decision axis for any $p > 0$ and let $\gamma : [0,1] \rightarrow \R^d$ be any path such that $\gamma(0) \in \mathcal{M}_1$ and $\gamma(1) \in \mathcal{M}_2$. Then $\gamma \cap M \neq \emptyset$, that is $\gamma$ must cross the decision axis.
\end{lemma}
\begin{proof}
Define $f_1, f_2: [0,1] \rightarrow \R$ as $f_1(t) = d(\gamma(t), \mathcal{M}_1)$ and $f_2(t) = d(\gamma(t), \mathcal{M}_2)$. Consider the function $g(t) = f_1(t) - f_2(t)$. Since $\mathcal{M}_1 \cap \mathcal{M}_2 = \emptyset$ and $\gamma$ starts on $\mathcal{M}_1$ and terminates on $\mathcal{M}_2$ the function $g(0) < 0$ and $g(1) > 0$. Then, since $g$ is continuous, the Intermediate Value Theorem implies that there exists $t_1 \in [0,1]$ such that $g(t_1) = 0$. Thus $d(\gamma(t_1), \mathcal{M}_1) = d(\gamma(t_1), \mathcal{M}_2)$, which implies that $\gamma(t_1)$ is on the decision axis $\Lambda$. 
\end{proof}

\begin{theorem}
\label{thm:accupperbound}
Let $f$ be any classifier on $\mathcal{M} = \mathcal{M}_{1} \cup \mathcal{M}_{2}$. The maximum accuracy achievable, assuming a uniform distribution, on $\mathcal{M}^{\epsilon}$ is 
\begin{equation}
1 - \frac{1}{2}\frac{\vol(\mathcal{M}_1^{_\epsilon} \cap \mathcal{M}_2^{\epsilon})}{\vol(\mathcal{M}_1^{_\epsilon} \cup \mathcal{M}_2^{\epsilon})}.
\end{equation}
\end{theorem}
\begin{proof}
It is clearly optimal to classify points in $\vol(\mathcal{M}_1^{_\epsilon} \setminus \mathcal{M}_2^{\epsilon})$ as class $1$ and to classify points in $\vol(\mathcal{M}_2^{_\epsilon} \setminus \mathcal{M}_1^{\epsilon})$ as class $2$. Such a classifier can only be wrong when points lie in this intersection. For points in this intersection, the probability of a misclassification is $\frac{1}{2}$ for any classification that $f$ makes. Thus, the probability of misclassification is $$\frac{1}{2}\frac{\vol(\mathcal{M}_1^{_\epsilon} \cap \mathcal{M}_2^{\epsilon})}{\vol(\mathcal{M}_1^{_\epsilon} \cup \mathcal{M}_2^{\epsilon})}.$$
\end{proof}

\begin{cor}
\label{cor:medialaxisdb}
For $\epsilon < \rch_p{(\Lambda_p; \mathcal{M})}$ there exists a decision boundary that correctly classifies $\mathcal{M}^{\epsilon}$. 
\end{cor}
\begin{proof}
For $\epsilon < \rch_p{\Lambda_p}$, $\mathcal{M}^{\epsilon} \cap \Lambda_p = \emptyset$ and so $\Lambda_p$ is one such decision boundary.
\end{proof}

\section{Additional Theoretical Results}
\label{sec:addtheory}
A finite sample $X$ of $\mathcal{M}$ is said to exhibit Hausdorff noise up to $\tau$ if $X \subset \mathcal{M}^{\tau}$. That is every sample lies in a $\tau$-tubular neighborhood of $\mathcal{M}$, not necessarily on $\mathcal{M}$. We can show a similar result to Theorem~\ref{thm:sampling} for $f_{\nn}$ under moderate amounts of Hausdorff noise.

\begin{theorem}
\label{thm:nnclassifiernoise}
Let $X$ be a finite set sampled from $\mathcal{M}$ such that $X \subset \mathcal{M}^{\tau}$ for some $\tau < \rch_p{\Lambda_p}$; that is $X$ lies near $\mathcal{M}$, in a $\tau$-tubular neighborhood. If $X$ is a $\delta$-cover with $\delta \leq 2(\rch_p{\Lambda_p} - \epsilon) - \tau$, then $f_{\nn}$ correctly classifies $\mathcal{M}^{\epsilon}$.  
\end{theorem}
\begin{proof}
Let $q \in \mathcal{M}_{i}^{\epsilon}$. The distance from $q$ to any sampled in $\mathcal{M}_{j}^{\epsilon}$ for $j \neq i$ is lower bounded as $d(q, \mathcal{M}_{j}^{\tau}) \geq 2\rch_p{\Lambda_p} - \epsilon - \tau$. It is then both necessary and sufficient that there exists a sample $x \in \mathcal{M}_{i}^{\tau}$ such that $d(q, x) \leq 2\rch_p{\Lambda_p} - \epsilon -\tau$. The distance from $q$ to the nearest sample in $\mathcal{M}_{i}^{\tau}$ is upper bounded by the $\delta$-cover condition as $d(q, x) \leq \epsilon + \delta$. It suffices that 
\begin{equation*}
d(q, x) \leq \epsilon + \delta \leq 2\rch_p{\Lambda_p} - \epsilon - \tau \leq d(q, \mathcal{M}_{j}^{\tau}),
\end{equation*}
which implies that $\delta \leq 2(\rch_p{\Lambda_p} - \epsilon) - \tau$. 
\end{proof}

\begin{theorem}
\label{thm:medialaxisproximity}
Let $z \in \mathcal{D}_{f_{\nn}}$ be a point on the decision boundary of $f_{\nn}$ for a $\delta$-cover $X$ with $\delta < 1$. Let $\sigma \subset \mathcal{D}_{f_{\nn}}$ be a linear facet of $\mathcal{D}_{f_{\nn}}$ and note that $\sigma$ is a Voronoi facet, let $\sigma^{*} = pq$ be the dual Delaunay edge of $\sigma$ such that $p \in \mathcal{M}_{1}$ and $q\in \mathcal{M}_{2}$. Define $d(z, \mathcal{M}_1) = \omega_1 \rch_2{\Lambda_2}$ and $d(z, \mathcal{M}_{2}) = \omega_2 \rch_2{\Lambda_2}$, with $\omega_1 < \omega_2 < 1$. Then there exists a decision axis point $m \in \Lambda$ such that $d(z, m) \leq \frac{\delta^2 + (\omega_2^2 - \omega_1^2) + 2\delta \omega_2}{1 + (\omega^2 - \omega_1^2)} \omega_2 \rch_2{\Lambda_2}$.
\end{theorem}
\begin{proof}
If $z \in \Lambda_2$ then the result holds, so suppose that $z \not \in \Lambda_2$.

The decision boundary $\mathcal{D}_{f_{\nn}}$ is the union of a subset of $(d-1)$-dimensional Voronoi cells (along with their lower dimensional faces) of the Voronoi diagram $\Vor{X}$ of $X$ with the following property. For every Voronoi $(d-1)$-cell $\sigma \in \mathcal{D}_{f_{\nn}}$, its dual Delaunay edge $\sigma^{*} = pq$ has endpoints $p, q \in X$ such that $p \in \mathcal{M}_1$ and $q \in \mathcal{M}_{2}$. That is, $p$ and $q$ have different class labels. In particular $pq$ crosses $\Lambda_2$. For every point $z \in \sigma$, $d(z, p) = d(z, q) \leq \min_{i} d(z, X_i)$; that is, $p, q$ minimize the distance from $z$ to any sample point in $X$. In the interior of $\sigma$ this inequality is strict, while on the boundary of $\sigma$ it may be realized by more points than just $p$ and $q$. (See Appendix~\ref{sec:vordel} for a brief review of Voronoi diagrams and Delaunay triangulations.)

Let $\sigma \in \mathcal{D}_{f_{\nn}}$ be a Voronoi $(d-1)$-cell that contains $z$ and let $\sigma^{*} = pq$ be $\sigma$'s dual Delaunay edge. Imagine growing a ball $B$ centered at $z$ by increasing the radius $r$ starting from $0$. Due to the properties of Voronoi cells outlined above, the fact that $z \not \in \Lambda_2$, and the fact that $X \subset \mathcal{M}$, the following three events occur in order as we increase $r$. First $B$ intersects the manifold to which $z$ is closest, without loss of generality $\mathcal{M}_1$. Second $B$ intersects $\mathcal{M}_{2}$. Notice that at this point $B$ has not intersected any sample points in $X$, since $p$ and $q$ are on $\mathcal{M}_{1}$ and $\mathcal{M}_{2}$ respectively and are the closest samples to $z$. Third $B$ intersects $p$ and $q$, when $r = d(z, p)$. Let $r_1, r_2, r_3$ denote the value of the radius at these three event points respectively. Similarly let $B_1, B_2, B_3$ denote the balls centered at $z$ with radii $r_1, r_2, r_3$ respectively. Let $z_1 \in B_1 \cap \mathcal{M}_1$ and let $z_2 \in \mathcal{M}_2 \cap B_{2}$. Since $\mathcal{M}_1$ is the closer of the two manifolds to $z$, the line segment $zz_2$ must intersect $\Lambda_2$. Let $\gamma: [0,1] \rightarrow \R^d$ parameterize the line segment $zz_2$, where $\gamma(0) = z$, $\gamma(1) = z_2$, and $\|z - \gamma(t) \|_{2} = r_2 t$. We will show that there exists a decision axis point $m \in \gamma$ that is close to $z$. 

The ball $B_2$ is tangent to $\mathcal{M}_{2}$ at $z_2$ but contains some portion of $\mathcal{M}_1$. Our approach will be to move the center of $B_2$ along $\gamma$ from $z$ to $z_2$ while maintaining tangency at $z_2$. That is we consider the balls $B_t = B(\gamma(t), \|\gamma(t) - z_2\|_{2})$ as $t$ increase from $0$ to $1$. For some $t^*$, $B_{t^*} \cap \mathcal{M}_{1} = \emptyset$ which means that we have crossed the decision axis. We will prove that $t^{*}$ must be small which implies that $\|z - m\|_{2} \leq \|z - \gamma(t^{*})\|_{2} \leq r_2 t^{*}$ is small.

We begin by considering the triangle $\triangle z_1zz_2$. Using the law of cosines we derive an expression for the angle $\angle{z_1zz_2}$ as 

\begin{align*}
\|z_1 - z_2\|_{2}^2 &= r_1^2 + r_2^2 - 2r_1r_2 \cos{\angle{z_1zz_2}}\\ 
\cos{\angle{z_1zz_2}} &= \frac{r_1^2 + r_2^2 - \|z_1 - z_2\|_{2}^2}{2r_1 r_2}.
\end{align*}

As $t$ increases the event $B_{t} \cap \mathcal{M}_{1} = \emptyset$ occurs when the distances from $\gamma(t)$ to any point $x \in B_2 \cap \mathcal{M}_{1}$ is greater than $r_2(1-t)$. Due to the $\delta$-cover condition at $z_1$ and the fact that $B_2 \subset B_3$ where $B_3$ is the event where a ball centered at $z$ intersects a sample point, every such $x$ must lie in a ball $B(z_1, g)$ for $g \leq \delta$. Thus the event $B_t \cap \mathcal{M}_1 = \emptyset$ occurs for the minimum $t$ such that 

\begin{align*}
\|z_1 - \gamma(t)\|_{2} - g &\geq r_2(1-t)\\
 \|z_1 - \gamma(t)\|_{2}^2 &\geq r_2^2(1-t)^2 + 2g r_2(1-t) + g^2.
\end{align*}

First we derive an expression for $\|z_1 - \gamma(t)\|_{2}$ again using the law of cosines and substituting the expression for $\angle{z_1zz_2}$.

\begin{align*}
\|z_1 - \gamma(t)\|_{2}^2 &= r_1^2 + r_2^2 t^2 - 2r_1r_2 t \cos{\angle{z_1zz_2}}\\
                          &= r_1^2 + r_2^2 t^2 - t (r_1^2 + r_2^2 - \|z_1 - z_2\|_{2}^2)\\
                          &= (1 - t)r_1^2 + (t-1)t r_2^2 + t \|z_1 - z_2\|_{2}^2.
\end{align*}

So then $\|z_1 - \gamma(t)\|_{2}^2 \geq r_2^2(1-t)^2 + 2g r_2(1-t) + g^2$ holds if and only if

\begin{align*}
(1 - t)r_2^2 + (t-1)t r_2^2 + t \|z_1 - z_2\|_{2}^2 &\geq r_2^2(1-t)^2 + 2g r_2(1-t) + g^2\\
t &\leq \frac{g^2 - r_1^2 + 2gr_2 + r_2^2}{\|z_1 - z_2\|^2 - r_1^2 + 2gr_2 + r_2^2}\\
  &\leq \frac{g^2 - r_1^2 + 2gr_2 + r_2^2}{\|z_1 - z_2\|^2 - r_1^2 + r_2^2}\\
  &\leq \frac{\delta^2 + (\omega_2^2 - \omega_1^2) + 2\delta \omega_2}{1 + (\omega^2 - \omega_1^2)}
\end{align*}
\end{proof}

\section{Additional Experiments}
We present additional experiments to support our theoretical predictions. We reproduce the results of Section~\ref{sec:exp} using different optimization algorithms (Section~\ref{ssec:sgdexps}) and attack methods (Section~\ref{ssec:fgsmexps}). These additional experiments are consistent with our conclusions in Section \ref{sec:exp}. 

\subsection{Reproducing Results using SGD}
\label{ssec:sgdexps}
In Section~\ref{ssec:codim} we showed that increasing the codimension reduces the robustness of the decision boundaries learned by Adam on {\sc Circles}. In Figure~\ref{fig:codimexpsgd} we reproduce this result using SGD. Again we see that as we increase the codimension the robustness decreases. SGD presents with much less variances than Adam, which we attribute to implicit regularization that has been observed for SGD (\cite{Soudry18})

\begin{figure}[h!]
\begin{center}
\begin{subfigure}{0.49\textwidth}
\includegraphics[width=0.98\linewidth]{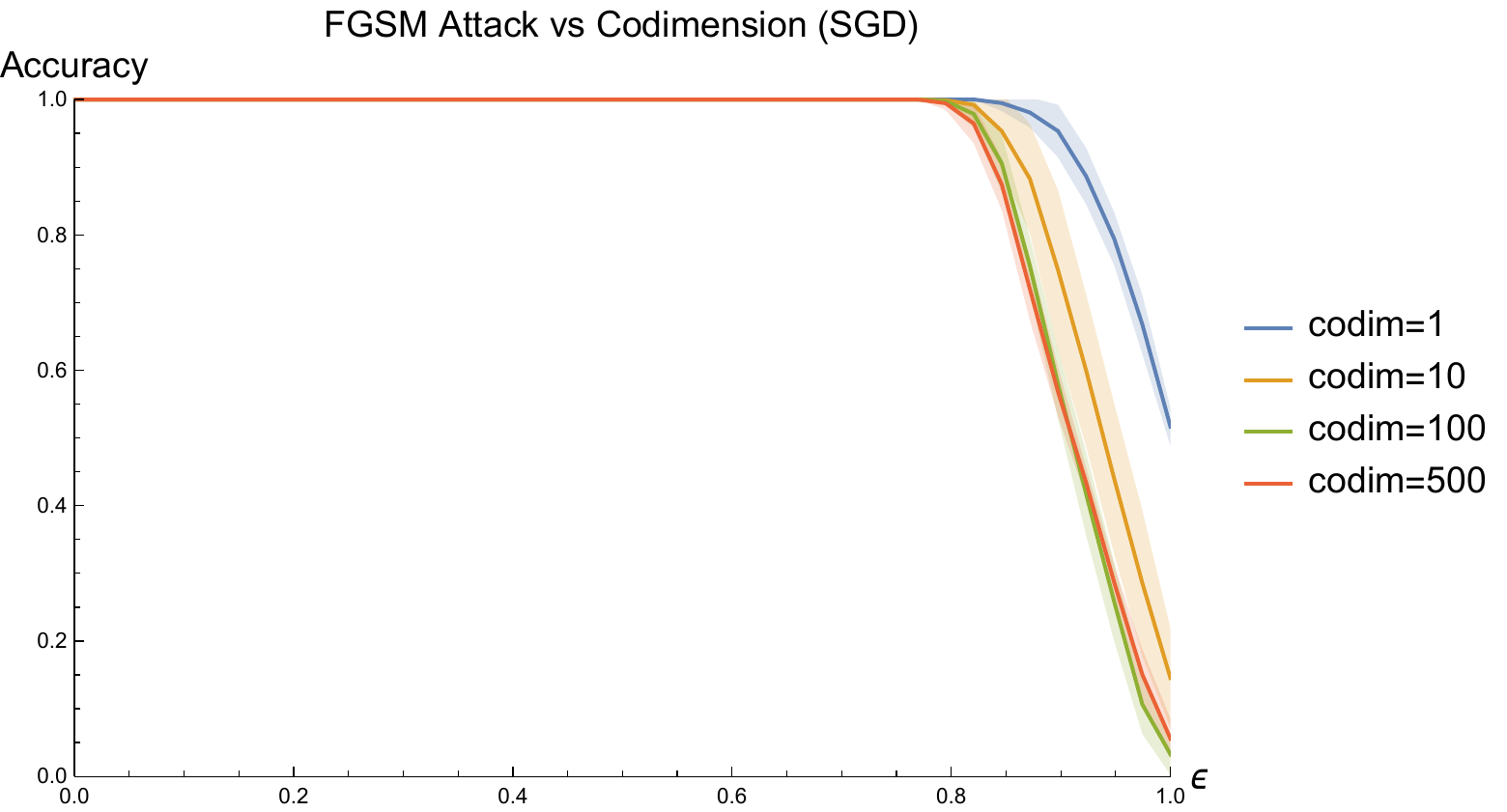}
\end{subfigure}
\begin{subfigure}{0.49\textwidth}
\includegraphics[width=0.98\linewidth]{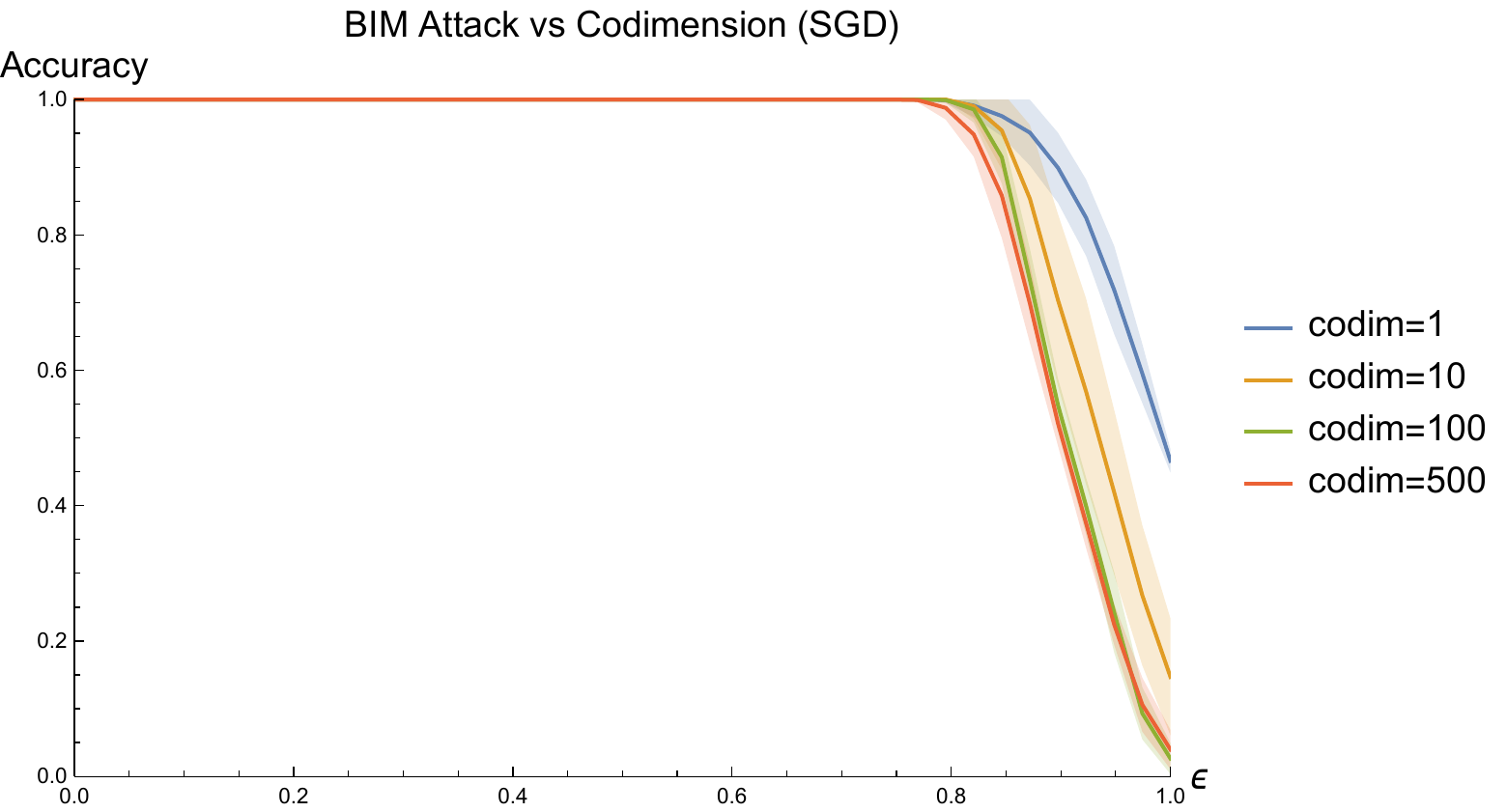}
\end{subfigure}
\caption{As in the case of training with Adam, we see a steady decrease in robustness on the {\sc Circles} dataset as the codimension increases when training with SGD.}
\label{fig:codimexpsgd}
\end{center}
\end{figure} 

Next we consider the adversarial training procedure of \cite{Madry17} using SGD instead of Adam. We note that the authors of \cite{Madry17} use Adam in their own experiments. Figure~\ref{fig:advtrainsgd} shows that the result is consist with the result in Section~\ref{ssec:codim}. Again SGD presents with lower variance. 

\begin{figure}[h!]
\begin{center}
\includegraphics[width=0.49\linewidth]{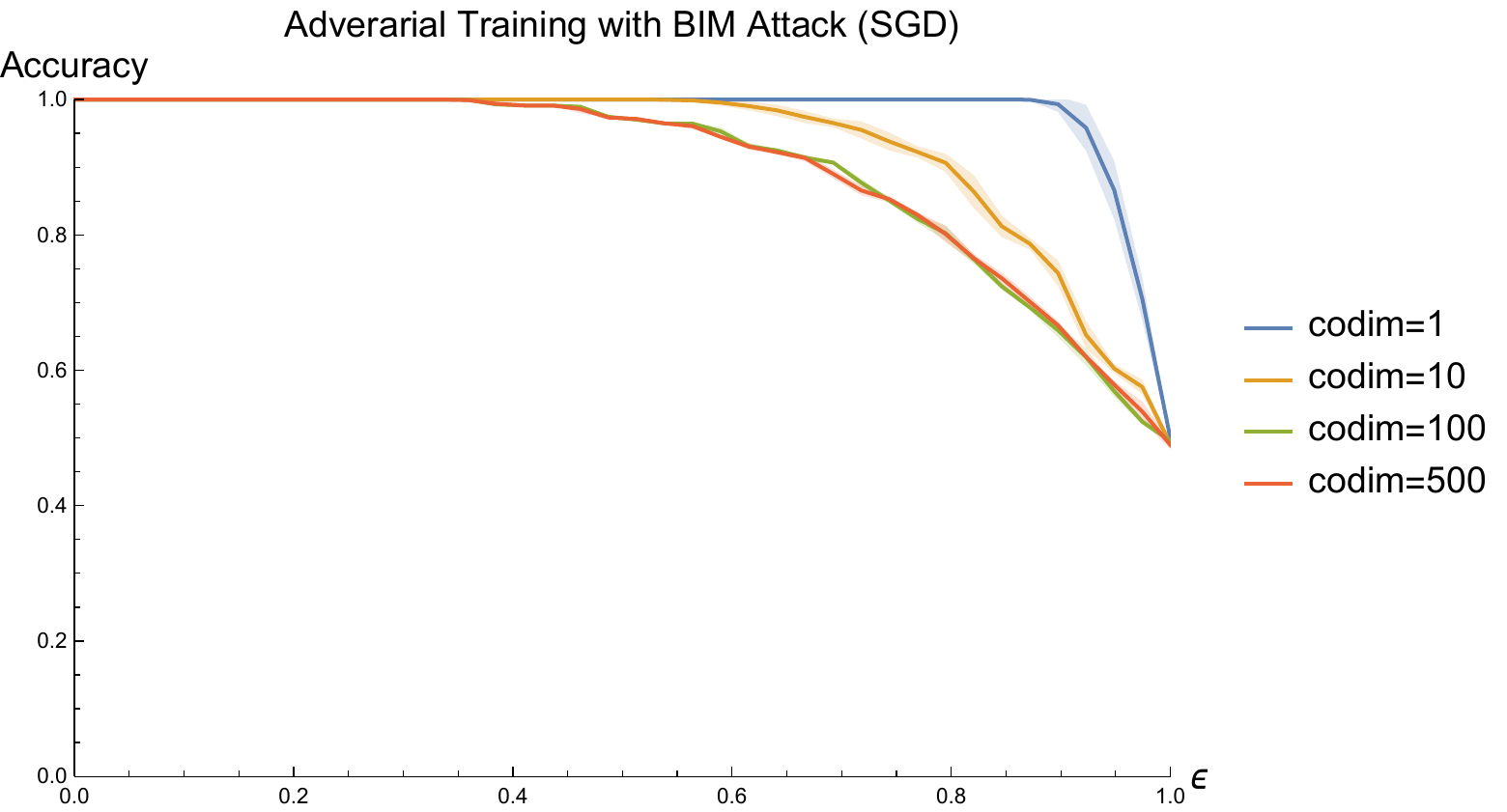}
\caption{Adverarial training with a PGD adversary, as in Figure~\ref{fig:codimexp}, using SGD. Similarly we see a drop in robustness as the codimension increases.}
\label{fig:advtrainsgd}
\end{center}
\end{figure} 

\subsection{Reproducing Results using FGSM}
\label{ssec:fgsmexps}
In Section~\ref{ssec:codim} we evaluated the robustness of nearest neighbors against BIM attacks under the $\|\cdot\|_{\infty}$ on MNIST. In Figure~\ref{fig:nnexpfgsm} we evaluate the robustness of nearest neighbors against FGSM attacks under the $\|\cdot\|_{\infty}$ on MNIST. We use the naturally pretrained (natural) and adversarially pretrained (robust) convolutional models provided by \cite{Madry17}\footnote{\url{https://github.com/MadryLab/mnist_challenge}}. Figure \ref{fig:nnexpfgsm} (Left) shows that nearest neighbors is substantially more robust to FGSM attacks than the naturally trained model. Figure \ref{fig:nnexpfgsm} (Right) shows that nearest neighbors is comparable to the robust model up to $\epsilon = 0.3$, which is the value for which the robust model was trained. After $\epsilon = 0.3$, nearest neighbors is substantially more robust to FGSM attacks than the robust model. At $\epsilon = 0.5$, nearest neighbors maintains accuracy of $78\%$ to adversarial perturbations that cause the accuracy of the robust model to drop to $39\%$. 

\label{ssec:fgsmexp}
\begin{figure}[h!]
\begin{center}
\begin{subfigure}{0.49\textwidth}
\includegraphics[width=0.99\linewidth]{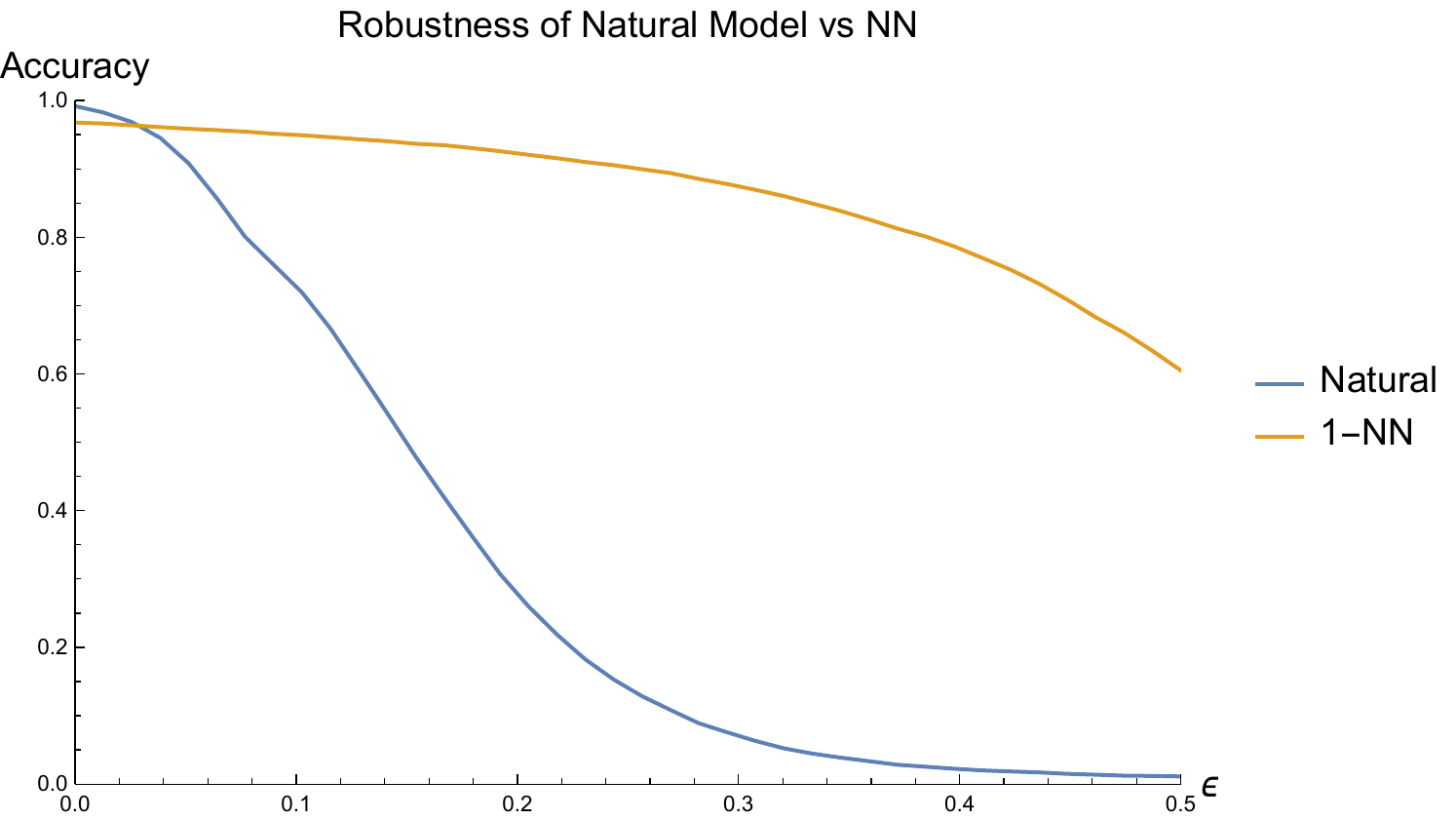}
\end{subfigure}
\begin{subfigure}{0.49\textwidth}
\includegraphics[width=0.99\linewidth]{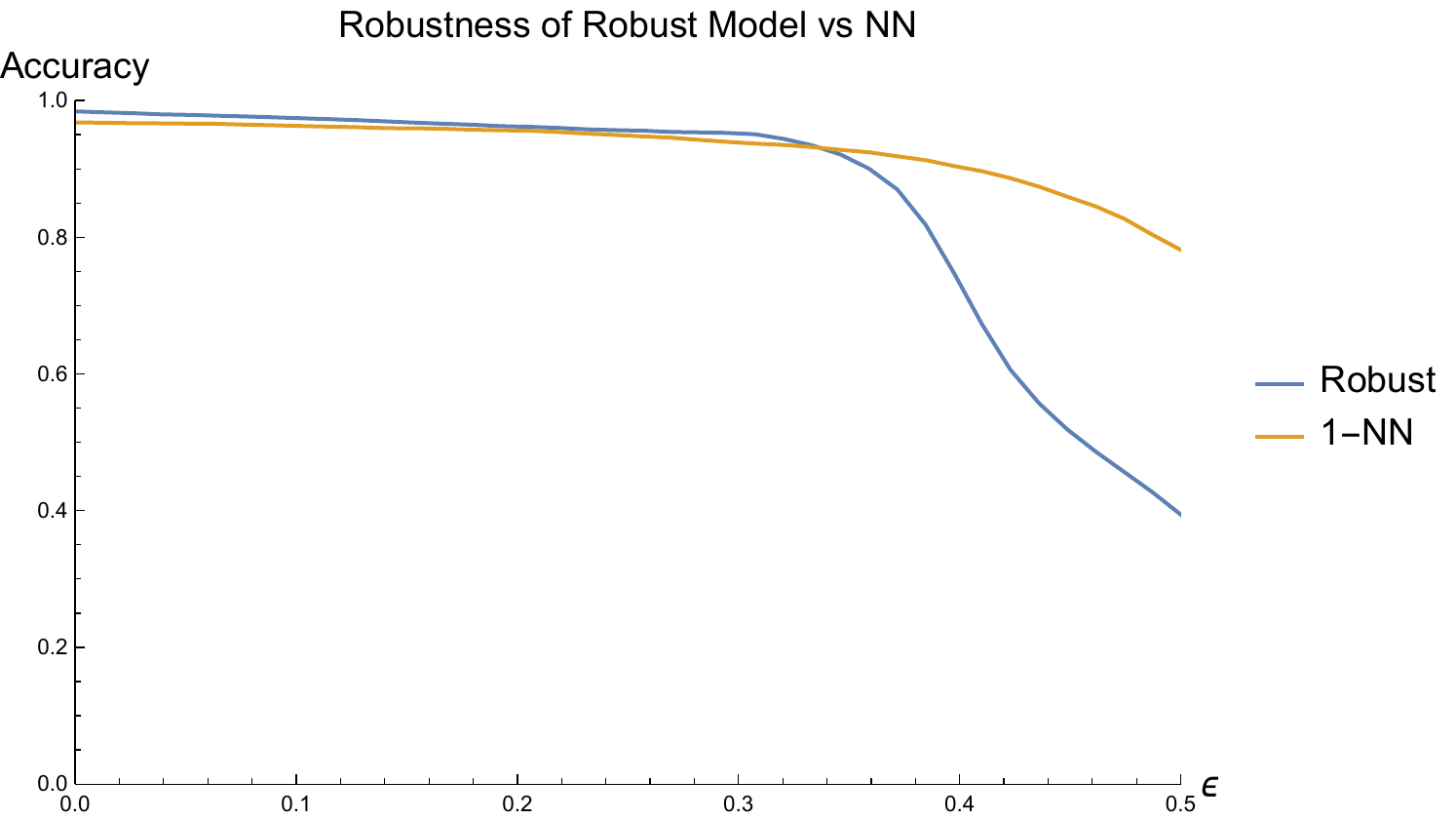}
\end{subfigure}
\caption{Robustness of nearest neighbors against the naturally trained (left) and adversarially trained (right) convolutional models of \cite{Madry17} against FGSM attacks under the $\|\cdot\|_{\infty}$ norm on MNIST.}
\label{fig:nnexpfgsm}
\end{center}
\end{figure} 

\section{The Madry Defense Suffers from Obfuscated Gradients}
\label{sec:madryobf}
\cite{Athalye18} identified the problem of ``obfuscated gradients'', a type of a gradient masking (\cite{Papernot17}) that many proposed defenses employed to defend against adversarial examples. They identified three different types of obfuscated gradients: shattered gradients, stochastic gradeints, and exploding/vanishing gradients. They examined nine recently proposed defenses, concluded that seven suffered from at least one type of obfuscated gradient, and showed how to circumvent each type of obfuscated gradient and thus each defense that employed obfuscated gradients.

Regarding the work of \cite{Madry17}, \cite{Athalye18} stated ``We believe this approach does not cause obfuscated gradients''. They note that ``our experiments with optimization based attacks do succeed with some probability''. In this section we provide evidence that the defense of \cite{Madry17} \emph{does} suffer from obfuscated gradients, specifically shattered gradients. Shattered gradients occur when a defence causes the gradient field to be ``nonexistent or incorrect'' (\cite{Athalye18}). Specifically we provide evidence that the defense of \cite{Madry17} works by shattering the gradient field of the loss function around the data manifolds. 

In Figure \ref{fig:gradfield} (Left) we show the normalized gradient field of the loss function for a network trained on a 2-dimensional version of our {\sc Planes} dataset using the adversarial training procedure of \cite{Madry17} with a PGD adversary. While the gradients have meaningful directions, Figure~\ref{fig:gradfield} (Left) shows that magnitude of the gradient field is nearly $0$ everywhere around the data manifolds, which are at $y = 0$ and $y = 2$. The only notable gradients are near the decision axis which is at $y = 1$.

\begin{figure}[h!]
\begin{center}
\begin{subfigure}{0.41\textwidth}
\includegraphics[width=0.98\linewidth]{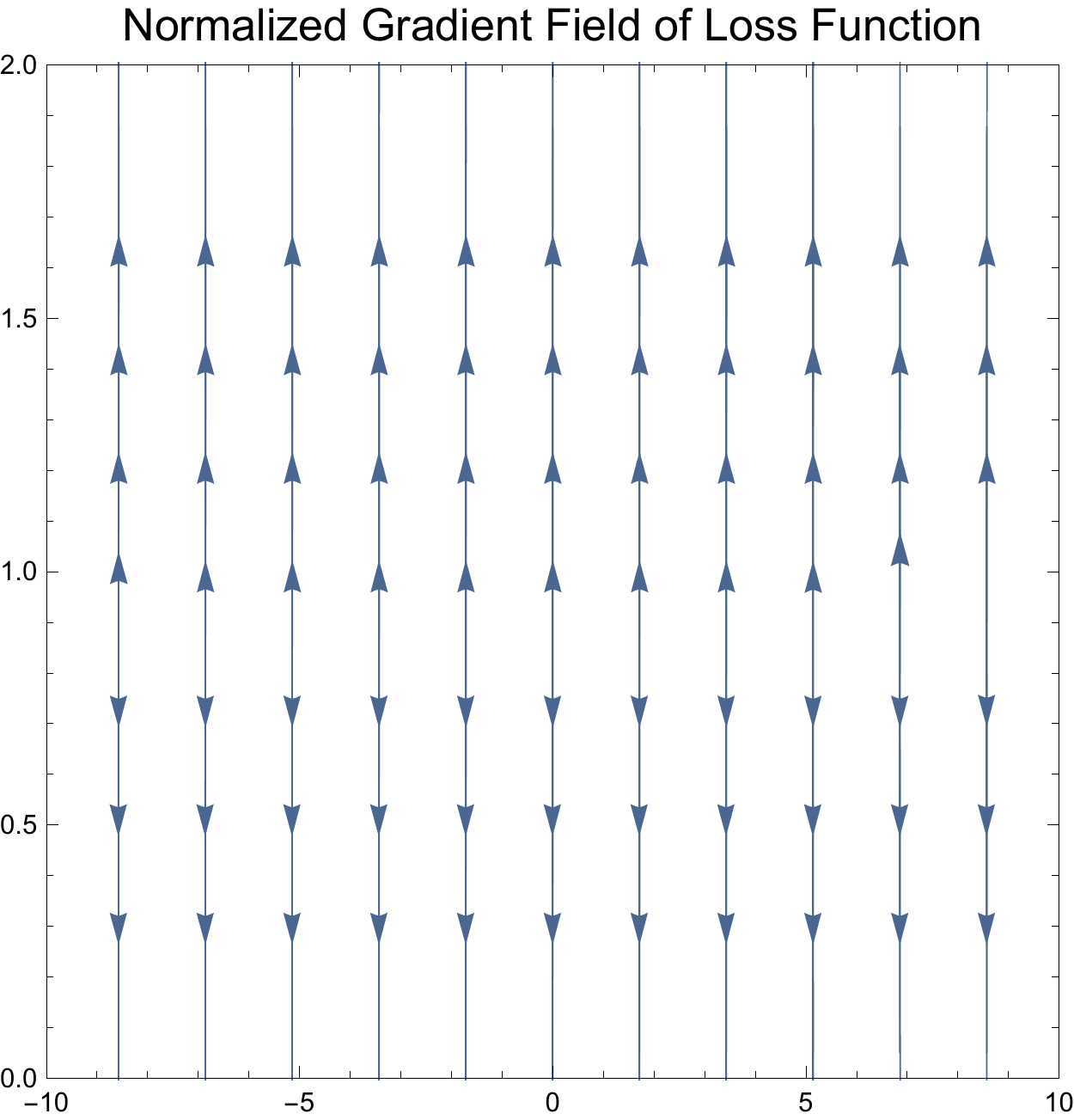}
\end{subfigure}
\begin{subfigure}{0.49\textwidth}
\includegraphics[width=0.98\linewidth]{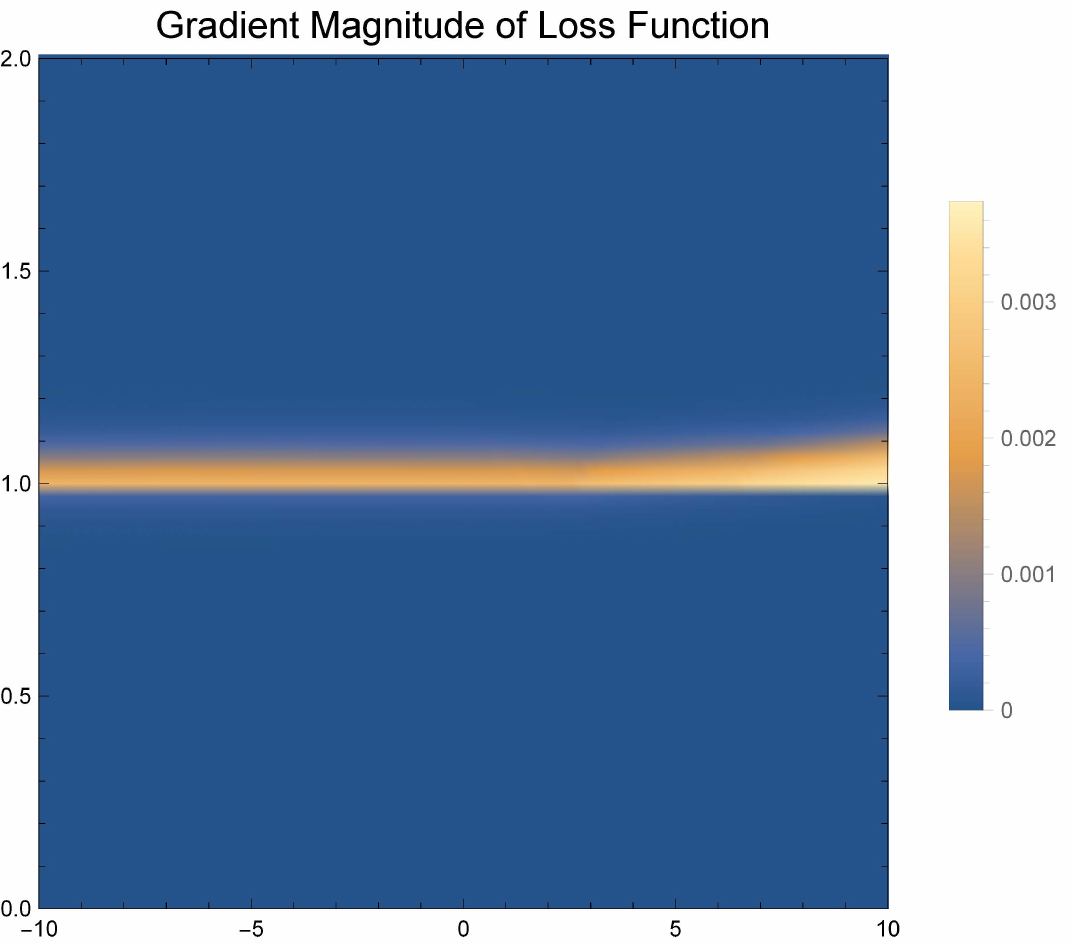}
\end{subfigure}
\caption{(Left) The normalized gradient field of the loss for an adversarially trained network. (Right) The magnitudes of the gradient. Notice that the gradients are largely $0$ except at the decision axis $y=1$.}
\label{fig:gradfield}
\end{center}
\end{figure} 

One criteria that \cite{Athalye18} propose for identifying obfuscated gradients is whether one-step attacks perform better than iterative attacks. The reason this criteria is useful for identifying obfuscated gradients is because one-step attacks like FGSM first normalize the gradient, ignoring its magnitude, then take as large of a step as allowed in the direction of the normalized gradient. So long as the gradient \emph{on the manifold} points towards the decision boundary, FGSM will be effective at finding an adversarial example. 

In Figure \ref{fig:advexamples} we show the adversarial examples generated using PGD (left), FGSM (center), and BIM (right) for $\epsilon = 1$ starting at the test set for the {\sc Planes} dataset. FGSM produces adversarial examples at the decision axis $y = 1$, exactly where we would expect. Notice that all of the adversarial perturbation is normal to the data manifold, suggesting that the gradient on the manifold points towards the decision boundary. However the adversarial examples produced by PGD lie closer to the manifold from which the example was generated. 

PGD splits the total perturbation between both the normal and the tangent spaces of the data manifold, as shown by the arrows in Figure~\ref{fig:advexamples}. This suggests that, when trained adversarially, the network learned a gradient field that has small but correct gradients on the data manifold, but gradients that curve in the tangent directions immediately \emph{off the manifold}. 

Lastly notice that BIM, another iterative method, also produces adversarial examples that are near the decision axis. \cite{Athalye18} cite success with iterative based optimization procedures as evidence against obfuscated gradients. However BIM also ignores the magnitude of the gradient, as it simply applies FGSM iteratively. The network has learned a gradient field that is overfit to the particulars of the PGD attack. BIM successfully navigates this gradient field, while PGD does not. While the network is robust to PGD attacks at test time, it is less robust to FGSM and BIM attacks.  

\begin{figure}[h!]
\begin{center}
\begin{subfigure}{0.31\textwidth}
\includegraphics[width=\linewidth]{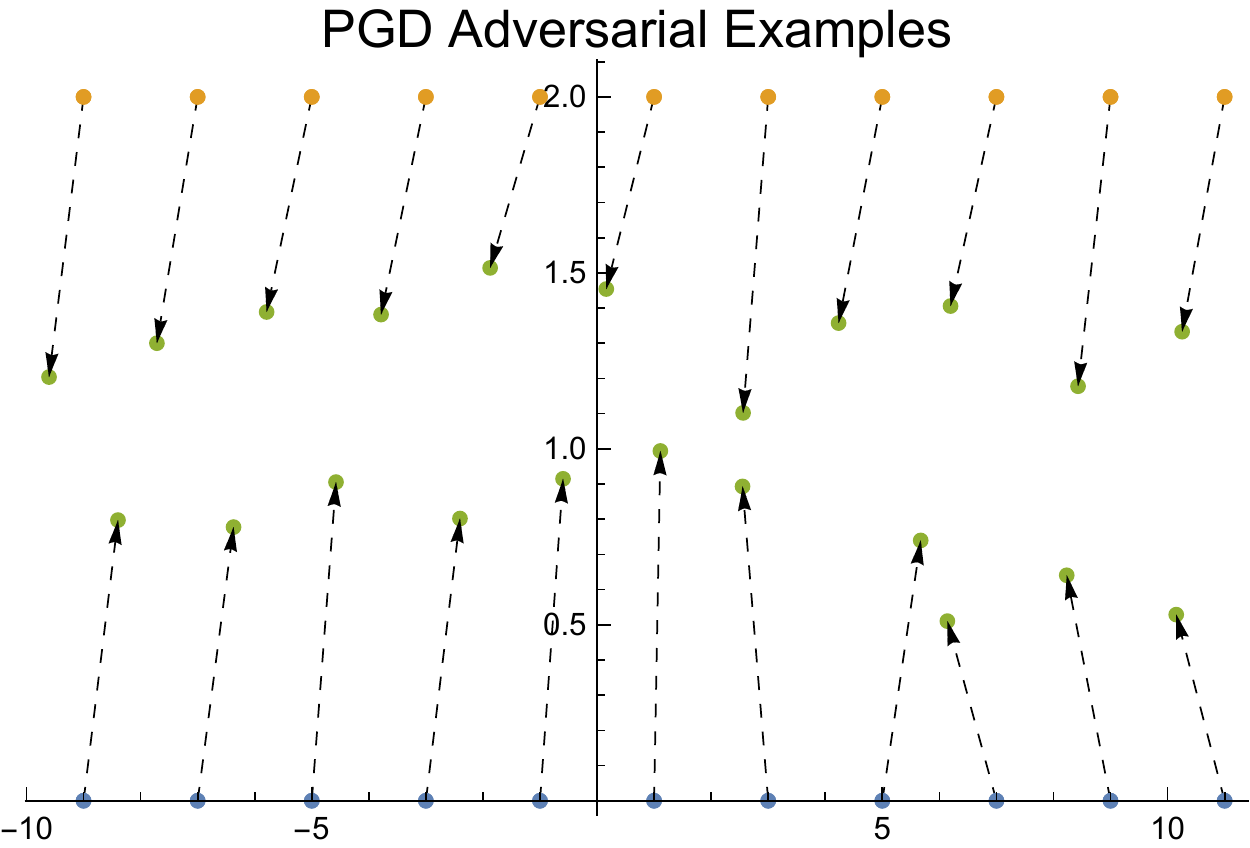}
\end{subfigure}
\begin{subfigure}{0.31\textwidth}
\includegraphics[width=\linewidth]{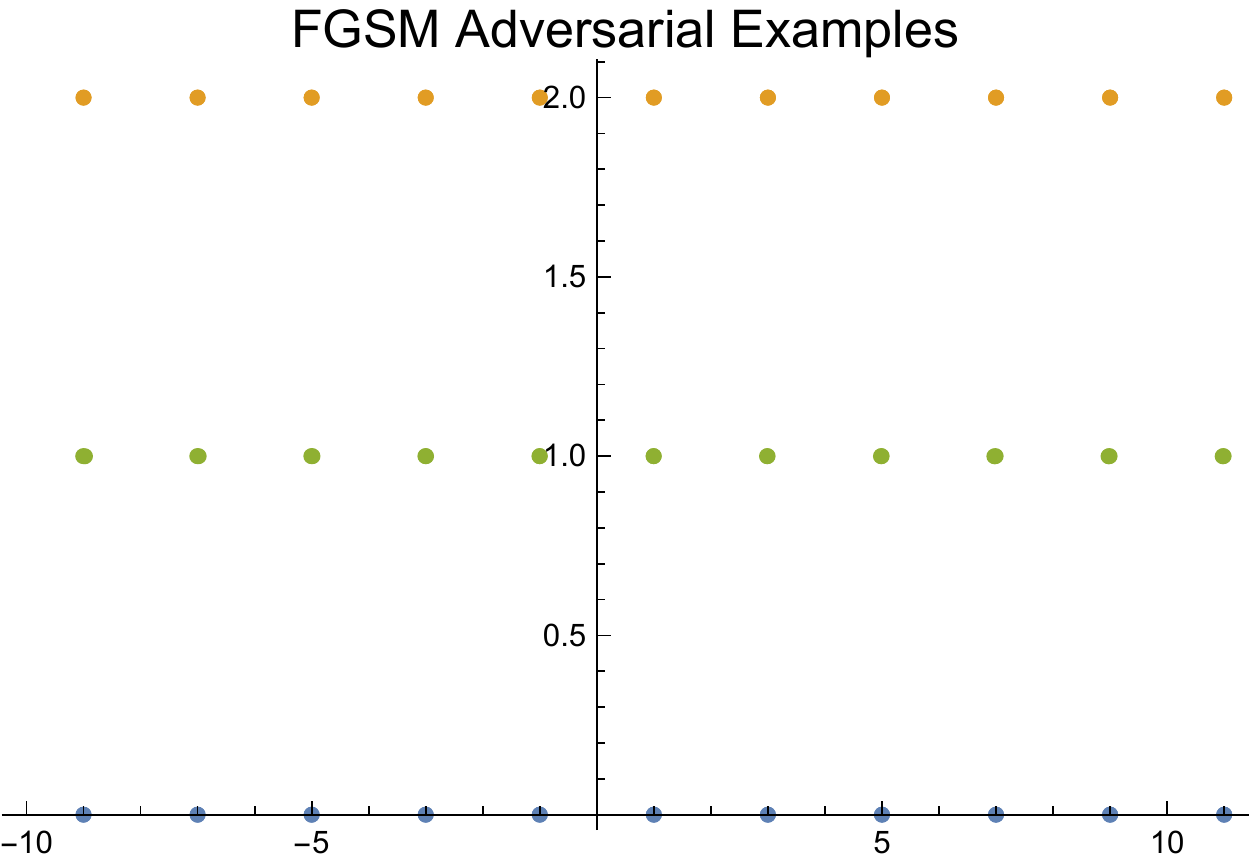}
\end{subfigure}
\begin{subfigure}{0.31\textwidth}
\includegraphics[width=\linewidth]{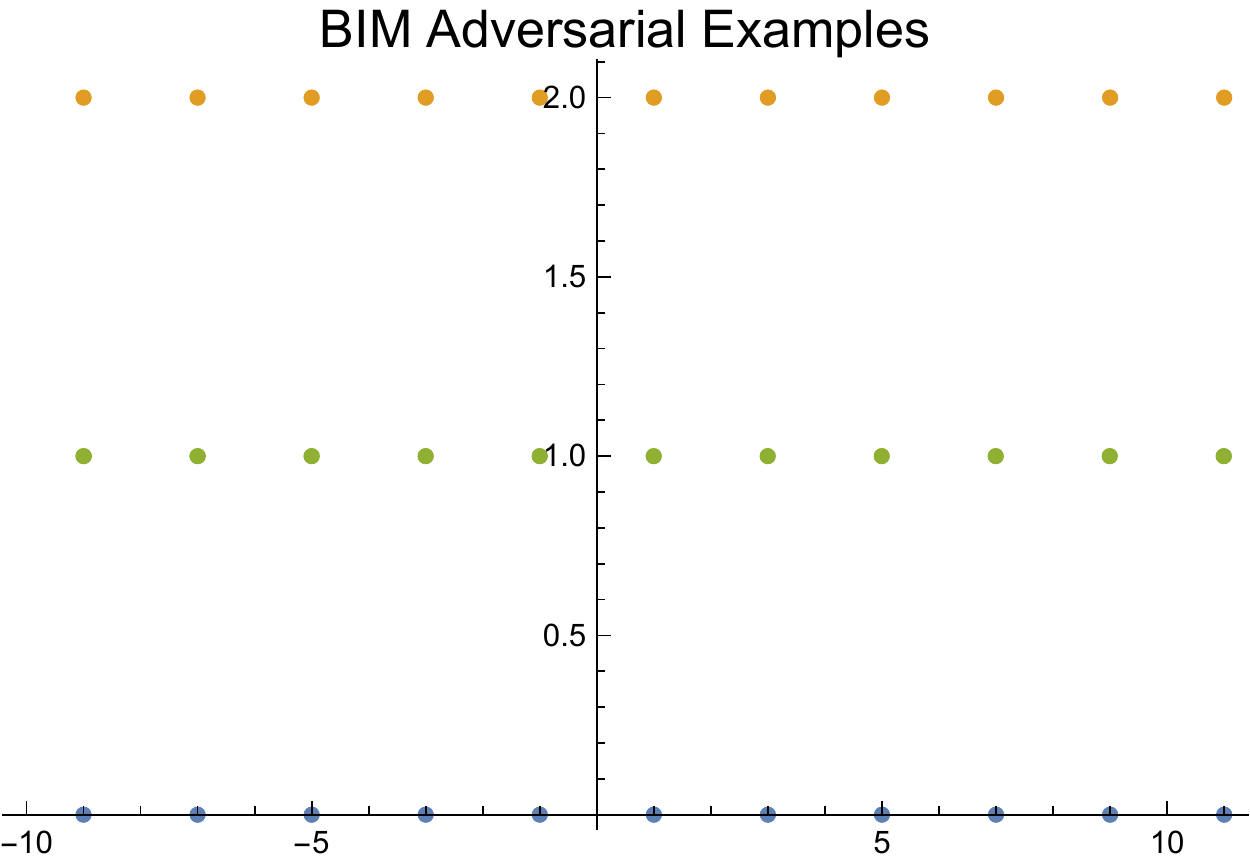}
\end{subfigure}
\caption{Adverarial examples generated using PGD (left), FGSM (center), and BIM (right). While the network is robust to PGD attacks, FGSM and BIM attacks are more effective because they ignore the magnitude of the gradient. For PGD we draw arrows from the test sample to the adversarial example generated from that point to aid the reader.}
\label{fig:advexamples}
\end{center}
\end{figure} 

\section{Implementation Details}
\label{sec:impdets}
For the iterative attacks BIM and PGD, we set the number of iterations to $30$ with a step size of $\epsilon_{\text{step}} = 0.05$ per iteration. 

Our controlled experiments on synthetic data consider a fully connected network with 1 hidden layer, 100 hidden units, and ReLU activations. This model architecture is more than capable of representing a nearly perfect robust decision boundary for both {\sc Circles} and {\sc Planes}, the latter of which is linearly separable. We set the learning rate for Adam as $\alpha = 0.1$, which we found to work best for our datasets. The parameters for the exponential decay of the first and second moment estimates were set to $\beta_1 = 0.9$ and $\beta_2 = 0.999$. We set the learning rate for SGD as $\alpha = 0.1$ and decrease the learning rate by a factor of $10$ every $100$ epochs. We train all of our models for $250$ epochs, following \cite{Wilson17}. We train using a cross-entropy loss.

All of our experiments are implemented using PyTorch. When comparing against a published result we use publicly available repositories, if able. For the robust loss of \cite{Wong18a}, we use the code provided by the authors\footnote{\url{https://github.com/locuslab/convex_adversarial}}.The provided implementation\footnote{\url{https://github.com/MadryLab/mnist_challenge}} of the adversarial training procedure of \cite{Madry17} considers a PGD adversary with $\|\cdot\|_{\infty}$-perturbations. We reimplemented their adversarial training procedure for $\|\cdot\|_{2}$-perturbations following their implementation and using the PGD attack implemented in the cleverhans library (\cite{Papernot18}).

The models of \cite{Madry17} consist of two convolutional layers with 32 and 64 filters respectively, each followed by $2 \times 2$ max pooling. After the two convolutional layers, there are two fully connected layers each with $1024$ hidden units. 

\section{Volume Arguments for $d$-Spheres}
\label{sec:volumespheres}
Let $S \subset \R^{d+1}$ be a unit $d$-sphere embedded in $\R^{d+1}$. The volume of $S^{\epsilon}$ is given by 
\begin{equation}
\vol{S^{\epsilon}} = \frac{\pi^{d/2} ((1 +\epsilon)^d - (1-\epsilon)^d)}{\Gamma(1 + \frac{d}{2})},
\end{equation} 
where $\Gamma$ denotes the gamma function. Let $X \subset S$ be a finite sample of size $n$ of $S$. The set $X^{\epsilon}$ is the set of all $\epsilon$ perturbations of points in $X$ under the norm $\| \cdot \|_{2}$. How well does $X^{\epsilon}$ approximate $S^{\epsilon}$ as a function of $n, d$ and $\epsilon$? 

To answer this question we upper bound the ratio $\vol{X^{\epsilon}} / \vol{S^{\epsilon}}$ by generously assuming that the balls $B(X_i, \epsilon)$ are disjoint. The resulting upper bound is

\begin{equation}
\label{equ:spheresvolumeupperbound}
\frac{\vol{X^{\epsilon}}}{\vol{S^{\epsilon}}} \leq \frac{n \vol{B_{\epsilon}}}{\vol{S^{\epsilon}}} = \frac{n \epsilon^d}{(1+\epsilon)^d - (1-\epsilon)^d}.
\end{equation}

In Figure \ref{fig:volume} we show three different views of this bound. In Figure \ref{fig:volume} (Left) we set $n = 10^{12}$ and plot four different values of $\epsilon$; in each case the percentage of volume of $S^{\epsilon}$ that is covered by $X^{\epsilon}$ quickly approaches $0$. Similarly, in Figure \ref{fig:volume} (Center), if we fix $\epsilon = 1$ and plot four different values of $n$, in each case we have the same result. Finally in Figure \ref{fig:volume} (Right) we plot a lower bound on number of samples necessary to cover $S^{\epsilon}$ by $X^{\epsilon}$ for four different values of $\epsilon$; in each case the number of samples necessary grows exponentially with the dimension. 

\begin{figure}[h!]
\begin{center}
\begin{subfigure}{0.32\textwidth}
\includegraphics[width=0.98\linewidth]{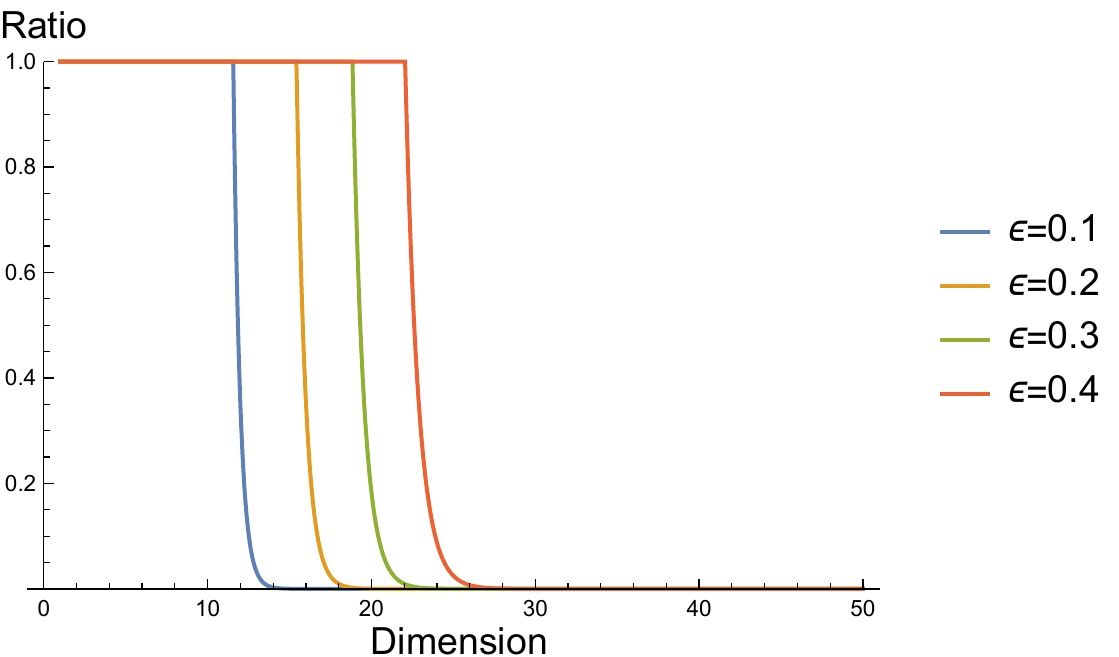}
\end{subfigure}
\begin{subfigure}{0.32\textwidth}
\includegraphics[width=0.98\linewidth]{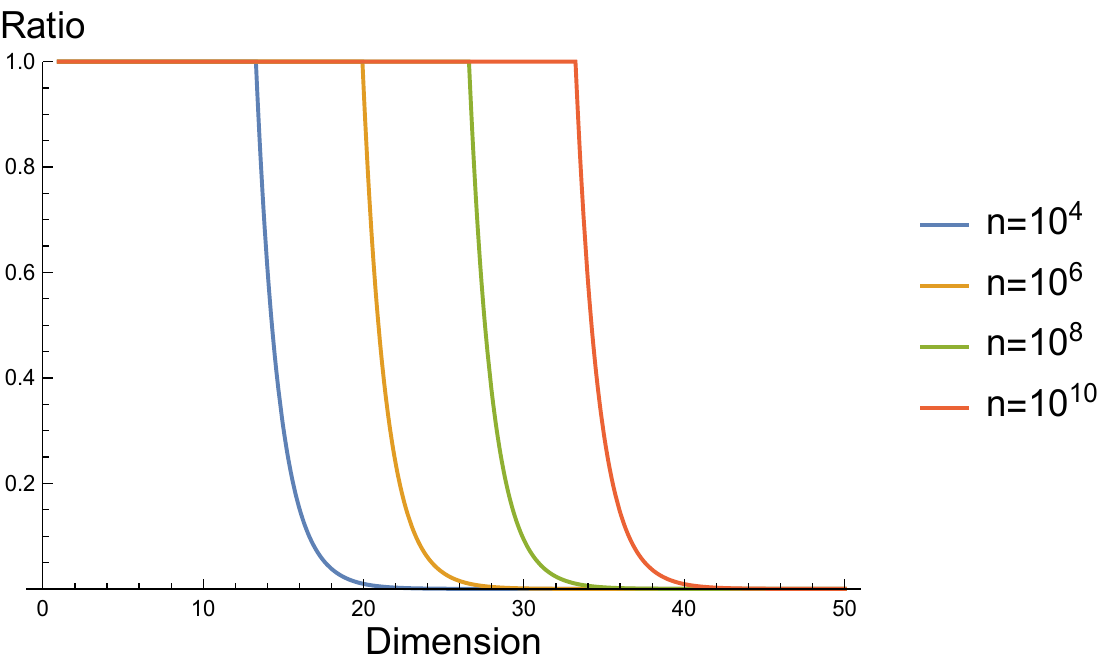}
\end{subfigure}
\begin{subfigure}{0.32\textwidth}
\includegraphics[width=0.98\linewidth]{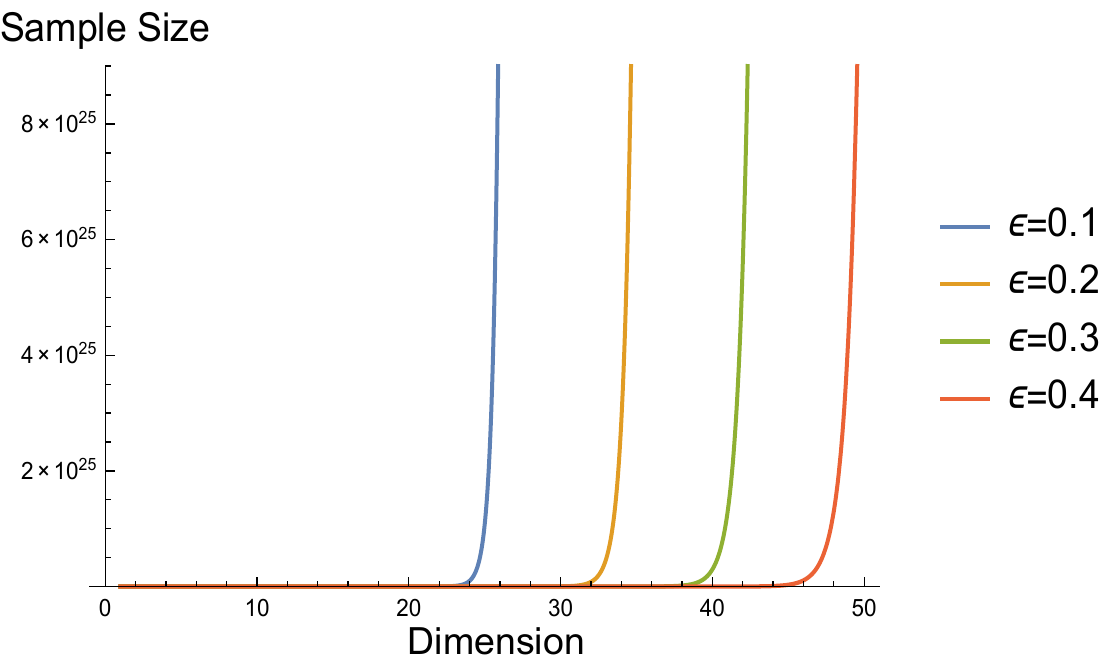}
\end{subfigure}
\caption{Three different perspectives on our upper bound in Equation \ref{equ:spheresvolumeupperbound}. (Left, Center) In each case the percentage of $S^{\epsilon}$ covered by $X^{\epsilon}$ goes to $0$. (Right) The number of points necessary to cover $S^{\epsilon}$ by $X^{\epsilon}$ grows exponentially with the dimension.}
\label{fig:volume}
\end{center}
\end{figure} 

\section{Voronoi Diagrams and Delaunay Triangulations}
\label{sec:vordel}
Let $X \subset \R^d$ be a finite set of $n$ points. The \emph{Voronoi diagram} of $X$, denoted $\Vor{X}$, under the metric $d(\cdot, \cdot)$ is a subdivision of $\R^d$ into $n$ cells where each cell is defined as
\begin{equation}
\Vor{v} = \{x \in \R^d: d(x, v) \leq d(x, u), \forall u \in X \backslash \{v\}\}. 
\end{equation}
In words, the Voronoi cell $\Vor{v}$ of $v \in X$ is the set of all points in $\R^d$ that are closer to $v$ than any other sample point $u \neq v$ in $X$. The Voronoi diagram is then defined as the set of all Voronoi cells, $\Vor{X} = \{\Vor{v}: v \in X\}$. When $d(\cdot, \cdot)$ is induced by the norm $\|\cdot \|_{2}$, the Voronoi cells are convex. See Figure \ref{fig:vordel}. 
 
\begin{figure}[h!]
\begin{center}
\includegraphics[width=0.7\linewidth]{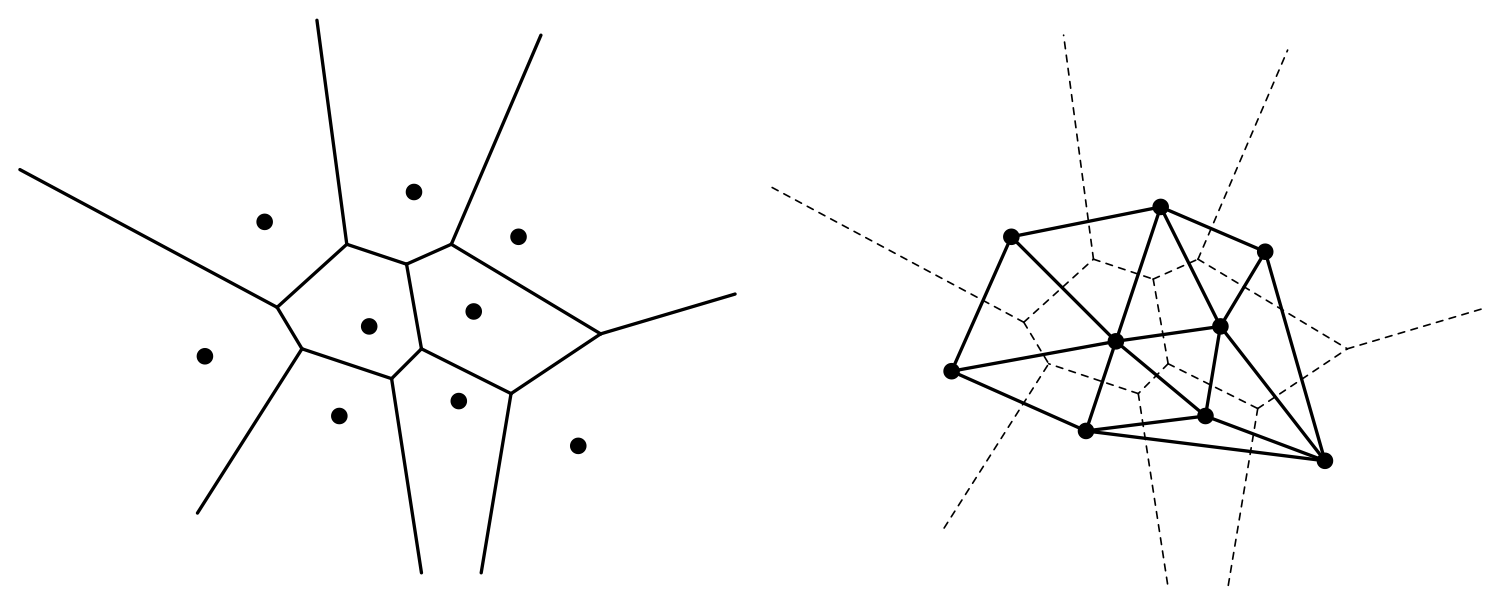}
\end{center}
\caption{The Voronoi diagram of a set of points in $\R^2$ (left) and its dual the Delaunay triangulation (right).}
\label{fig:vordel}
\end{figure}

The \emph{Delaunay triangulation} of $X$, denoted $\Del{X}$ is a triangulation of the convex hull of $X$ into $d$-simplices. Every $d$-simplex $\tau \in \Del{X}$, as well as every lower-dimensional face of $\tau$, has the defining property that there exists an empty circumscribing ball $B$ such that the vertices of $\tau$ lie on the boundary of $B$ and the interior of $B$ is free from any points in $X$. See Figure \ref{fig:vordel}. This \emph{empty circumscribing ball} property of Delaunay triangulations implies many desirable properties that are useful in mesh generation (\cite{Cheng12}) and manifold reconstruction (\cite{Edelsbrunner97}). The Delaunay triangulation of a point set always exists, but is not unique in general.

There exists a well known \emph{duality} between the Voronoi diagram and the Delaunay triangulation of $X$. For every $j$-dimensional face $\sigma \in \Vor{X}$ there exist a dual $(d - j)$-dimensional simplex denoted $\sigma^{*} \in \Del{X}$ whose $d - j + 1$ vertices are the $d-j+1$ vertices of $X$ whose Voronoi cells intersect at $\sigma$. In particular, every $d$-cell of $\Vor{X}$ is dual to the vertex of $\Del{X}$ that generates that cell, and every $(d-1)$-face of $\Vor{X}$ is dual to an edge of $\Del{X}$.

A nearest neighbor classifier $f_{nn}$ given a query point $q$ simply returns the class of the point in $X$ that generated the Voronoi cell in which $q$ lies. Thus the decision boundary of $f_{nn}$ is the union of $(d-1)$ and lower dimensional Voronoi faces. Furthermore, when $X$ is a dense sample of a manifold $\mathcal{M}$, the Voronoi cells are well known to be elongated in the directions normal to $\mathcal{M}$ \cite{Dey07}. This fact underlies many of our results. 

\newpage
\section{Visualization of Decision Boundaries}
In Figure \ref{fig:visdb} we provide visualizations of the decision boundaries learned by (a-d) our fully connected network architecture with cross entropy loss for various optimization procedures and various training lengths, and (e) a nearest neighbor classifier for $\|\cdot\|_{2}$ on the training set. Specifically we train on the {\sc Circles} dataset, embedded in $\R^3$. The training set is entirely contained in the $xy$-plane. We then visualize cross sections of the decision boundary for various values of $z \in [-5, 5]$. We color points labeled as in the same class as the outer circle with the color blue and points labeled as in the same class as the inner circle as orange. Figure~\ref{fig:visdb} shows the cross sections of the decision boundaries, averaged over 20 retrainings. The visualization shows how various optimization algorithms learn decision boundaries that extend into the normal directions where no data is provided.  
\begin{figure}[h!]
\begin{center}
\begin{subfigure}{\textwidth}
\includegraphics[width=0.98\linewidth]{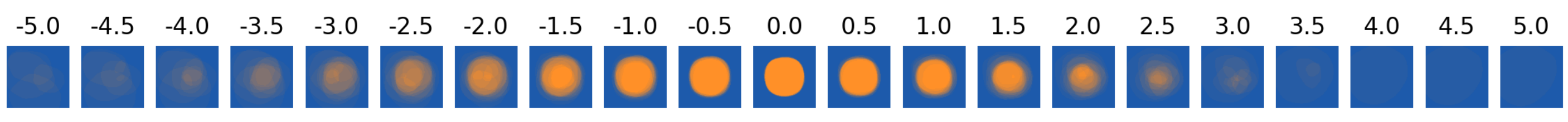}
\caption{Decision boundary learned by running SGD for 25 epochs, averaged over 20 trainings.}
\end{subfigure}
\begin{subfigure}{\textwidth}
\includegraphics[width=0.98\linewidth]{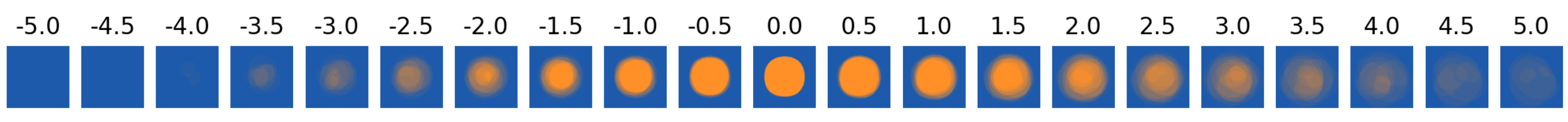}
\caption{Decision boundary learned by running SGD for 250 epochs, averaged over 20 trainings.}
\end{subfigure}
\begin{subfigure}{\textwidth}
\includegraphics[width=0.98\linewidth]{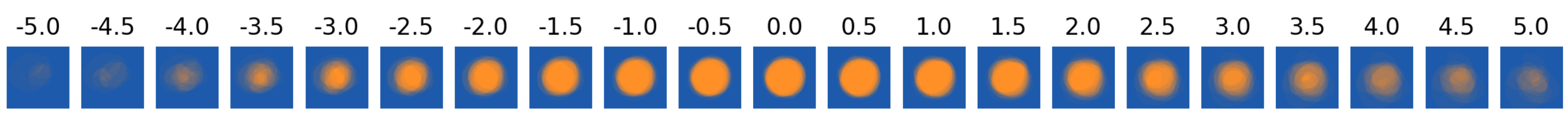}
\caption{Decision boundary learned by running Adam for 25 epochs, averaged over 20 trainings.}
\end{subfigure}
\begin{subfigure}{\textwidth}
\includegraphics[width=0.98\linewidth]{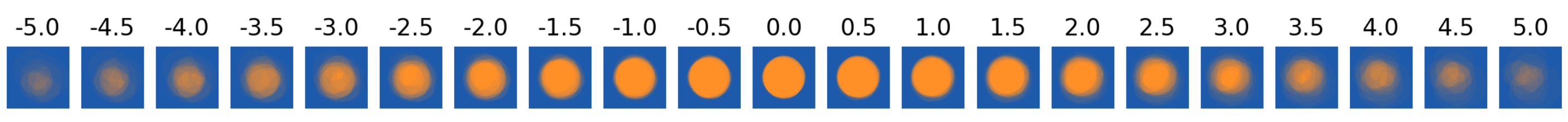}
\caption{Decision boundary learned by running Adam for 250 epochs, averaged over 20 trainings.}
\end{subfigure}
\begin{subfigure}{\textwidth}
\includegraphics[width=0.98\linewidth]{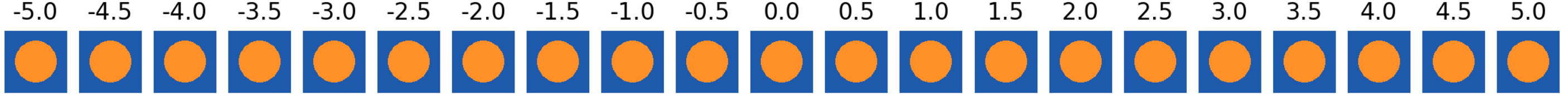}
\caption{Decision boundary of a nearest neighbor classifier for the $\|\cdot\|_{2}$ norm.}
\end{subfigure}
\caption{The training set lies entirely in the $xy$-plane, shown here at $z=0$. We visualize cross sections of the decision boundary for $z \in [-5,5]$ for various optimization algorithms training for different lengths of time. The results show how various optimization algorithm learn decision boundaries that extend into the normal directions in which there is no data provided. We average the decision boundary over 20 retrainings, so faded results indicate how frequently a point was labeled a specific class.}
\label{fig:visdb}
\end{center}
\end{figure} 

\end{document}

%% file: math_commands.tex
\usepackage{amsmath,amsfonts,bm}

\def\eqref#1{equation~\ref{#1}}

\def\ceil#1{\lceil #1 \rceil}

\def\1{\bm{1}}

\def\vtheta{{\bm{\theta}}}

\def\vn{{\bm{n}}}

\DeclareMathAlphabet{\mathsfit}{\encodingdefault}{\sfdefault}{m}{sl}
\SetMathAlphabet{\mathsfit}{bold}{\encodingdefault}{\sfdefault}{bx}{n}

\newcommand{\R}{\mathbb{R}}

